\documentclass{article}

\usepackage{etoolbox}
\newcommand{\arxiv}[1]{\iftoggle{neurips}{}{#1}}
\newcommand{\neurips}[1]{\iftoggle{neurips}{#1}{}}
\newtoggle{neurips}
\global\togglefalse{neurips}

\neurips{
\usepackage[nonatbib]{neurips_2023}
}

\neurips{
\usepackage[subtle]{savetrees}
}

\neurips{
  \usepackage[numbers]{natbib}
  \bibliographystyle{abbrvnat}
}
\arxiv{
\usepackage{natbib}
\bibliographystyle{plainnat}
\bibpunct{(}{)}{;}{a}{,}{,}
}

\newcommand{\loose}{\looseness=-1}

\usepackage[utf8]{inputenc} %
\usepackage[T1]{fontenc}    %
\usepackage{url}            %
\usepackage{booktabs}       %
\usepackage{amsfonts}       %
\usepackage{nicefrac}       %
\usepackage{microtype}      %

\usepackage{tocloft}            %

\usepackage{enumitem}

\usepackage{breakcites}

\usepackage{etoolbox}
\usepackage{comment}
\newtoggle{draft}
\togglefalse{draft}
\newcommand{\draft}[1]{\iftoggle{draft}{#1}{}}

\usepackage{algorithm}
\usepackage{verbatim}
\usepackage[noend]{algpseudocode}
\newcommand{\multiline}[1]{\parbox[t]{\dimexpr\linewidth-\algorithmicindent}{#1}}

\usepackage{multicol}

\usepackage{colortbl}

\usepackage{setspace}

\usepackage{transparent}

\usepackage{inconsolata}
\usepackage[scaled=.90]{helvet}
\usepackage{xspace}

\arxiv{
\usepackage[letterpaper, left=1in, right=1in, top=1in, bottom=1in]{geometry}
}

\PassOptionsToPackage{hypertexnames=false}{hyperref}  %
  \usepackage{parskip}

\usepackage[dvipsnames]{xcolor}
\usepackage[colorlinks=true, linkcolor=blue!70!black, citecolor=blue!70!black,urlcolor=black,breaklinks=true]{hyperref}
\usepackage{microtype}
\usepackage{algorithm}

\usepackage{amsthm}
\usepackage{mathtools}
\usepackage{amsmath}
\usepackage{bbm}
\usepackage{amsfonts}
\usepackage{amssymb}

\usepackage{xpatch}

\usepackage{thmtools}
\usepackage{thm-restate}
\declaretheorem[name=Theorem,parent=section]{theorem}
\declaretheorem[name=Lemma,parent=section]{lemma}
\declaretheorem[name=Assumption, parent=section]{assumption}

\declaretheorem[qed=$\triangleleft$,name=Remark,style=definition, parent=section]{remark}
\declaretheorem[name=Proposition, parent=section]{proposition}

\makeatletter
  \renewenvironment{proof}[1][Proof]%
  {%
   \par\noindent{\bfseries\upshape {#1.}\ }%
  }%
  {\qed\newline}
  \makeatother

\theoremstyle{definition}  %

\newtheorem{corollary}{Corollary}[section]

\theoremstyle{plain}
\newtheorem{definition}{Definition}[section]

\xpatchcmd{\proof}{\itshape}{\normalfont\proofnameformat}{}{}
\newcommand{\proofnameformat}{\bfseries}

\usepackage[nameinlink,capitalize]{cleveref}

\newcommand{\pref}[1]{\cref{#1}}
\newcommand{\pfref}[1]{Proof of \pref{#1}}

\crefformat{equation}{#2(#1)#3}
\Crefformat{equation}{#2(#1)#3}

\Crefformat{figure}{#2Figure #1#3}
\Crefname{assumption}{Assumption}{Assumptions}
\Crefformat{assumption}{#2Assumption #1#3}

\usepackage{crossreftools}
\DeclarePairedDelimiter{\abs}{\lvert}{\rvert} %
\DeclarePairedDelimiter{\brk}{[}{]}
\DeclarePairedDelimiter{\crl}{\{}{\}}
\DeclarePairedDelimiter{\prn}{(}{)}

\DeclarePairedDelimiter{\tri}{\langle}{\rangle}

\DeclareMathOperator{\En}{\mathbb{E}}

\DeclareMathOperator*{\argmin}{arg\,min} %
\DeclareMathOperator*{\argmax}{arg\,max}

\newcommand{\mb}[1]{\boldsymbol{#1}}
\newcommand{\wt}[1]{\widetilde{#1}}
\newcommand{\wh}[1]{\widehat{#1}}
\newcommand{\wb}[1]{\widebar{#1}}

\def\ddefloop#1{\ifx\ddefloop#1\else\ddef{#1}\expandafter\ddefloop\fi}
\def\ddef#1{\expandafter\def\csname bb#1\endcsname{\ensuremath{\mathbb{#1}}}}
\ddefloop ABCDEFGHIJKLMNOPQRSTUVWXYZ\ddefloop
\def\ddefloop#1{\ifx\ddefloop#1\else\ddef{#1}\expandafter\ddefloop\fi}
\def\ddef#1{\expandafter\def\csname b#1\endcsname{\ensuremath{\mathbf{#1}}}}
\ddefloop ABCDEFGHIJKLMNOPQRSTUVWXYZ\ddefloop
\def\ddef#1{\expandafter\def\csname sf#1\endcsname{\ensuremath{\mathsf{#1}}}}
\ddefloop ABCDEFGHIJKLMNOPQRSTUVWXYZ\ddefloop
\def\ddef#1{\expandafter\def\csname c#1\endcsname{\ensuremath{\mathcal{#1}}}}
\ddefloop ABCDEFGHIJKLMNOPQRSTUVWXYZ\ddefloop
\def\ddef#1{\expandafter\def\csname h#1\endcsname{\ensuremath{\widehat{#1}}}}
\ddefloop ABCDEFGHIJKLMNOPQRSTUVWXYZ\ddefloop
\def\ddef#1{\expandafter\def\csname hc#1\endcsname{\ensuremath{\widehat{\mathcal{#1}}}}}
\ddefloop ABCDEFGHIJKLMNOPQRSTUVWXYZ\ddefloop
\def\ddef#1{\expandafter\def\csname t#1\endcsname{\ensuremath{\widetilde{#1}}}}
\ddefloop ABCDEFGHIJKLMNOPQRSTUVWXYZ\ddefloop
\def\ddef#1{\expandafter\def\csname tc#1\endcsname{\ensuremath{\widetilde{\mathcal{#1}}}}}
\ddefloop ABCDEFGHIJKLMNOPQRSTUVWXYZ\ddefloop
\def\ddefloop#1{\ifx\ddefloop#1\else\ddef{#1}\expandafter\ddefloop\fi}
\def\ddef#1{\expandafter\def\csname scr#1\endcsname{\ensuremath{\mathscr{#1}}}}
\ddefloop ABCDEFGHIJKLMNOPQRSTUVWXYZ\ddefloop

\newcommand{\ls}{\ell}
\newcommand{\ind}{\mathbbm{1}}    %

\newcommand{\veps}{\varepsilon}

\newcommand{\ldef}{\vcentcolon=}
\newcommand{\rdef}{=\vcentcolon}

\newcommand{\Gclass}{\cG}
\newcommand{\loss}{\ell}

\DeclarePairedDelimiter{\set}{\{}{\}}
\newcommand{\Xclass}{\cX}
\newcommand{\Yclass}{\cY}
\newcommand{\Zclass}{\cW}
\newcommand{\Kernel}{\cK}
\newcommand{\histSet}{S}
\newcommand{\x}{x}
\newcommand{\y}{y}
\newcommand{\z}{w}
\newcommand{\coef}{\beta}
\newcommand{\Event}{A}

\newcommand{\vanilla}{original\xspace}

\DeclareFontFamily{U}{mathx}{\hyphenchar\font45}
\DeclareFontShape{U}{mathx}{m}{n}{
      <5> <6> <7> <8> <9> <10>
      <10.95> <12> <14.4> <17.28> <20.74> <24.88>
      mathx10
      }{}
\DeclareSymbolFont{mathx}{U}{mathx}{m}{n}
\DeclareFontSubstitution{U}{mathx}{m}{n}
\DeclareMathAccent{\widecheck}{0}{mathx}{"71}
\DeclareMathAccent{\wideparen}{0}{mathx}{"75}

\newcommand{\Piq}{\Pi_{\cQ}}

\newcommand{\cEhat}{\wh{\cE}}%

\newcommand{\vepsconc}{\veps_{\mathrm{conc}}(n)}%
\newcommand{\vepsconcs}{\veps^2_{\mathrm{conc}}(n)}%

\newcommand{\piest}{\pi^{\mathrm{est}}}
\newcommand{\piestq}[1][Q]{\pi^{\mathrm{est}}_{#1}}
\newcommand{\lest}{\ell^{\mathrm{est}}_h}

\newcommand{\fq}{f^{Q}}

\newcommand{\Dbipi}[3][\pi]{D^{#1}_{\mathsf{bi}}\prn*{#2\dmid#3}}

\newcommand{\Dbishort}{D_{\mathsf{bi}}}
\newcommand{\Dbiflipshort}{\Dflipshort_{\mathsf{bi}}}

\newcommand{\ocompbi}[1][\gamma]{\odec^{\mathsf{bi}}_{#1}}
\newcommand{\EstOptBi}[1][\gamma]{\EstOpt^{\mathsf{bi}}_{#1}}
\newcommand{\EstOptBilong}[1][\gamma]{\EstOpt_{#1}^{\Dbishort}}

\newcommand{\ocompsbe}[1][\gamma]{\odec^{\mathsf{sbe}}_{#1}}
\newcommand{\EstOptSB}[1][\gamma]{\EstOpt^{\mathsf{sbe}}_{#1}}
\newcommand{\EstOptSBlong}[1][\gamma]{\EstOpt_{#1}^{\Dsbeshort}}

\newcommand{\Dhelspi}[3][\pi]{D^{#1}_{\mathsf{H}}\prn*{#2,#3}}

\newcommand{\piq}[1][Q]{\pi_{#1}}

\newcommand{\tst}{\textsf{TS3}\xspace}

\newcommand{\dlike}{divergence-like\xspace}

\newcommand{\Dsqpi}[3][\pi]{D^{#1}_{\mathsf{sq}}\prn*{#2,#3}}
\newcommand{\Dsqshort}{D_{\mathsf{sq}}}

\newcommand{\Dhels}[2]{D^{2}_{\mathsf{H}}\prn*{#1,#2}}

\newcommand{\Dhelshort}{D_{\mathsf{H}}}

\newcommand{\Dsbepi}[3][\pi]{D^{#1}_{\mathsf{sbe}}\prn*{#2\dmid#3}}
\newcommand{\Dsbeshort}{D_{\mathsf{sbe}}}

\newcommand{\Dflippi}[3][\pi]{\Dflipshort^{#1}\prn*{#2\dmid{}#3}}
\newcommand{\Dflipshort}{\widecheck{D}}

\newcommand{\Dgen}[2]{D\prn*{#1\dmid{}#2}}

\newcommand{\Dgenshort}{D}
\newcommand{\Dgenpi}[3][\pi]{D^{#1}\prn*{#2\dmid{}#3}}
\newcommand{\DgenpiX}[4][\pi]{D^{#1}\prn[#2]{#3\dmid{}#4}}

\newcommand{\suff}{\psi}
\newcommand{\suffhat}{\wh{\psi}}
\newcommand{\suffmap}{\mb{\psi}}
\newcommand{\Suff}{\Psi}
\newcommand{\fsuff}{f^{\suff}}
\newcommand{\pisuff}{\pi_{\suff}}
\newcommand{\fsuffhat}{f^{\suffhat}}
\newcommand{\pisuffhat}{\pi_{\suffhat}}

\newcommand{\EstOpt}{\mathrm{\normalfont{\textbf{OptEst}}}}
\newcommand{\EstOptD}[1][\gamma]{\EstOpt_{#1}^{D}}
\newcommand{\EstOptDFull}[1][\gamma]{\EstOpt_{#1}^{D}(T,\delta)}

\newcommand{\EstOptDFullKn}[1][\gamma]{\EstOpt_{#1}^{D}(K,n,\delta)}
\newcommand{\EstOptHel}[1][\gamma]{\EstOpt^{\mathsf{H}}_{#1}}

\newcommand{\EstOptSq}[1][\gamma]{\EstOpt^{\mathsf{sq}}_{#1}}

\newcommand{\odec}{\normalfont{\textsf{o-dec}}}

\newcommand{\ocompD}[1][\gamma]{\odec^{D}_{#1}}
\newcommand{\ocompH}[1][\gamma]{\odec^{\mathsf{H}}_{#1}}
\newcommand{\ocompSq}[1][\gamma]{\odec^{\mathsf{bi}}_{#1}}
\newcommand{\newcomp}{Optimistic DEC\xspace}

\newcommand{\ocomphel}[1][\gamma]{\odec^{\mathsf{H}}_{#1}}

\newcommand{\compgen}[1][D]{\comp^{#1}}
\newcommand{\compgenrand}[1][D]{\underline{\mathsf{dec}}_{\gamma}^{#1}}
\newcommand{\compgenrandbasic}[1][D]{\underline{\mathsf{dec}}^{#1}}

\newcommand{\compsbe}[1][D]{\comp^{\Dsbeshort}}

\newcommand{\Lcont}{L_{\mathrm{lip}}}

\newcommand{\fmhatt}{f\sups{\Mhat\ind{t}}}

\newcommand{\dec}{\mathsf{dec}}

\newcommand{\RegDM}{\Reg_{\mathsf{DM}}}

\newcommand{\Empi}[1][M]{\En^{\sss{#1},\pi}}

\newcommand{\cTm}[1][M]{\cT\sups{#1}}
\newcommand{\cTmstar}[1][\Mstar]{\cT\sups{#1}}

\newcommand{\compbi}{\compgen[\mathsf{bi}]}

\newcommand{\dimbi}{d_{\mathsf{bi}}}

\newcommand{\Lbi}{L_{\mathrm{bi}}}
\newcommand{\Lbis}{L^2_{\mathrm{bi}}}

\newcommand{\pialphaq}[1][Q]{\pi^{\alpha}_{\sss{#1}}}

\renewcommand{\emptyset}{\varnothing}

\newcommand{\filt}{\mathfrak{F}}

\newcommand{\hist}{\mathcal{H}}

\newcommand{\Framework}{Decision Making with Structured Observations\xspace}
\newcommand{\FrameworkShort}{DMSO\xspace}

\newcommand{\learner}{learner\xspace}

\newcommand{\act}{\pi}
\newcommand{\Act}{\Pi}

\newcommand{\obs}{o}

\newcommand{\cMhat}{\wh{\cM}}

\newcommand{\comp}[1][\gamma]{\mathsf{dec}_{#1}}

\newcommand{\compH}{\compgen[\mathsf{H}]}

\newcommand{\CompText}{Decision-Estimation Coefficient\xspace}
\newcommand{\CompAbbrev}{DEC\xspace}
\newcommand{\CompShort}{\CompAbbrev}

\newcommand{\etdtext}{Estimation-to-Decisions\xspace}
\newcommand{\etd}{\textsf{E\protect\scalebox{1.04}{2}D}\xspace}
\newcommand{\etdopt}{\textsf{E\protect\scalebox{1.04}{2}D.Opt}\xspace}
\newcommand{\etdopttext}{Optimistic Estimation-to-Decisions\xspace}

\newcommand{\M}[1]{^{{\scriptscriptstyle M}}}  %

\newcommand{\sups}[1]{^{{\scriptscriptstyle#1}}}
\newcommand{\subs}[1]{_{{\scriptscriptstyle#1}}}

\newcommand{\sss}[1]{{\scriptscriptstyle#1}}

\newcommand{\Ens}[2]{\En^{\sss{#1},#2}}

\newcommand{\Enmpi}[2]{\En^{\sss{#1},#2}}

\newcommand{\fm}[1][M]{f\sups{#1}}
\newcommand{\pim}[1][M]{\pi_{\sss{#1}}}

\newcommand{\cFm}{\cF_{\cM}}

\newcommand{\fmbar}{f\sups{\Mbar}}
\newcommand{\pimbar}{\pi\subs{\Mbar}}

\newcommand{\fmstar}{f\sups{\Mstar}}
\newcommand{\pimstar}{\pi\subs{\Mstar}}

\newcommand{\Mbar}{\wb{M}}

\newcommand{\pimi}{\pi\subs{M_i}}

\newcommand{\Rm}[1][M]{R\sups{#1}}
\newcommand{\Pm}[1][M]{P\sups{#1}}

\newcommand{\PiRNS}{\Pi_{\mathrm{RNS}}} %

\newcommand{\Qmstarstar}[1][\Mstar]{Q^{\sss{#1},\star}}

\newcommand{\Qmstar}[1][M]{Q^{\sss{#1},\star}}

\newcommand{\Reg}{\mathrm{\mathbf{Reg}}}

\newcommand{\Est}{\mathrm{\mathbf{Est}}}

\newcommand{\EstHel}{\mathrm{\mathbf{Est}}^{\mathsf{H}}}

\newcommand{\EstH}{\EstHel}
\newcommand{\EstHFull}{\EstHel(T,\delta)}

\newcommand{\EstD}{\mathrm{\mathbf{Est}}^{D}}
\newcommand{\EstDFull}{\mathrm{\mathbf{Est}}^{D}(T,\delta)}

\newcommand{\RegOL}{\mathrm{\mathbf{Reg}}_{\mathsf{OL}}}

  \newcommand{\AlgEst}{\mathrm{\mathbf{Alg}}_{\mathsf{Est}}}

\newcommand{\Mhat}{\wh{M}}
\newcommand{\Mstar}{M^{\star}}

\newcommand{\algcommentlight}[1]{\textcolor{blue!70!black}{\transparent{0.5}\footnotesize{\texttt{\textbf{//\hspace{2pt}#1}}}}}

\newcommand{\midsem}{\,;}

\newcommand{\approxleq}{\lesssim}
\newcommand{\approxgeq}{\gtrsim}

\newcommand{\ul}[1][T]{u_{1:#1}}

\renewcommand{\ind}[1]{^{{\scriptscriptstyle #1}}}

\newcommand{\bigoh}{O}
\newcommand{\bigoht}{\wt{O}}
\newcommand{\bigom}{\Omega}

\newcommand{\indic}{\mathbb{I}}

\newcommand{\dmid}{\;\|\;}

\newcommand{\conv}{\mathrm{co}}

\newcommand{\Qstar}{Q^{\star}}

\newcommand{\unif}{\mathrm{Unif}}

\newcommand{\mathand}{\quad\text{and}\quad}

\def\multiset#1#2{\ensuremath{\left(\kern-.3em\left(\genfrac{}{}{0pt}{}{#1}{#2}\right)\kern-.3em\right)}}

\renewcommand{\ls}{\ell}

\renewcommand{\emptyset}{\varnothing}

\newcommand{\nc}{\newcommand}
\nc{\DMO}{\DeclareMathOperator}
\newcount\Comments  %
\DMO{\prox}{prox}
\DMO{\Span}{span}
\DMO{\UCB}{UCB}
\DMO{\LCB}{LCB}
\nc{\br}[2]{{\rm br}^{#1}({#2})}
\nc{\depth}[1]{{\rm d}({#1})}
\nc{\child}[2]{{\rm ch}_{#1}({#2})}
\nc{\parent}[1]{{\rm pa}({#1})}
\nc{\dg}{\dagger}
\nc{\indsig}[2]{\mathcal{I}_{#1}({#2})}
\nc{\total}{{\rm fin}}
\nc{\early}{{\rm pre}}
\nc{\zsink}{z_{\rm sink}}
\nc{\lowv}{{\rm low}}
\nc{\ol}{\overline}
\nc{\madec}[3]{\texttt{ma-dec}_{#1}({#2}, {#3})}
\nc{\madeco}[1]{\texttt{ma-dec}_{#1}}
\nc{\madecd}[3]{\texttt{ma-dec}^{\texttt{d}}_{#1}({#2}, {#3})}
\nc{\mainf}{\texttt{ma-inf}}
\nc{\maexo}{\texttt{ma-exo}}
\nc{\decc}{\texttt{dec}^{\texttt{c}}}
\nc{\deccp}{\texttt{dec}^{\texttt{c-pac}}}
\nc{\deccr}{\texttt{dec}^{\texttt{c-reg}}}
\nc{\Alg}{{\rm\bf Alg}}
\nc{\co}{{\rm co}}
\nc{\BV}{\mathbb{V}}
\nc{\ham}[2]{d_{\rm Ham}({#1}, {#2})}

\nc{\gamvec}{\gamma}
\nc{\til}{\widetilde}
\nc{\td}{\tilde}
\nc{\todo}[1]{\ifnum\Comments=1 {\color{red}  [TODO: #1]}\fi}
\nc{\old}[1]{\ifnum\Comments=1 {\color{brown}  [OLD: #1]}\fi}

\nc{\BP}{\mathbb{P}}
\nc{\BI}{\mathbb{I}}

\nc{\fools}[3]{\MF_{#3}({#1}, {#2})}
\nc{\fool}[2]{\MF({#1},{#2})}
\nc{\clip}[2]{{\rm clip}\left[ \left. {#1} \right| {#2} \right]}
\nc{\imax}{\omega}
\nc{\CF}{\mathscr{F}}
\nc{\CG}{\mathscr{G}}
\nc{\CA}{\mathscr{A}}
\nc{\MH}{\mathcal{H}}
\nc{\MV}{\mathcal{V}}
\nc{\MC}{\mathcal{C}}
\nc{\MI}{\mathcal{I}}
\nc{\MQ}{\mathcal{Q}}
\nc{\st}{\star}
\nc{\lng}{\langle}
\nc{\rng}{\rangle}
\DMO{\OOPT}{opt}
\nc{\dopt}[2]{\ell_{\OOPT}({#1},{#2})}
\nc{\MG}{\mathcal{G}}
\nc{\MP}{\mathcal{P}}
\nc{\PP}{\mathbb{P}}
\nc{\TT}{\mathbb{T}}
\nc{\TTmax}{\TT_{\max}}
\DMO{\REG}{Reg}
\DMO{\WREG}{wReg}
\nc{\wreg}[2]{{\Delta}^{\rm w}_{{#1}}({#2})}
\nc{\wReg}[2]{{\WREG}_{{#1}}({#2})}
\DMO{\Ham}{Ham}
\DMO{\Gap}{Gap}
\DMO{\GD}{GD}
\DMO{\GDA}{GDA}
\DMO{\EG}{EG}
\DMO{\OGDA}{OGDA}
\DMO{\Unif}{Unif}
\DMO{\Tr}{Tr}
\nc{\Qu}{\ul{Q}}
\nc{\Qo}{\ol{Q}}
\nc{\Ro}{\ol{R}}
\nc{\Vu}{\ul{V}}
\nc{\Vo}{\ol{V}}
\nc{\RanQ}{\Delta Q}
\nc{\RanV}{\Delta V}
\nc{\clipQ}{\Delta \breve{Q}}
\nc{\frzQ}{\Delta \mathring{Q}}
\nc{\clipV}{\Delta \breve{V}}
\nc{\clipdelta}{\breve{\delta}}
\nc{\cliptheta}{\breve{\theta}}
\nc{\delmin}{\Delta_{{\rm min}}}
\nc{\delmins}[1]{\Delta_{{\rm min},{#1}}}
\nc{\gapfinal}[1]{\max \left\{ \frac{\frzQ_{{#1}}^{k^\st}(x,a)}{2H}, \frac{\delmin}{4H} \right\}}
\nc{\post}[2]{R({#1}; {#2})}
\nc{\posts}[3]{R_{#3}({#1}; {#2})}
\nc{\pstr}{{\rm po}}
\nc{\prior}{{\rm pr}}

\nc{\algnst}[1]{\begin{align*}#1\end{align*}}
\nc{\algn}[1]{\begin{align}#1\end{align}}
\nc{\matx}[1]{\left(\begin{matrix}#1\end{matrix}\right)}

\nc{\nuu}{\nu}

\nc{\bel}[1]{\mathbf{b}({#1})}
\nc{\nbel}[1]{\bar{\mathbf{b}}({#1})}
\nc{\sbel}[2]{\mathbf{b}'_{#1}({#2})}
\nc{\nsbel}[2]{\bar{\mathbf{b}}'_{#1}({#2})}

\nc{\bone}{\mathbf{1}}

\nc{\MO}{\mathcal O}
\nc{\MU}{\mathcal{U}}
\nc{\ME}{\mathcal{E}}
\nc{\MN}{\mathcal{N}}
\nc{\MK}{\mathcal{K}}
\nc{\MM}{\mathcal{M}}
\nc{\ML}{\mathcal{L}}
\nc{\MS}{\mathcal{S}}
\nc{\MT}{\mathcal{T}}
\nc{\BF}{\mathbb F}
\nc{\BQ}{\mathbb Q}
\nc{\MX}{\mathcal{X}}
\nc{\MA}{\mathcal{A}}
\nc{\MD}{\mathcal{D}}
\nc{\MB}{\mathcal{B}}
\nc{\MZ}{\mathcal{Z}}
\nc{\MJ}{\mathcal{J}}
\nc{\MW}{\mathcal{W}}
\nc{\MR}{\mathcal{R}}
\nc{\MY}{\mathcal{Y}}
\nc{\BZ}{\mathbb Z}
\nc{\BN}{\mathbb N}
\nc{\ep}{\epsilon}
\nc{\vep}{\varepsilon}
\nc{\gapfn}[1]{\varepsilon_{#1}}
\nc{\ggapfn}[2]{\varphi_{#1}({#2})}
\nc{\epsahk}{\gapfn{0}}
\nc{\BH}{\mathbb H}
\nc{\BG}{\mathbb{G}}
\nc{\D}{\Delta}
\nc{\MF}{\mathcal{F}}
\nc{\One}[1]{\mathbbm{1}\left\{{#1}\right\}}
\nc{\bOne}{\mathbf{1}}
\nc{\Aopt}{\mathcal{A}^{\rm opt}}
\nc{\Amul}{\mathcal{A}^{\rm mul}}
\nc{\CM}{\mathscr{M}}
\nc{\CO}{\mathscr{O}}
\nc{\CR}{\mathsscr{R}}

\nc{\SP}{\mathsf P}
\nc{\SQ}{\mathsf Q}
\nc{\SC}{\mathscr{C}}
\nc{\SD}{\mathscr{D}}
\nc{\SE}{\mathscr{E}}
\nc{\SG}{\mathscr{G}}

\nc{\DO}{\accentset{\circ}{\D}}
\nc{\mf}{\mathfrak}
\nc{\mfp}{\mathfrak{p}}
\nc{\mfq}{\mf{q}}
\nc{\Sp}{\mbox{Spec}}
\nc{\Spm}{\mbox{Specm}}
\nc{\hookuparrow}{\mathrel{\rotatebox[origin=c]{90}{$\hookrightarrow$}}}
\nc{\hookdownarrow}{\mathrel{\rotatebox[origin=c]{-90}{$\hookrightarrow$}}}
\nc{\hra}{\hookrightarrow}
\nc{\tra}{\twoheadrightarrow}
\nc{\sgn}{{\rm sgn}}
\nc{\aut}{{\rm Aut}}
\nc{\Hom}{{\rm Hom}}
\nc{\img}{{\rm Im}}
\DMO{\id}{Id}
\DMO{\KL}{KL}
\nc{\kld}[2]{\KL({#1}||{#2})}
\nc{\ren}[3]{D_{#3}({#1}||{#2})}
\nc{\chisq}[2]{\chi^2({#1},{#2})}
\nc{\dvg}[2]{D({#1} \| {#2})}
\DMO{\BSS}{BSS}
\DMO{\BES}{BES}
\DMO{\BGS}{BGS}
\nc{\indep}{\perp}
\DMO{\sink}{sink}
\nc{\fp}[1]{\MP_1({#1})}
\nc{\BO}{\mathbb{O}}
\nc{\BT}{\mathbb{T}}

\nc{\RR}{\mathbb{R}}
\nc{\Gradient}{\nabla}
\nc{\norm}[1]{\left \lVert #1 \right \rVert}
\nc{\EE}{\mathbb{E}}

\DMO{\PR}{Pr}

\nc{\E}{\mathbb{E}}
\nc{\ra}{\rightarrow}

\nc{\opo}{\texttt{opo}} \let\underbar\undefined
\makeatletter
\let\save@mathaccent\mathaccent
\newcommand*\if@single[3]{%
  \setbox0\hbox{${\mathaccent"0362{#1}}^H$}%
  \setbox2\hbox{${\mathaccent"0362{\kern0pt#1}}^H$}%
  \ifdim\ht0=\ht2 #3\else #2\fi
  }
\newcommand*\rel@kern[1]{\kern#1\dimexpr\macc@kerna}
\newcommand*\widebar[1]{\@ifnextchar^{{\wide@bar{#1}{0}}}{\wide@bar{#1}{1}}}
\newcommand*\underbar[1]{\@ifnextchar_{{\under@bar{#1}{0}}}{\under@bar{#1}{1}}}
\newcommand*\wide@bar[2]{\if@single{#1}{\wide@bar@{#1}{#2}{1}}{\wide@bar@{#1}{#2}{2}}}
\newcommand*\under@bar[2]{\if@single{#1}{\under@bar@{#1}{#2}{1}}{\under@bar@{#1}{#2}{2}}}
\newcommand*\wide@bar@[3]{%
  \begingroup
  \def\mathaccent##1##2{%
    \let\mathaccent\save@mathaccent
    \if#32 \let\macc@nucleus\first@char \fi
    \setbox\z@\hbox{$\macc@style{\macc@nucleus}_{}$}%
    \setbox\tw@\hbox{$\macc@style{\macc@nucleus}{}_{}$}%
    \dimen@\wd\tw@
    \advance\dimen@-\wd\z@
    \divide\dimen@ 3
    \@tempdima\wd\tw@
    \advance\@tempdima-\scriptspace
    \divide\@tempdima 10
    \advance\dimen@-\@tempdima
    \ifdim\dimen@>\z@ \dimen@0pt\fi
    \rel@kern{0.6}\kern-\dimen@
    \if#31
      \overline{\rel@kern{-0.6}\kern\dimen@\macc@nucleus\rel@kern{0.4}\kern\dimen@}%
      \advance\dimen@0.4\dimexpr\macc@kerna
      \let\final@kern#2%
      \ifdim\dimen@<\z@ \let\final@kern1\fi
      \if\final@kern1 \kern-\dimen@\fi
    \else
      \overline{\rel@kern{-0.6}\kern\dimen@#1}%
    \fi
  }%
  \macc@depth\@ne
  \let\math@bgroup\@empty \let\math@egroup\macc@set@skewchar
  \mathsurround\z@ \frozen@everymath{\mathgroup\macc@group\relax}%
  \macc@set@skewchar\relax
  \let\mathaccentV\macc@nested@a
  \if#31
    \macc@nested@a\relax111{#1}%
  \else
    \def\gobble@till@marker##1\endmarker{}%
    \futurelet\first@char\gobble@till@marker#1\endmarker
    \ifcat\noexpand\first@char A\else
      \def\first@char{}%
    \fi
    \macc@nested@a\relax111{\first@char}%
  \fi
  \endgroup
}
\newcommand*\under@bar@[3]{%
  \begingroup
  \def\mathaccent##1##2{%
    \let\mathaccent\save@mathaccent
    \if#32 \let\macc@nucleus\first@char \fi
    \setbox\z@\hbox{$\macc@style{\macc@nucleus}_{}$}%
    \setbox\tw@\hbox{$\macc@style{\macc@nucleus}{}_{}$}%
    \dimen@\wd\tw@
    \advance\dimen@-\wd\z@
    \divide\dimen@ 3
    \@tempdima\wd\tw@
    \advance\@tempdima-\scriptspace
    \divide\@tempdima 10
    \advance\dimen@-\@tempdima
    \ifdim\dimen@>\z@ \dimen@0pt\fi
    \rel@kern{0.6}\kern-\dimen@
    \if#31
      \underline{\rel@kern{-0.6}\kern\dimen@\macc@nucleus\rel@kern{0.4}\kern\dimen@}%
      \advance\dimen@0.4\dimexpr\macc@kerna
      \let\final@kern#2%
      \ifdim\dimen@<\z@ \let\final@kern1\fi
      \if\final@kern1 \kern-\dimen@\fi
    \else
      \underline{\rel@kern{-0.6}\kern\dimen@#1}%
    \fi
  }%
  \macc@depth\@ne
  \let\math@bgroup\@empty \let\math@egroup\macc@set@skewchar
  \mathsurround\z@ \frozen@everymath{\mathgroup\macc@group\relax}%
  \macc@set@skewchar\relax
  \let\mathaccentV\macc@nested@a
  \if#31
    \macc@nested@a\relax111{#1}%
  \else
    \def\gobble@till@marker##1\endmarker{}%
    \futurelet\first@char\gobble@till@marker#1\endmarker
    \ifcat\noexpand\first@char A\else
      \def\first@char{}%
    \fi
    \macc@nested@a\relax111{\first@char}%
  \fi
  \endgroup
}
\makeatother

 \usepackage[suppress]{color-edits}
 \addauthor{df}{ForestGreen}
 \addauthor{as}{magenta}
 \addauthor{jq}{yellow!60!black}
 \addauthor{ng}{purple}
 \addauthor{sr}{BurntOrange}
 \draft{

 \usepackage[textsize=tiny]{todonotes}

 }
\makeatletter
\let\OldStatex\Statex
\renewcommand{\Statex}[1][3]{%
  \setlength\@tempdima{\algorithmicindent}%
  \OldStatex\hskip\dimexpr#1\@tempdima\relax}
\makeatother

 \usepackage{accents}

 \addtocontents{toc}{\protect\setcounter{tocdepth}{0}}

\let\oldparagraph\paragraph
\renewcommand{\paragraph}[1]{\oldparagraph{#1.}}

\newcommand{\paragraphi}[1]{\par\noindent\emph{#1.}}

\newcommand{\fakepar}[1]{\par\noindent\textbf{#1.}~}

\title{Model-Free Reinforcement Learning with \\ the
  Decision-Estimation Coefficient}
\arxiv{
  \author{%
    Dylan J. Foster\\%
    {\small\texttt{dylanfoster@microsoft.com}} \and
    Noah Golowich\\
    {\small\texttt{nzg@mit.edu}}
    \and Jian Qian\\%
    {\small\texttt{jianqian@mit.edu}} \and Alexander
    Rakhlin\\%
    {\small\texttt{rakhlin@mit.edu}} \and Ayush
    Sekhari\\%
    {\small\texttt{sekhari@mit.edu}}
  }
}

\date{}

\begin{document}
\maketitle

\begin{abstract}

We consider the problem of interactive decision making, encompassing structured bandits and reinforcement
learning with general function approximation. Recently, \citet{foster2021statistical} introduced the
\CompText, a measure of statistical complexity that lower bounds the optimal regret for interactive decision
making, as well as a meta-algorithm, \etdtext, which achieves upper
bounds in terms of the same quantity. \etdtext is a \emph{reduction}, which lifts
algorithms for (supervised) online estimation into algorithms for
decision making. In this paper, we show that by combining \etdtext with
a specialized form of \emph{optimistic estimation} introduced by
\citet{zhang2022feel}, it is possible to obtain guarantees
that improve upon those of \citet{foster2021statistical} by
accommodating more lenient notions of estimation error.
We use this approach to derive regret bounds for
model-free reinforcement learning with value function approximation,
and give structural results showing when it can and cannot help more generally.

 \end{abstract}

\section{Introduction}
\label{sec:intro}
\arxiv{We study interactive decision making problems, ranging from bandits
to reinforcement learning, in which a
learning agent repeatedly interacts with an unknown environment with
the goal of maximizing a reward function, and aims to improve the quality
of their decisions on the fly by learning from feedback.
To learn to make decisions in a sample-efficient fashion, particularly in domains
with complex, high-dimensional feedback, it is critical to
incorporate prior knowledge, typically via modeling and function
approximation. In this context, a fundamental challenge is to understand what
modeling assumptions lead to low sample complexity, and what
algorithms achieve this.
While the literature on reinforcement
learning and interactive decision making contains a variety of
sufficient conditions for sample-efficient learning,
\citep{russo2013eluder,jiang2017contextual,sun2019model,du2021bilinear,jin2021bellman,dean2020sample,yang2019sample,jin2020provably,modi2020sample,ayoub2020model,krishnamurthy2016pac,du2019latent,li2009unifying,dong2019provably,zhou2021nearly},
\emph{necessary conditions} have been comparatively unexplored. Recently,
however, \citet{foster2021statistical} introduced the \CompText
(\CompShort), a measure of statistical complexity which leads to upper \emph{and}
lower bounds on the optimal sample complexity for interactive decision
making.\\
Regret bounds based on the \CompText are achieved by \etdtext (\etd),
a meta-algorithm which reduces the problem of interactive decision
making to supervised online estimation. While the \CompText leads to tight lower bounds on regret for many
problem settings, the upper bounds in \citet{foster2021statistical}
can be suboptimal in certain situations due to the need to perform estimation with respect to
\emph{Hellinger distance}, a stringent notion of estimation error. When specialized to reinforcement learning, the guarantees
for the \etd meta-algorithm in \citet{foster2021statistical} are only
tight for model-based settings (where function approximation is employed
to model and estimate transition probabilities), and do not lead to
meaningful guarantees for model-free settings with value function approximation.}
\neurips{
The theory of interactive decision making---ranging from bandits
to reinforcement learning with function approximation---contains a variety of
sufficient conditions for sample-efficient learning,
\citep{russo2013eluder,jiang2017contextual,sun2019model,du2021bilinear,jin2021bellman,dean2020sample,yang2019sample,jin2020provably,modi2020sample,ayoub2020model,krishnamurthy2016pac,du2019latent,li2009unifying,dong2019provably,zhou2021nearly},
but \emph{necessary conditions} have been comparatively unexplored. Recently,
however, \citet{foster2021statistical} introduced the \CompText
(\CompShort), a measure of statistical complexity which leads to upper \emph{and}
lower bounds on the optimal sample complexity for interactive decision
making.

Regret bounds based on the \CompText are achieved by \etdtext (\etd),
a meta-algorithm which reduces the problem of interactive decision
making to supervised online estimation. While the \CompText leads to tight lower bounds on regret for many
problem settings, the upper bounds in \citet{foster2021statistical}
can be suboptimal in certain situations due to the need to perform estimation with respect to
\emph{Hellinger distance}, a stringent notion of estimation error. When specialized to reinforcement learning, the guarantees
for the \etd meta-algorithm in \citet{foster2021statistical} are only
tight for model-based settings (where function approximation is employed
to model and estimate transition probabilities), and do not lead to
meaningful guarantees for model-free settings with value function
approximation. In this paper, we explore the prospect of developing
tighter regret bounds suitable for model-free settings.
}

\paragraph{Contributions} \arxiv{In this paper, we}\neurips{We} show that by
combining \etdtext with \emph{optimistic online estimation}, an elegant technique
recently introduced by \citet{zhang2022feel}, it is possible to obtain
regret bounds that improve upon \citet{foster2021statistical} by
accommodating weaker notions of estimation error. Our main
contributions are\arxiv{ as follows}:
\neurips{\vspace{2pt}\\
  $\bullet$ We introduce a new \emph{optimistic} variant of the \CompText, and
  show that a variant of \etdtext that incorporates optimistic
  estimation achieves regret bounds that scale with this quantity (\cref{sec:main}). Using this approach, we derive the first regret bounds for
  \etdtext applied to model-free reinforcement learning with \emph{bilinear
    classes} \citep{du2021bilinear} (\cref{sec:rl}). \vspace{2pt}\\
  $\bullet$ We show that in general, whether or not optimistic estimation
  leads to improvement depends on the \emph{divergence} with respect
  to which estimation is performed: For \emph{symmetric} divergences,
  optimistic estimation offers no improvement, but for
  \emph{asymmetric} divergences, including those found in
  reinforcement learning, the improvement can be drastic (\cref{sec:discussion}). In addition, we highlight settings in which combining optimistic
estimation with \etdtext offers provable improvement over previous
approaches that apply the technique with posterior sampling
\citep{zhang2022feel}.\loose

Perhaps the most important aspect of our work is to elucidate the connection between the \CompShort framework
and optimistic estimation, building stronger foundations
for further research into these techniques.\loose
  }
\arxiv{\begin{itemize}
\item We introduce a new \emph{optimistic} variant of the \CompText, and
  show that a variant of \etdtext that incorporates optimistic
  estimation achieves regret bounds that scale with this quantity.
\item Using this approach, we show \neurips{for the first time }how to use the
  \etdtext paradigm to derive regret bounds for model-free
  reinforcement learning with the general \emph{bilinear
    class} framework \citep{du2021bilinear} (\cref{sec:rl}).
\item We show that in general, whether or not optimistic estimation
  leads to improvement depends on the \emph{divergence} with respect
  to which estimation is performed: For \emph{symmetric} divergences,
  optimistic estimation offers no improvement, but for
  \emph{asymmetric} divergences, including those found in
  reinforcement learning, the improvement can be drastic.
\end{itemize}
}

In what follows, we \arxiv{formally introduce the interactive decision making
problem (\pref{sec:setup}), then }review the \CompText and \etdtext
meta-algorithm (\pref{sec:background}),
highlighting opportunities for improvement. In \pref{sec:main},
we present our main results, including our application to model-free
reinforcement learning. We close with discussion and structural results,
highlighting situations in which optimistic estimation can and cannot
help (\pref{sec:discussion}).

\subsection{Problem Setting}
\label{sec:setup}

We adopt the \emph{\Framework} (\FrameworkShort) framework of
\citet{foster2021statistical}, which is a general setting for
interactive decision making that encompasses bandit problems
(structured, contextual, and so forth) and reinforcement learning with function
approximation. 

The protocol consists of $T$ rounds. For each round $t=1,\ldots,T$:
  \begin{enumerate}
  \item The \learner selects a \emph{decision} $\act\ind{t}\in\Act$,
    where $\Act$ is the \emph{decision space}.
  \item The learner receives a reward $r\ind{t}\in\cR\subseteq\bbR$
    and observation $o\ind{t}\in\cO$ sampled via
    $(r\ind{t},o\ind{t})\sim{}\Mstar(\pi\ind{t})$, where
    $\Mstar:\Pi\to\Delta(\cR\times\cO)$ is the underlying \emph{model}.

  \end{enumerate}
Above, $\cR$ is the \emph{reward space} and $\cO$ is the
\emph{observation space}.
    The model (conditional distribution)
    $\Mstar$ represents the underlying
 con   environment, and is unknown to the learner, but the learner
    is assumed to have access to a \emph{model class} $\cM\subset(\Pi\to\Delta(\cR\times\cO))$ that
    is flexible enough to capture $\Mstar$.
    \begin{assumption}[Realizability]
  \label{ass:realizability}
  The learner has access to a model class $\cM$ containing the true model $\Mstar$.
\end{assumption}
The model class $\cM$ represents the learner's prior knowledge about the
decision making problem, and allows one to appeal to estimation and function
approximation. For structured bandit problems, models correspond to
reward distributions, and $\cM$ encodes structure in the reward
landscape. For reinforcement learning problems, models correspond to
Markov decision processes (MDPs), and $\cM$ typically encodes structure in value functions or transition probabilities. We refer to \citet{foster2021statistical}
  for further background.

For a model $M\in\cM$, $\Empi[M]\brk*{\cdot}$ denotes the expectation
  under the process $(r,\obs)\sim{}M(\pi)$,
  $\fm(\pi)\ldef{}\Empi[M]\brk*{r}$ denotes the mean reward function,
  and $\pim\ldef{}\argmax_{\act\in\Act}\fm(\act)$ denotes the optimal decision.
  We measure performance in terms of regret, which is given by \neurips{$    \RegDM\ldef\sum_{t=1}^{T}\En_{\pi\ind{t}\sim{}p\ind{t}}\brk*{\fmstar(\pimstar)-\fmstar(\pi\ind{t})}$,}
  \arxiv{\begin{equation}
    \label{eq:regret}
    \RegDM\ldef\sum_{t=1}^{T}\En_{\pi\ind{t}\sim{}p\ind{t}}\brk*{\fmstar(\pimstar)-\fmstar(\pi\ind{t})},
  \end{equation}}
  where $p\ind{t}$ is the learner's randomization distribution for
  round $t$.

  \paragraph{Additional notation} For an integer $n\in\bbN$, we let $[n]$ denote the set
  $\{1,\dots,n\}$. For a set $\cZ$, we let
        $\Delta(\cZ)$ denote the set of all probability distributions
        over $\cZ$. For a model class $\cM$, $\conv(\cM)$ denotes the convex hull. We write $f=\bigoht(g)$ to denote that $f =
        \bigoh(g\cdot{}\max\crl*{1,\mathrm{polylog}(g)})$, and use $\approxleq$ as shorthand for $a=\bigoh(b)$.

  \subsection{Background: Estimation-to-Decisions and \CompText}
  \label{sec:background}
\begin{algorithm}[t]
    \setstretch{1.3}
    \begin{algorithmic}[1]
      \neurips{
              \State \textbf{parameters}: Estimation oracle
              $\AlgEst$, Exp. parameter $\gamma>0$, divergence $\Dgenpi{\cdot}{\cdot}$.
       }
      \arxiv{
       \State \textbf{parameters}:
       \Statex[1] Online estimation oracle $\AlgEst$.
       \Statex[1] Exploration parameter $\gamma>0$.
       \Statex[1] Divergence $\Dgenpi{\cdot}{\cdot}$.
       }
  \For{$t=1, 2, \cdots, T$}
  \State Compute estimate $\Mhat\ind{t} = \AlgEst\ind{t}\prn[\big]{ \crl*{(\act\ind{i},
r\ind{i},\obs\ind{i})}_{i=1}^{t-1} }$.
\State 
$p\ind{t}\gets\argmin_{p\in\Delta(\Act)}\sup_{M\in\cM}\En_{\act\sim{}p}\brk[\big]{\fm(\pim)-\fm(\pi)
    -\gamma\cdot\DgenpiX{\big}{\Mhat\ind{t}}{M}}$.
\algcommentlight{Eq. \pref{eq:comp_general}}.%
\State{}Sample decision $\act\ind{t}\sim{}p\ind{t}$ and update estimation
oracle with $(\act\ind{t},r\ind{t}, \obs\ind{t})$.
\EndFor
\end{algorithmic}
\caption{\etdtext (\etd) for General Divergences}
\label{alg:etd}
\end{algorithm}

To motivate our results, this section provides a short primer on the \etdtext
meta-algorithm and the \CompText. We refer to
\citet{foster2021statistical} for further background.

\paragraph{Online estimation}
\etdtext (\pref{alg:etd}) is a reduction that lifts algorithms
for online estimation into algorithms for decision making. An
online estimation oracle, denoted by $\AlgEst$, is an algorithm
that, using knowledge of the class $\cM$, estimates the underlying model $\Mstar$ from
data in a sequential fashion. At each round $t$, given the data
$\hist\ind{t-1}=(\pi\ind{1},r\ind{1},o\ind{1}),\ldots,(\pi\ind{t-1},r\ind{t-1},o\ind{t-1})$
observed so far, the estimation oracle computes an estimate
\neurips{$
\Mhat\ind{t}=\AlgEst\ind{t}\prn*{ \crl*{(\act\ind{i},
r\ind{i},\obs\ind{i})}_{i=1}^{t-1} }$
}
\arxiv{\[
\Mhat\ind{t}=\AlgEst\ind{t}\prn*{ \crl*{(\act\ind{i},
r\ind{i},\obs\ind{i})}_{i=1}^{t-1} }
\]}
for the true model $\Mstar$.

To measure the estimation oracle's performance, we make use of a
user-specified \emph{\dlike function}, which quantifies the discrepancy between models. Formally, we define a
\dlike function (henceforth, ``divergence'') as any function
$\Dgenshort:\Pi\times\conv(\cM)\times\conv(\cM)\to\bbR_{+}$, with
$\Dgenpi[\pi]{M}{M'}$ representing the discrepancy between the models $M$
and $M'$ at the decision $\pi$. Standard choices used in past work
\citep{foster2021statistical,foster2022complexity,chen2022unified,foster2023tight,foster2023complexity,wagenmaker2023instance} include the squared
error $\Dsqpi{M}{M'}\ldef{}(\fm(\pi)-f\sups{M'}(\pi))^2$ for bandit
problems, and squared Hellinger distance\footnote{When
  $\Dgenpi{\cdot}{\cdot}$ is symmetric, we write
  $D^{\pi}(\cdot,\cdot)$ to make this explicit.}
\neurips{$
\Dhelspi{M}{M'}\ldef\Dhels{M(\pi)}{M'(\pi)}
$}
\arxiv{\[
\Dhelspi{M}{M'}\ldef\Dhels{M(\pi)}{M'(\pi)}
\]}
for \arxiv{reinforcement learning}\neurips{RL}, where \arxiv{we recall that }for distributions
$\bbP$ and $\bbQ$,
\neurips{$
  \Dhels{\bbP}{\bbQ}\ldef\int\prn[\big]{\sqrt{d\bbP}-\sqrt{d\bbQ}}^2$.
}
\arxiv{\[
\Dhels{\bbP}{\bbQ}\ldef\int\prn*{\sqrt{d\bbP}-\sqrt{d\bbQ}}^2.
\]}
We then measure the estimation oracle's performance in terms of
\emph{cumulative estimation error} with respect to $D$, defined as
    \begin{equation}
    \label{eq:general_error}
    \EstD \ldef{} \sum_{t=1}^{T}\En_{\act\ind{t}\sim{}p\ind{t}}\brk*{\Dgenpi[\pi\ind{t}]{\Mhat\ind{t}}{\Mstar}},
  \end{equation}
  where $p\ind{t}$ is the conditional distribution over $\pi\ind{t}$
  given $\cH\ind{t-1}$. We make the following assumption on the
  algorithm's performance.

\begin{assumption}
    \label{ass:oracle}
	At each time $t\in[T]$, the \emph{online estimation oracle}
        $\AlgEst$ returns,
        given  \arxiv{$$(\pi\ind{1},r\ind{1},o\ind{1}),\ldots,(\pi\ind{t-1},r\ind{t-1},o\ind{t-1})$$}\neurips{$(\pi\ind{1},r\ind{1},o\ind{1}),\ldots,(\pi\ind{t-1},r\ind{t-1},o\ind{t-1})$}
        with $(r\ind{i},o\ind{i})\sim\Mstar(\pi\ind{i})$ and
        $\pi\ind{i}\sim p\ind{i}$, an estimator
        $\Mhat\ind{t}:\Pi\to\Delta(\cR\times\cO)$ such that
        $\EstD\leq\EstDFull$,
	with probability at least $1-\delta$, where $\EstDFull$ is a
        known upper bound.
      \end{assumption}
      For the squared error, one can obtain
      $\Est^{\mathsf{sq}}(T,\delta)\ldef{}\Est^{\Dsqshort}(T,\delta)\approxleq\log(\abs{\cFm}/\delta)$, where
      $\cFm\ldef\crl{\fm\mid{}M\in\cM}$, and for
      Hellinger distance, it is possible to obtain
      $\vphantom{e^{e^{e^e}}}\EstHFull\ldef\Est^{\Dhelshort}(T,\delta)\approxleq\log(\abs{\cM}/\delta)$.

\paragraph{Estimation-to-Decisions}
A general version of the \etd meta-algorithm is displayed in
\pref{alg:etd}. At each timestep $t$, the algorithm queries estimation oracle to
obtain an estimator $\Mhat\ind{t}$ using the data
$(\pi\ind{1},r\ind{1},o\ind{1}),\ldots,(\pi\ind{t-1},r\ind{t-1},o\ind{t-1})$ observed so
far. The algorithm then computes the decision distribution $p\ind{t}$ by solving a
min-max optimization problem involving $\Mhat\ind{t}$ and $\cM$ (as
well as the divergence $D$), and
then samples the decision $\pi\ind{t}$ from this distribution.

  \paragraph{The \CompText} The min-max optimization problem in \pref{alg:etd} is
  derived from the \emph{\CompText} (\CompShort), a complexity measure whose value,
  for a given scale parameter $\gamma>0$ and reference model
  $\Mbar:\Pi\to\Delta(\cR\times\cO)$, is given by
\begin{equation}
  \label{eq:comp_general}
  \compgen(\cM,\Mbar) =
    \inf_{p\in\Delta(\Act)}\sup_{M\in\cM}\En_{\act\sim{}p}\biggl[\fm(\pim)-\fm(\pi)
    -\gamma\cdot\Dgenpi{\Mbar}{M}
    \biggr],
  \end{equation}
  with
  $\compgen(\cM)\ldef\sup_{\Mbar\in\conv(\cM)}\compgen(\cM,\Mbar)$. Informally,
  the \CompShort measures the best tradeoff between suboptimality $(\fm(\pim)-\fm(\pi))$
  and information gain (measured by $\Dgenpi{\Mbar}{M}$) that can be
  achieved by a decision distribution $p$ in the face of a worst-case
  model $M\in\cM$.

  The main result of \citet{foster2021statistical} shows that the
  regret of \etd is controlled by the \CompShort and the estimation
  oracle's cumulative error $\EstD$. \arxiv{To state the result, let}\neurips{Let} $\cMhat$ be any set
for which $\Mhat\ind{t}\in\cMhat$ for all $t$ almost surely.
\begin{restatable}[\citet{foster2021statistical}]{theorem}{uppergeneraldistance}
  \label{thm:prev}
  \pref{alg:etd} with exploration parameter $\gamma>0$ guarantees
  that
\neurips{$\RegDM \leq{}
  \sup_{\Mbar\in\cMhat}\compgen(\cM,\Mbar)\cdot{}T + \gamma\cdot\EstD
  $}
\arxiv{\begin{equation}
  \label{eq:upper_general_distance}
\RegDM \leq{}
\sup_{\Mbar\in\cMhat}\compgen(\cM,\Mbar)\cdot{}T + \gamma\cdot\EstD
\end{equation}}
almost surely.
\end{restatable}
For the special case of Hellinger distance, standard algorithms
(exponential weights) achieve
$\EstHFull\approxleq\log(\abs{\cM}/\delta)$ with $\cMhat=\conv(\cM)$, so that \pref{thm:prev}
gives (abbreviating $\compH\equiv\dec_{\gamma}^{\Dhelshort}$):
\begin{align}
  \RegDM \approxleq{}
  \compH(\cM)\cdot{}T + \gamma\cdot\log(\abs{\cM}/\delta).
  \label{eq:hellinger_ub}
\end{align}

\paragraph{Opportunities for improvement}
\citet{foster2021statistical,foster2023tight} provide \emph{lower bounds} on the
regret for any decision making problem that have
a similar expression to \pref{eq:hellinger_ub}, showing that the
\CompText with Hellinger distance ($\compH(\cM)$) plays a fundamental role in
determining the statistical complexity of decision making. However,
these lower bounds do contain the estimation term
$\EstHFull=\log(\abs{\cM}/\delta)$) appearing in \pref{eq:hellinger_ub}, and thus capture the price
of moving from estimation to decision making (as characterized by the \CompShort),
but not the complexity of estimation itself. 

In general, the dependence on $\EstHFull=\log(\abs{\cM}/\delta)$ in
the upper bound \cref{eq:hellinger_ub} can render the bound loose. In reinforcement learning, working with Hellinger
distance necessitates modeling transition probabilities. While this
leads to optimal results in some settings, in general the optimal rates for Hellinger
estimation error can be prohibitively
large, even in settings where model-free (or, \emph{value-based}) methods,
which directly model value functions, are known to
succeed; this drawback is shared by all subsequent work based on the
\CompShort
\citep{foster2022complexity,chen2022unified,foster2023tight,foster2023complexity,wagenmaker2023instance}. A
natural solution is to replace Hellinger distance
with a divergence (e.g., based on Bellman error)
tailored to value function approximation, but naive choices for $D$
along these lines render $\compgen(\cM,\Mbar)$ too large to give
meaningful guarantees, and \citet{foster2021statistical,foster2023tight} left this as
an open problem.\footnote{For simpler model classes, it \emph{is}
possible to improve upon
\pref{eq:hellinger_ub} by moving from Hellinger distance to lenient notions of estimation
error: In bandit problems with Gaussian rewards,\arxiv{ the mean
reward function acts as a sufficient statistic, and} it suffices to consider the
$\Dsqpi{M}{\Mbar}\ldef{}(\fm(\act)-\fmbar(\act))^2$, which leads to
upper bounds that scale with $\log\abs{\cFm}\ll\log\abs{\cM}$ \citep{foster2021statistical}.}

\section{\etdtext with Optimistic Estimation}
\label{sec:main}

To derive improved regret bounds that address the shortcomings
described in the prequel, we combine \etdtext with a
specialized estimation approach introduced by \citet{zhang2022feel}
(see also \citet{dann2021provably,agarwal2022model,agarwal2022non,zhong2022posterior}),
which we refer to as \emph{optimistic estimation}. We then use this
approach to derive regret bounds for model-free reinforcement learning.

\subsection{Optimistic Estimation}

The idea of optimistic estimation is to augment the estimation objective with
a bonus that biases the estimator toward models $M\in\cM$ for which
the value $\fm(\pim)$ is large. We present a general version of the
technique.

Let a divergence $\Dgenpi{\cdot}{\cdot}$ be fixed. Following the
development in \pref{sec:background}, an \emph{optimistic estimation oracle}
$\AlgEst$ is an algorithm which, at each step $t$, given the
observations and rewards collected so far, computes an estimate for
the underlying model. For technical reasons, it
will be useful to consider \emph{randomized estimators}
(\citet{foster2021statistical,chen2022unified}) that, at
each round, produce a distribution $\mu\ind{t}\in\Delta(\cM)$ over
models. Such estimators take the form
\neurips{$
\mu\ind{t}=\AlgEst\ind{t}\prn*{ \crl*{(\act\ind{i},
r\ind{i},\obs\ind{i})}_{i=1}^{t-1} }$.
}
\arxiv{\[
\mu\ind{t}=\AlgEst\ind{t}\prn*{ \crl*{(\act\ind{i},
r\ind{i},\obs\ind{i})}_{i=1}^{t-1} }.
\]}
where $\mu\ind{t}\in\Delta(\cM)$. For a parameter $\gamma>0$, we
define the \emph{optimistic estimation error} as
\begin{align}
  \label{eq:est_error_basic}
  \EstOptD \ldef{}
\sum_{t=1}^{T}\En_{\act\ind{t}\sim{}p\ind{t}}\En_{\Mhat\ind{t}\sim\mu\ind{t}}\brk*{\Dgenpi{\Mhat\ind{t}}{\Mstar}
  + \gamma^{-1}(\fmstar(\pimstar)-\fmhatt(\pi\subs{\Mhat\ind{t}})}.
\end{align}
This quantity is similar to \pref{eq:general_error}, but incorporates a bonus
term
\neurips{$
  \gamma^{-1}(\fmstar(\pimstar)-\fmhatt(\pi\subs{\Mhat\ind{t}}))$,
}
\arxiv{\[
  \gamma^{-1}(\fmstar(\pimstar)-\fmhatt(\pi\subs{\Mhat\ind{t}})),
\]}
which encourages the estimation algorithm to \emph{over-estimate} the
optimal value $\fmstar(\pimstar)$ for the underlying model.
\begin{assumption}
    \label{ass:optimistic_oracle}
	At each time $t\in[T]$, the \emph{optimistic estimation oracle}
        $\AlgEst$ returns,
        given  \arxiv{$$(\pi\ind{1},r\ind{1},o\ind{1}),\ldots,(\pi\ind{t-1},r\ind{t-1},o\ind{t-1})$$}\neurips{$(\pi\ind{1},r\ind{1},o\ind{1}),\ldots,(\pi\ind{t-1},r\ind{t-1},o\ind{t-1})$}
        with $(r\ind{i},o\ind{i})\sim\Mstar(\pi\ind{i})$ and
        $\pi\ind{i}\sim p\ind{i}$, a randomized estimator $\mu\ind{t}\in\Delta(\cM)$ such that
        $\EstOptD\leq{}\EstOptDFull$,
	with \arxiv{probability at least}\neurips{w.p.} $1-\delta$, where $\EstOptDFull$ is a
        known upper bound.
      \end{assumption}
      For the case of contextual bandits, \citet{zhang2022feel}
      proposes an augmented version of the exponential weights
      algorithm which, for a learning rate parameter $\eta>0$, sets
      \neurips{$
        \mu(M) \propto\exp\prn*{
          -\eta\prn*{L\ind{t}(\fm) - \gamma^{-1}\fm(\pim)}
          }
        $,}
      \arxiv{\[
        \mu(M) \propto\exp\prn*{
          -\eta\prn*{L\ind{t}(\fm) - \gamma^{-1}\fm(\pim)}
          },
        \]}
        where $L\ind{t}(\fm)$ is the squared prediction error for the
        rewards observed so far. This method achieves $\En\brk[\big]{\EstOptSq}\approxleq{}
\log(\abs{\cFm}) + \sqrt{T\log\abs{\cFm}}/\gamma$, and
\citet{zhang2022feel} combines this estimator with posterior sampling
to achieve optimal contextual bandit regret. \citet{agarwal2022model,zhong2022posterior,agarwal2022non} extend this
development to reinforcement learning, also using posterior
sampling as the exploration
mechanism.\footnote{\citet{dann2021provably} also apply the optimistic estimation
  idea to model-free reinforcement learning, but do not provide
  \emph{online} estimation guarantees.}
In what follows, we combine optimistic estimation with
\etdtext, which provides a universal mechanism for exploration. Beyond
giving guarantees which were previously out of reach for \etd
(\pref{sec:rl}), this approach generalizes and subsumes posterior
sampling, and can succeed in situations where posterior sampling fails
(\pref{sec:discussion}).

\begin{remark}
  For the non-optimistic estimation error $\EstD$, it is possible to
  obtain low error for well-behaved losses such as the square loss and
  Hellinger distance without the use of randomization by appealing to
  improper mixture estimators (e.g., \citet{foster2021statistical}). We show in \pref{sec:discussion} that
  for such divergences, randomization does not lead to statistical
  improvements. For the optimistic estimation error \pref{eq:est_error_basic}, randomization is essential due to
  the presence of the term $\gamma^{-1}(\fmstar(\pimstar)-\fmhatt(\pi\subs{\Mhat\ind{t}})$.
\end{remark}

\arxiv{\paragraph{Sufficient statistics}}
\neurips{\fakepar{Sufficient statistics}} Before proceeding, we note that many divergences of interest have
the useful property that they depend on the estimated model $\Mhat$
only through a ``sufficient statistic'' for the model class under
consideration. Formally, there exists a
\emph{sufficient statistic space} $\Suff$ and \emph{sufficient
  statistic} $\suffmap:\cM\to\Psi$ with the property that we can write
(overloading notation)
\[
\Dgenpi{M}{M'} = \Dgenpi{\suffmap(M)}{M'},\quad\fm(\pi)=f^{\suffmap(M)}(\pi),\mathand\pim=\pi_{\suffmap(M)}
\]
for all models $M,M'$. In this case, it suffices for the online
estimation oracle to directly estimate the sufficient statistic by
producing a randomized estimator $\mu\ind{t}\in\Delta(\Suff)$. We measure
performance via
\begin{align}
  \EstOptD \ldef{}
\sum_{t=1}^{T}\En_{\act\ind{t}\sim{}p\ind{t}}\En_{\suffhat\ind{t}\sim\mu\ind{t}}\brk*{\Dgenpi[\pi\ind{t}]{\suffhat\ind{t}}{\Mstar}
  + \gamma^{-1}(\fmstar(\pimstar)-f^{\suffhat\ind{t}}(\pi_{\suffhat\ind{t}}))}
\end{align}
Examples include bandit problems, where one may use squared estimation
error $\Dsqpi{\cdot}{\cdot}$ and take $\suffmap(\cM)=\fm$, and
model-free reinforcement learning, 
where we show that by choosing the divergence $D$
appropriately, one can use \emph{Q-value functions} as a sufficient
statistic. \dfedit{Note that we only focus on sufficient statistics for the
first argument to $\Dgenpi{\cdot}{\cdot}$, since this is the quantity
we wish to estimate.}

\subsection{Algorithm and Main Result}
\begin{algorithm}[t]
    \setstretch{1.3}
    \begin{algorithmic}[1]
      \neurips{\State \textbf{parameters}: Estimation oracle
        $\AlgEst$, Exp. parameter $\gamma>0$, divergence $D$ with suff.
        stat. space $\Suff$.}
       \arxiv{\State \textbf{parameters}:
       \Statex[1] Online estimation oracle $\AlgEst$.
       \Statex[1] Exploration parameter $\gamma>0$.
       \Statex[1] Divergence $\Dgen{\cdot}{\cdot}$ with sufficient
       statistic space $\Suff$.}
  \For{$t=1, 2, \cdots, T$}
  \State Receive randomized estimator $\mu\ind{t}\in\Delta(\Suff) = \AlgEst\ind{t}\prn[\big]{ \crl*{(\act\ind{i},
r\ind{i},\obs\ind{i})}_{i=1}^{t-1} }$.
\State 
\mbox{\arxiv{Compute}\neurips{Get} $p\ind{t}\gets\argmin_{p\in\Delta(\Act)}\sup_{M\in\cM}\En_{\act\sim{}p}\En_{\suffhat\sim\mu\ind{t}}\brk*{\fsuffhat(\pisuffhat)-\fm(\pi)
    -\gamma\cdot\Dgenpi{\suffhat}{M}}$.
  \arxiv{\algcommentlight{Eq. \pref{eq:optimistic_dec}}.}
  \neurips{\algcommentlight{Eq. \pref{eq:optimistic_dec}.}}}
\State{}Sample decision $\act\ind{t}\sim{}p\ind{t}$ and update estimation
oracle with $(\act\ind{t},r\ind{t}, \obs\ind{t})$.
\EndFor
\end{algorithmic}
\caption{Optimistic \etdtext (\etdopt)}
\label{alg:main}
\end{algorithm}

We provide an \emph{optimistic} variant of the \etd meta-algorithm (\etdopt) in
\pref{alg:main}. At each timestep $t$, the algorithm calls the estimation oracle to
obtain a randomized estimator $\mu\ind{t}$ using the data
$(\pi\ind{1},r\ind{1},o\ind{1}),\ldots,(\pi\ind{t-1},r\ind{t-1},o\ind{t-1})$ collected so
far. The algorithm then uses the estimator to compute a distribution
$p\ind{t}\in\Delta(\Pi)$ and samples $\pi\ind{t}$ from this
distribution, with the main change relative to \pref{alg:etd} being that the minimax problem in \pref{alg:main} is
derived from an ``optimistic'' variant of the \CompShort, which
we refer to as the \emph{Optimistic \CompText}. For $\mu\in\Delta(\cM)$, define
\begin{align}
  \ocompD(\cM,\mu)
  =
  \inf_{p\in\Delta(\Pi)}\sup_{M\in\cM}\En_{\pi\sim{}p}\En_{\Mbar\sim\mu}\brk*{
  \fmbar(\pimbar) - \fm(\pi)  - \gamma\cdot{}\Dgenpi{\Mbar}{M}
  }.
  \label{eq:optimistic_dec}
\end{align}
and \neurips{$\ocompD(\cM) = \sup_{\mu\in\Delta(\cM)}\ocompD(\cM,\mu)$.}
\arxiv{\begin{align}
  \label{eq:optimistic_max}
  \ocompD(\cM) = \sup_{\mu\in\Delta(\cM)}\ocompD(\cM,\mu).
\end{align}}
The Optimistic \CompShort has two difference from the \vanilla \CompShort. First, it is parameterized by a
distribution $\mu\in\Delta(\cM)$ rather than a reference model
$\Mbar:\Pi\to\Delta(\cR\times\cO)$, which reflects the use of
randomized estimators; the value in \pref{eq:optimistic_dec} takes
the expectation over a reference model $\Mbar$ drawn from this
distribution (this modification also appears in the randomized
\CompShort introduced in \citet{foster2021statistical}). Second, and more critically, the optimal value $\fm(\pim)$ in
\pref{eq:comp_general} is replaced by the optimal value $\fmbar(\pimbar)$ for the
(randomized) reference model. This seemingly small change is the main
advantage of incorporating optimistic
estimation, and makes it possible to bound the Optimistic
\CompShort for certain divergences $D$ for which the value of the
unmodified \CompShort would otherwise be unbounded
(cf. \pref{sec:optimistic-estimation-help}).\loose

\begin{remark}
  When the divergence $D$ admits a sufficient statistic
  $\suffmap:\cM\to\Psi$, for any distribution $\mu\in\Delta(\cM)$, if
  we define $\nu\in\Delta(\Psi)$ via
  $\nu(\psi) = \mu(\crl{M\in\cM: \suffmap(M)=\psi})$, we have
  \[
    \ocompD(\cM,\mu) =
    \inf_{p\in\Delta(\Pi)}\sup_{M\in\cM}\En_{\pi\sim{}p}\En_{\suff\sim\nu}\brk*{
      \fsuff(\pisuff) - \fm(\pi) - \gamma\cdot{}\Dgenpi{\suff}{M} }.
  \]
  In this case, by overloading notation slightly, \neurips{we may
    write $\ocompD(\cM) = \sup_{\nu\in\Delta(\Psi)}\ocompD(\cM,\nu)$.}\arxiv{we may simplify the
  definition in \pref{eq:optimistic_max} to
  \[
    \ocompD(\cM) = \sup_{\nu\in\Delta(\Psi)}\ocompD(\cM,\nu).
      \vspace{-15pt}
  \]}
\end{remark}

\arxiv{\paragraph{Main result}}
\neurips{\fakepar{Main result}}
Our main result shows that the regret of \etdopttext is controlled by
the Optimistic \CompShort and the optimistic estimation error for the
oracle $\AlgEst$.
\begin{theorem}
  \label{thm:main}
For any $\delta>0$, \pref{alg:main} ensures that with probability at
least $1-\delta$,
  \begin{align}
    \label{eq:main_single_sample}
    \RegDM \leq \ocompD(\cM)\cdot{}T + \gamma\cdot{}\EstOptDFull.
  \end{align}
\end{theorem}
This regret bound has the same structure as \pref{thm:prev}, with the
\CompShort and estimation error replaced by their optimistic
counterparts. In the remainder of the paper, we show that 1) by adopting \emph{asymmetric} divergences
specialized to reinforcement learning, this result leads to the first
guarantees for model-free RL with \etd, but 2) for
symmetric divergences such as Hellinger distance, the result never
improves upon \pref{thm:prev}.

\arxiv{\paragraph{Estimation with batching}}
\neurips{\fakepar{Estimation with batching}} \neurips{
  For our application to reinforcement learning, we
generalize the results above to accomodate estimation algorithms
that draw \emph{batches} of multiple samples from each distribution
$p\ind{t}$. Given a \emph{batch size} $n$, we break the
$T$ rounds of the decision making protocol into $K\ldef{}T/n$
contiguous epochs (or, ``iterations''). Within each epoch, the
learner's distribution $p\ind{k}$ is unchanged (we index by $k$ rather
than $t$ to reflect this), and we create a \emph{batch}
$B\ind{k}=\crl*{(\pi\ind{k,l},r\ind{k,l},o\ind{k,l})}_{l=1}^{n}$ by
sampling $\pi\ind{k,l}\sim{}p\ind{k}$ independently and observing
$(r\ind{k,l},o\ind{k,l})\sim\Mstar(\pi\ind{k,l})$ for each
$l\in\brk{n}$. We can then appeal to estimation algorithms of
the form $\mu\ind{k}=\AlgEst\ind{k}\prn[\big]{
  \crl*{B\ind{k}}_{i=1}^{k-1}}$. Regret bounds for a variant of \etdopt with batching are given in \cref{sec:batched}.
}

\arxiv{
\arxiv{\begin{algorithm}[t]}
  \neurips{\begin{algorithm}[htp]}
    \setstretch{1.3}
     \begin{algorithmic}[1]
       \State \textbf{parameters}:
       \Statex[1] Online estimation oracle $\AlgEst$ with batch size $n$.
       \Statex[1] Exploration parameter $\gamma>0$.
       \Statex[1] Divergence $\Dgen{\cdot}{\cdot}$ with sufficient
       statistic space $\Suff$.
       \State Let $K\ldef{}T/n$.
  \For{$k=1, 2, \cdots, K$}
  \State Receive randomized estimator $\mu\ind{k}\in\Delta(\Suff) = \AlgEst\ind{t}\prn[\big]{(B\ind{i})_{i=1}^{k-1} }$.
\State \mbox{\arxiv{Compute}\neurips{Get}
$p\ind{k}\gets\argmin_{p\in\Delta(\Act)}\sup_{M\in\cM}\En_{\act\sim{}p}\En_{\suffhat\sim\mu\ind{k}}\brk*{\fsuffhat(\pisuffhat)-\fm(\pi)
    -\gamma\cdot\Dgenpi{\suffhat}{M}}$.
  \algcommentlight{Eq. \pref{eq:optimistic_dec}.}}
\State{}\multiline{Sample batch
$B\ind{k}=\crl*{(\pi\ind{k,l},r\ind{k,l},o\ind{k,l})}_{l=1}^{n}$ where
$\pi\ind{k,l}\sim{}p\ind{k}$ and
$(r\ind{k,l},o\ind{k,l})\sim\Mstar(\pi\ind{k.l})$, and update estimation
oracle with $B\ind{k}$.}
\EndFor
\end{algorithmic}
\caption{Optimistic \etdtext (\etdopt) with Batching}
\label{alg:main_batched}
\end{algorithm}

For our application to reinforcement learning, it will be useful to
generalize \etdopt to accomodate estimation algorithms
that draw \emph{batches} of multiple samples from each distribution
$p\ind{t}$. Given a \emph{batch size} $n$, we break the
$T$ rounds of the decision making protocol into $K\ldef{}T/n$
contiguous epochs (or, ``iterations''). Within each epoch, the
learner's distribution $p\ind{k}$ is unchanged (we index by $k$ rather
than $t$ to reflect this), and we create a \emph{batch}
$B\ind{k}=\crl*{(\pi\ind{k,l},r\ind{k,l},o\ind{k,l})}_{l=1}^{n}$ by
sampling $\pi\ind{k,l}\sim{}p\ind{k}$ independently and observing
$(r\ind{k,l},o\ind{k,l})\sim\Mstar(\pi\ind{k,l})$ for each
$l\in\brk{n}$. We consider estimation oracles of the form
\neurips{$\mu\ind{k}=\AlgEst\ind{k}\prn*{ \crl*{B\ind{k}}_{i=1}^{k-1}}$,}
\arxiv{\[
\mu\ind{k}=\AlgEst\ind{k}\prn*{ \crl*{B\ind{k}}_{i=1}^{k-1}},
\]}
and measure estimation error via
\begin{equation}
  \EstOptD \ldef{}
\sum_{k=1}^{K}\En_{\act\ind{k}\sim{}p\ind{k}}\En_{\Mhat\ind{k}\sim\mu\ind{k}}\brk*{\Dgenpi{\Mhat\ind{k}}{\Mstar}
  +
  \gamma^{-1}(\fmstar(\pimstar)-\fmhatt(\pi\subs{\Mhat\ind{k}}))}.\label{eq:estimation_batched}
\end{equation}
We assume that the estimation oracle ensures that with probability at least $1-\delta$,
$\EstOptD\leq{}\EstOptDFullKn$, where $\EstOptDFullKn$ is a known
upper bound.

\pref{alg:main_batched} is a variant of \etdopt that incorporates
batching. The algorithm updates the distribution $p\ind{k}$ in the
same fashion as its non-batched counterpart, but does so only at the
beginning of each epoch. The main guarantee for this method as follows.
\begin{theorem}
  \label{thm:main_batched}
  Let $T\in\bbN$ be given, and let $n$ be the batch size. For any
  $\delta>0$, \pref{alg:main_batched} ensures that with probability at
  least $1-\delta$,
  \neurips{$    \RegDM \leq \ocompD(\cM)\cdot{}T + \gamma{}n\cdot{}\EstOptD(T/n,n,\delta)$.}
\arxiv{\begin{align}
  \label{eq:main_multi_sample}
    \RegDM \leq \ocompD(\cM)\cdot{}T + \gamma{}n\cdot{}\EstOptD(T/n,n,\delta).
       \end{align}
       }
\end{theorem}
\neurips{See \cref{app:main_proofs} for the proof.}
When working with divergences for which unbiased estimates are
unavailable, this approach can lead to stronger guarantees than
\pref{thm:main}. We refer to the proof of
\pref{prop:estimation_bilinear} for a concrete example.

 }

\subsection{Application to Model-Free Reinforcement Learning}
\label{sec:rl}

In this section, we use \etdopttext to provide sample-efficient
guarantees for model-free reinforcement learning with \emph{bilinear classes}
\citep{du2021bilinear}, a general class of tractable reinforcement
learning problems which encompasses many standard settings \citep{jiang2017contextual,sun2019model,dean2020sample,yang2019sample,jin2020provably,modi2020sample,ayoub2020model,krishnamurthy2016pac,du2019latent,li2009unifying,dong2019provably,zhou2021nearly}.

\arxiv{\subsubsection{Reinforcement Learning Preliminaries}}
\neurips{\fakepar{Reinforcement learning preliminaries}} To state our results, let us recall how reinforcement learning fits into
the \FrameworkShort framework. We consider an episodic, finite-horizon reinforcement
  learning setting. With $H$ denoting the horizon, each model $M\in\cM$ specifies a non-stationary Markov decision process
  \arxiv{$$M=\crl*{\cS, \cA, \crl{\Pm_h}_{h=1}^{H}, \crl{\Rm_h}_{h=1}^{H},
      d_1},$$}
  \neurips{$M=\crl*{\cS, \cA, \crl{\Pm_h}_{h=1}^{H}, \crl{\Rm_h}_{h=1}^{H},
    d_1},$}
  where $\cS$ is the state space, $\cA$ is the action space,
  $\Pm_h:\cS\times\cA\to\Delta(\cS)$ is the probability transition
  distribution at step $h$, $\Rm_h:\cS\times\cA\to\Delta(\bbR)$ is
  the reward distribution, and $d_1\in\Delta(\cS_1)$ is the initial
  state distribution. We allow
 the reward distribution and transition
  kernel to vary across models in $\cM$, but assume that the initial
  state distribution is fixed.

    For a fixed MDP $M\in\cM$, each episode proceeds under the
  following protocol.
  At the beginning of the episode, the learner selects a
  randomized, non-stationary \emph{policy}
  $\pi=(\pi_1,\ldots,\pi_H),$ where $\pi_h:\cS\to\Delta(\cA)$; we
  let $\PiRNS$ (for ``randomized, non-stationary'') denote the set of all such policies. The episode
  then evolves through the following process, beginning from
  $s_1\sim{}d_1$: \neurips{For $h=1,\ldots,H$: $a_h\sim\pi_h(s_h)$, $r_h\sim\Rm_h(s_h,a_h)$, and $s_{h+1}\sim{}P\sups{M}_h(\cdot\mid{}s_h,a_h)$.}
  \arxiv{For $h=1,\ldots,H$,
  \begin{itemize}
  \item $a_h\sim\pi_h(s_h)$.
  \item $r_h\sim\Rm_h(s_h,a_h)$ and $s_{h+1}\sim{}P\sups{M}_h(\cdot\mid{}s_h,a_h)$.
  \end{itemize}
  }
  For notational convenience, we take $s_{H+1}$ to be a deterministic terminal
  state. We assume for simplicity that $\cR\subseteq\brk{0,1}$ (that is,
$\sum_{h=1}^{H}r_h\in\brk{0,1}$ almost surely). Within the \FrameworkShort framework, at each time $t$, the learning
agent chooses $\pi\ind{t}\in\PiRNS$, then observes the cumulative reward
$r\ind{t}=\sum_{h=1}^{H}r_h\ind{t}$ and trajectory
$o\ind{t}\ldef{}(s_1\ind{t},a_1\ind{t},r_1\ind{t}),\ldots,
(s_H\ind{t},a_H\ind{t},r_H\ind{t})$ that results from
\neurips{executing $\pi\ind{t}$.}\arxiv{running the
policy for a single episode.}

\arxiv{\paragraph{Value functions}}
\neurips{\paragraphi{Value functions}}
  The value for a policy $\pi$ under $M$ is given by $\fm(\pi)\ldef\Ens{M}{\pi}\brk[\big]{\sum_{h=1}^{H}r_h}$,
where
  $\Ens{M}{\pi}\brk{\cdot}$ denotes expectation under the process
  above. 
  For a given model $M$ and policy $\pi$, we define the state-action
  value function and state value functions via
  \neurips{$Q_h^{\sss{M},\pi}(s,a)=\En^{\sss{M},\pi}\brk[\big]{\sum_{h'=h}^{H}r_{h'}\mid{}s_h=s,
  a_h=a}$, and $V_h^{\sss{M},\pi}(s)=\En^{\sss{M},\pi}\brk[\big]{\sum_{h'=h}^{H}r_{h'}\mid{}s_h=s}$.
}
\arxiv{\[
Q_h^{\sss{M},\pi}(s,a)=\En^{\sss{M},\pi}\brk*{\sum_{h'=h}^{H}r_{h'}\mid{}s_h=s,
  a_h=a},
\mathand V_h^{\sss{M},\pi}(s)=\En^{\sss{M},\pi}\brk*{\sum_{h'=h}^{H}r_{h'}\mid{}s_h=s}.
\]}
We define $\pim$ as the optimal policy, which maximizes
$Q_h^{\sss{M},\pim}(s,a)$ for all states simultaneously. We abbreviate $\Qmstar\equiv{}Q^{\sss{M},\pim}$.

\arxiv{\paragraph{Value function approximation}}
\neurips{\paragraphi{Value function approximation}}
To apply our results to reinforcement learning, we take a model-free (or, value
function approximation) approach, and estimate value functions for the
underlying MDP $\Mstar$; this contrasts with model-based methods, such
as those considered in \citet{foster2021statistical}, which estimate transition
probabilities for $\Mstar$ directly. We assume access to a
class $\cQ$ of value functions of the form $Q=(Q_1,\ldots,Q_H)$\neurips{.}\arxiv{, and
make the following realizability assumption.}
\begin{assumption}
  \label{ass:value_realizability}
  The value function class $\cQ$ has $\Qmstarstar\in\cQ$, where
  $\Mstar$ is the underlying model.
\end{assumption}
For $Q=(Q_1,\ldots,Q_H)\in\cQ$, we
define $\piq=(\pi_{Q,1},\ldots,\pi_{Q,H})$ via
$\pi_{Q,h}(s)=\argmax_{a\in\cA}Q_h(s,a)$. We define
$\Piq=\crl*{\piq\mid{}Q\in\cQ}$ as the induced policy class. While \arxiv{it
will not be necessary for the results that follow}\neurips{is not
necessary for our results}, we mention in
passing that the class $\cQ$, under \pref{ass:value_realizability}, implicitly
induces a model class via $\cM_{\cQ} \ldef \crl*{M \mid{} \Qmstar\in\cQ}$.\loose
\arxiv{\subsubsection{Bilinear Classes}}

\neurips{\noindent\emph{Bilinear classes.}}
The bilinear class framework \citep{du2021bilinear} gives structural conditions for sample-efficient reinforcement learning that capture
most known settings where tractable guarantees are possible. The
following is an adaptation of the definition from \citet{du2021bilinear}.\footnote{For the sake of simplicity, we adopt a less general
  definition than \citet{du2021bilinear}: We 1) assume that the ``hypothesis class'' is parameterized by the $Q$-function class
  $\cQ$, and 2) limit to discrepancy functions that do not explicitly
  depend on the function $Q$ indexing the factor $X_h(Q;M)$. The results here readily extend to the full definition.}

\begin{definition}[Bilinear class]
  \label{def:bilinear}
\arxiv{   An MDP $M$ is said to be bilinear with dimension $d$ relative to a value function
  class $\cQ$ if:}
\neurips{   An MDP $M$ is said to be bilinear with dimension $d$ relative to a class $\cQ$ if:}
  \begin{enumerate}
  \item There exist functions $W_h(\cdot\midsem{}M):\cQ\to\bbR^{d}$,
    $X_h(\cdot\midsem{}M):\cQ\to\bbR^{d}$ such that for all
    $Q\in\cQ$ and $h\in\brk{H}$,
    \begin{equation}
      \label{eq:bilinear_residual}
\abs*{\Enmpi{M}{\piq}\brk*{
          Q_h(s_h, a_h) - r_h - \max_{a'\in\cA}Q_{h+1}(s_{h+1},a')
        }}
\leq\abs{\tri{X_h(Q;M),W_h(Q;M)}}.
    \end{equation}
  \item Let $z_h \ldef{} (s_h, a_h, r_h, s_{h+1})$. There exists a collection of estimation policies
    $\crl*{\piestq}_{Q\in\cQ}$ and a discrepancy function $\lest(\cdot;\cdot):\cQ\times\cZ\to\bbR$
    such that for all $Q, Q'\in\cQ$ and $h\in\brk{H}$,\footnote{For
      \arxiv{policies }$\pi$ and $\pi'$, $\pi\circ_h\pi'$ denotes the policy
      that follows $\pi$ for layers $1,\ldots,h-1$ and follows $\pi'$
      for layers $h,\ldots,H$.}
    \begin{equation}
      \label{eq:bilinear_tv}
      \abs{\tri{X_h(Q;M), W_h(Q';M)}}
      = \abs*{\Enmpi{M}{\pi^{\vphantom{\mathrm{est}}}_{Q}\circ_{h}\piest_{Q}}\brk*{
        \lest(Q';z_h)}
        }.
      \end{equation}
      If $\piestq=\piq$, we say that estimation is
      \emph{on-policy}. We assume 
      $\abs{\Enmpi{M}{\pi}\brk*{\lest(\Qmstar;z_h)}}=0$ for all $\pi$.
    \end{enumerate}
    We let $\dimbi(\cQ\midsem{}M)$ denote the minimal dimension $d$ for which the
    bilinear property holds for $M$. For a model class
    $\cM$, we define $\dimbi(\cQ;\cM) =
    \sup_{M\in\cM}\dimbi(\cQ\midsem{}M)$. We let
    $\Lbi(\cQ;M)\geq{}1$ denote any almost sure upper bound on
    $\abs{\lest(Q;z_h)}$ under $M$.
\end{definition}

\subsection{Guarantees for Bilinear Classes}
We now apply our main results to derive regret bounds for bilinear classes.
We first
provide optimistic estimation guarantees, then bound the Optimistic 
\CompShort, and conclude by \arxiv{combining these results with
  \pref{thm:main_batched}.}\neurips{applying \etdopt.}

We take $\Psi=\cQ$ as the sufficient
statistic space, with $\suffmap(M)\ldef{}\Qmstar$, and 
define $\fq(\piq)\ldef\En_{s_1\sim{}d_1}\brk*{Q_1(s_1,\piq(s_1))}$. For the divergence
$\Dgenshort$, we appeal to squared discrepancy, in the vein of \citet{jiang2017contextual,du2021bilinear}:
\begin{align}
  \Dbipi{Q}{M}
  = \sum_{h=1}^{H}\prn*{\Enmpi{M}{\pi}\brk*{
  \lest(Q;z_h)}}^2.
  \label{eq:discrepancy}
\end{align}
We abbreviate $\EstOptBi=\EstOptBilong$ and $\ocompbi(\cM,\mu) =
\odec^{\Dbishort}_{\gamma}(\cM,\mu)$.

\neurips{
\fakepar{Estimation} To perform estimation, we approximate the average discrepancy
in \pref{eq:discrepancy} from samples (drawing a batch of $n$ samples
at each step; cf. \cref{alg:main_batched}), then appeal to the exponential
weights method for online learning, with a bonus to enforce
optimism. See \cref{alg:bilinear_est} (deferred to \cref{app:main}
for space).
}

\arxiv{
\paragraph{Estimation}
To perform estimation, we adopt a batched approach
(cf. \pref{alg:main_batched}). We approximate the average discrepancy
in \pref{eq:discrepancy} from samples, then appeal to the exponential
weights method for online learning, with a bonus to enforce optimism.

Let the batch size $n$ be fixed. Recall that $K\ldef{}T/n$ is the
number of epochs, and that at the beginning of each epoch $k\in\brk{K}$, the estimation
oracle is given a new batch of examples $B\ind{k-1}=\crl*{(\pi\ind{k,l},r\ind{k,l},o\ind{k,l})}_{l=1}^{n}$ where
$\pi\ind{k,l}\sim{}p\ind{k}$ and
$(r\ind{k,l},o\ind{k,l})\sim\Mstar(\pi\ind{k,l})$; for RL, each observation
(trajectory) takes the form
$o\ind{k,l}=(s_1\ind{k,l},a_1\ind{k,l},r_1\ind{k,l}),\ldots, 
(s_H\ind{k,l},a_H\ind{k,l},r_H\ind{k,l})$. For each epoch $k$, we compute a loss function
\neurips{$
  \ls\ind{k}(Q) \ldef \sum_{h=1}^{H}\prn*{\frac{1}{n}\sum_{l=1}^{n}\lest(Q;z_h\ind{k,l})}^2
  - \frac{1}{8\gamma}\cdot{}\fq(\piq)$,
}
\arxiv{\begin{equation}
  \ls\ind{k}(Q) \ldef \sum_{h=1}^{H}\prn*{\frac{1}{n}\sum_{l=1}^{n}\lest(Q;z_h\ind{k,l})}^2
  - \frac{1}{8\gamma}\cdot{}\fq(\piq),
  \label{eq:est_loss}
\end{equation}
}
where
$z_h\ind{k,l}\ldef{}(s_h\ind{k,l},a_h\ind{k,l},r_h\ind{k,l},s_{h+1}\ind{k,l})$;
this loss approximates \pref{eq:discrepancy}, but incorporates a
bonus term $ - \frac{1}{8\gamma}\cdot{}\fq(\piq)$. We then form the randomized
estimator $\mu\ind{k}$ via
\neurips{$\mu\ind{k}(Q)
  \propto\exp\prn*{
  -\eta\sum_{i<k}\ls\ind{i}(Q)
  }$,
       }
\arxiv{\begin{align}
  \mu\ind{k}(Q)
  \propto\exp\prn*{
  -\eta\sum_{i<k}\ls\ind{i}(Q)
  },\label{eq:exp_weights}
       \end{align}
     }where $\eta>0$ is a learning rate.
   }

\begin{proposition}
  \label{prop:estimation_bilinear}
  For any batch size $n$ (with $K\ldef{}T/n$) and parameter
  $\gamma\geq{}1
  $, \cref{alg:bilinear_est}, with an appropriate learning
  rate, ensures that with probability at
  least $1-\delta$, 
    \begin{align}
  \EstOptBi  \approxleq{}
     \frac{\sqrt{K\log\abs{\cQ}}}{\gamma}
     + 
      H\Lbis(\cQ;\Mstar)\log(\abs{\cQ}KH/\delta)\prn*{1+\frac{K}{n}},
      \label{eq:bilinear_est}
    \end{align}
    whenever $\Mstar$ is bilinear relative to $\cQ$ and
    \pref{ass:value_realizability} is satisfied.   \end{proposition}
  This result does not actually use the bilinear class
  structure, and gives a bound on \cref{eq:discrepancy} for any choice
  of $\lest$. Similar to the optimistic estimation result given by
\citet{zhang2022feel} for contextual bandits, this guarantee consists
of ``slow'' term $     \frac{\sqrt{K\log\abs{\cQ}}}{\gamma}$ resulting
from optimism (which decays as $\gamma$ grows), and a ``fast''
term. However, compared to the contextual bandit setting, the fast
term,
$\log(\abs{\cQ}KH/\delta)\prn*{1+\frac{K}{n}}$,
scales with the ratio $\frac{K}{n}$, which reflects sampling error in
the estimated discrepancy\arxiv{ \pref{eq:est_loss}}, and necessitates a large
batch size. Previous algorithms for bilinear classes
\citep{jiang2017contextual,du2021bilinear} require large batch sizes
for similar reasons (specifically, because the expectation in
\pref{eq:discrepancy} appears inside the square, it is not possible to
form an unbiased estimate for $\Dbishort$ directly).

\arxiv{\paragraph{Bounding the \CompShort}}\neurips{\fakepar{Bounding the \CompShort}} A bound on the \newcomp follows by adapting arguments in
\citet{foster2021statistical}; our result has slightly improved
dependence on $H$ compared to the bounds \neurips{in that work.}\arxiv{on the
  \vanilla \CompShort found in that work.%
  }
\begin{proposition}
  \label{prop:dec_bilinear}
  Let $\cM$ be any model class for which the bilinear class property
  holds relative to $\cQ$. \textbf{1.} In the on-policy case where $\piestq=\piq$,
  we have that for all $\gamma>0$, \neurips{$    \ocompbi(\cM) \approxleq{} \frac{H\cdot{}\dimbi(\cQ;\cM)}{\gamma}$.}
  \arxiv{\begin{align}
    \ocompbi(\cM) \approxleq{} \frac{H\cdot{}\dimbi(\cQ;\cM)}{\gamma}.
  \end{align}}\\
\textbf{2.} In the general case ($\piestq\neq\piq$),
  we have that for all $\gamma\geq{}H^2\dimbi(\cQ;\cM)$,
    \neurips{$\ocompbi(\cM) \approxleq{} \sqrt{\frac{H^2\cdot{}\dimbi(\cQ;\cM)}{\gamma}}$.
    }
  \arxiv{\begin{align}
    \ocompbi(\cM) \approxleq{} \sqrt{\frac{H^2\cdot{}\dimbi(\cQ;\cM)}{\gamma}}.
         \end{align}
         }
\end{proposition}

\arxiv{\paragraph{Putting everything together}}

Combining \pref{thm:main_batched}, \pref{prop:estimation_bilinear}, and
\pref{prop:dec_bilinear}, we obtain the following result.

\begin{corollary}[Regret bound for bilinear classes]
  \label{cor:bilinear}
Let $\cQ$ be given. Assume that $\Mstar\in\cM$, where $\cM$ is
bilinear relative to $\cQ$, and that
\pref{ass:value_realizability} holds. Abbreviate $d\equiv\dimbi(\cQ;\cM)$
and $L\equiv\Lbi(\cQ;\cM)\ldef{}\sup_{M\in\cM}\Lbi(\cQ;M)$. For an appropriate choice of $n$ and
$\gamma$, \pref{alg:main_batched}, using
the algorithm from \pref{prop:estimation_bilinear} as an oracle, enjoys
the following guarantees with probability at least $1-\delta$:\\
\neurips{\textbf{1.} In the on-policy case where $\piestq=\piq$: \neurips{$    \RegDM \approxleq{}
  (H^2dL^2\log(\abs{\cQ}TH/\delta))^{1/2}T^{3/4}$.}\\
\textbf{2.} In the general case where $\piestq\neq\piq$: \neurips{$
  \RegDM \approxleq{}
  (H^6d^2L^4\log^3(\abs{\cQ}TH/\delta))^{1/6}T^{5/6}$.}
  }
\arxiv{\begin{itemize}
\item In the on-policy case where $\piestq=\piq$: \neurips{$    \RegDM \approxleq{}
    (H^2dL^2\log(\abs{\cQ}TH/\delta))^{1/2}T^{3/4}$.}
  \arxiv{\begin{align}
    \RegDM \approxleq{}
    (H^2dL^2\log(\abs{\cQ}TH/\delta))^{1/2}T^{3/4}.
    \label{eq:bilinear1}
  \end{align}}
\item In the general case where $\piestq\neq\piq$: \neurips{$    \RegDM \approxleq{} (H^6d^2L^4\log^3(\abs{\cQ}TH/\delta))^{1/6}T^{5/6}$.}
  \arxiv{
    \begin{align}
    \RegDM \approxleq{} (H^6d^2L^4\log^3(\abs{\cQ}TH/\delta))^{1/6}T^{5/6}.
    \end{align}
    }
  \end{itemize}
  }
\end{corollary}
This is the first regret bound for model-free reinforcement learning
with the \etdtext meta-algorithm. Importantly, the result scales only
with the horizon, the \arxiv{bilinear }dimension $d$, and the capacity
$\log\abs{\cQ}$\arxiv{ for the value function class}. Improving the dependence
on $T$ in \pref{cor:bilinear} is an interesting question: currently, there are no
algorithms for general bilinear classes that achieve $\sqrt{T}$ regret
without additional assumptions.
Let us emphasize that regret bounds for
bilinear classes can already be achieved by a number of existing
methods \citep{du2021bilinear,jin2021bellman}. The contribution here is to show that such guarantees can be achieved through the general
\CompShort framework, thereby placing this line of research on
stronger foundations.

\arxiv{
  \subsection{Tighter Guarantees under Bellman Completeness}
  \label{sec:complete}
\arxiv{In the remainder of the section}\neurips{In this section}, we show how to derive tighter
estimation guarantees (and consequently tighter regret bounds) when
$\cQ$ satisfies a \emph{Bellman completeness} assumption \citep{,zanette2020learning,jin2021bellman}. For a given
MDP $M$, let $\cTm_h\brk*{f}(s,a) \ldef \En\sups{M}\brk*{r_h +
  \max_{a'}f(s_{h+1},a')\mid{}s_h=s,a_h=a}$ denote the Bellman
operator for layer $h$.
\begin{assumption}[Completeness]
  We assume that $\cQ=\cQ_1\times\cdots\times\cQ_H$ is a product
  class, and that for all $h$ and $Q_{h+1}\in\cQ_{h+1}$,
  $\brk{\cTmstar_hQ_{h+1}}\in\cQ_{h}$.
\end{assumption}
As before, we take $\Psi=\cQ$, $\suffmap(M)\ldef{}\Qmstar$, and $\fq(\piq)\ldef\En_{s_1\sim{}d_1}\brk*{Q_1(s_1,\piq(s_1))}$. For the divergence
$\Dgenshort$, we now appeal to \emph{squared Bellman error} (e.g., \cite{jin2021bellman,xie2022role}):
\begin{align}
  \Dsbepi{Q}{M}
  = \sum_{h=1}^{H}\Enmpi{M}{\pi}\brk*{\prn*{Q_h(s_h,a_h)-\brk{\cT\sups{M}_hQ_{h+1}}(s_h,a_h)}^2}.
  \label{eq:discrepancy}
\end{align}
We abbreviate $\EstOptSB=\EstOptSBlong$ and $\ocompsbe(\cM,\mu) =
\odec^{\Dsbeshort}_{\gamma}(\cM,\mu)$.

\paragraph{Estimation}
\cref{alg:ts3_rl} %
performs optimistic online estimation with squared Bellman error. The
algorithm is an adaptation of a two-timescale exponential weights strategy
originally introduced by \citet{agarwal2022non} within an optimistic posterior
sampling algorithm referred to as \tst. We show that this technique
leads to a self-contained online estimation guarantee
outside the context of the \tst algorithm.

\begin{proposition}[Estimation for square Bellman error]
  \label{prop:ts3_rl}
  Assume that $\cQ$ satisfies completeness relative to $M^{\star}$. Moreover assume $\sum_{h=1}^{H}\sup_{s,a} r_h(s,a) \leq 1$ and $\sup_{Q,h,s,a} Q_h(s,a) \leq 1$.
  Then for any $\gamma\geq{}1
  $ and $\eta\leq 1/(2^{16}(\log (|\cQ|K/\delta) + 1))$
  , with batch size $n=H$
  ($K\ldef{}T/n$), $\lambda=1/8$,  $\beta=(12\gamma H)^{-1}$ and $\delta>0$, 
  \cref{alg:ts3_rl} ensures that with probability at least $1-\delta$,
    \begin{align}
      \EstOptSB  \approxleq{}
      \frac{H\log\abs{\cQ}}{\eta} + \frac{\eta\log (\abs{\cQ}K/\delta)K}{\gamma} + \frac{K}{\gamma^2 H}.
    \end{align}
    whenever $\cQ$ satisfies completeness relative to
    $\Mstar$.\footnote{\citet{agarwal2022non} give a tighter estimation
      error bound of roughly $      \EstOptSB  \approxleq{}
      \log^2(\abs{\cQ}HK) + \frac{K}{\gamma^2}$, but this result takes advantage of a Bellman rank assumption on the underlying MDP. The estimation error bound we state here does not require any structural assumptions on the MDP under consideration, but gives a worse rate.
      }
  \end{proposition}
Note that \cref{prop:ts3_rl} does not make use of the bilinear class
assumption, and only requires that $\cQ$ satisfies completeness. \neurips{As
such, we expect that this result will find use more broadly.}

\paragraph{Regret bound for bilinear classes}
To provide regret bounds, we assume that $\Mstar$ satisfies the
bilinear class property relative to $\cQ$ as in \cref{sec:rl}. In addition to
assuming that $\Mstar$ is bilinear, we make the following
restrictions: (1) $\piestq=\piq$, i.e. estimation is on policy, (2)
$\lest(Q;z_h)=Q_h(s_h,a_h)-r_h-\max_{a'}Q_{h+1}(s_{h+1},a')$, so that 
$\Enmpi{M}{\pi}\brk*{\lest(Q;z_h)}$ is the average
Bellman error for $Q$ under $M$.\footnote{These restrictions
  correspond to restricting attention to \emph{$Q$-type Bellman rank}, a
  special case of the bilinear class property \citep{jin2021bellman,foster2021statistical}.}
  With this discrepancy function,
Jensen's inequality implies that $\ocompsbe(\cM)\leq{}\ocompbi(\cM)$,
so combining \pref{thm:main_batched}, \pref{prop:ts3_rl}, and
\pref{prop:dec_bilinear}, we obtain the following result.
\begin{corollary}[Regret bound under completeness]
  \label{cor:regret_complete}
Let $\cQ$ be given. Assume that $\Mstar\in\cM$, where $\cM$ is
bilinear relative to $\cQ$, and that
completeness holds. Moreover assume $\sum_{h=1}^{H}\sup_{s,a} r_h(s,a) \leq 1$ and $\sup_{Q,h,s,a} Q_h(s,a) \leq 1$. Abbreviate $d\equiv\dimbi(\cQ;\cM)$. For an appropriate choice of $n$ and
$\gamma$, \pref{alg:main_batched}, using \cref{alg:ts3_rl} (with
appropriate parameter choice) as an oracle, enjoys
the following guarantees with probability at least $1-\delta$:
\begin{align}
\RegDM \approxleq{} H d^{1/3} (\log(\abs{\cQ}K/\delta))^{4/5}T^{2/3} .
  \end{align}
\end{corollary}
This improves upon the $T^{3/4}$-type rate in \cref{cor:bilinear}.

  \begin{algorithm}[tp]
      \caption{Two-Timescale Exponential Weights for Bellman Complete Value Function Classes}
\label{alg:ts3_rl}
    \begin{algorithmic}[1] 
    \State Initialize $\histSet\ind{0}=\emptyset$.
    \For{$k=1,\dots,K$}
    \State For any $Q,Q'\in \cQ$ and $h\in [H]$, define 
    \begin{gather*}
      \Delta_h\ind{k}(Q',Q) \ldef Q'(s_h\ind{k,h},a_h\ind{k,h}) - r_h\ind{k,h} - Q(s_{h+1}\ind{k,h}),\\
      q\ind{k}(Q'|Q) \ldef q\ind{k}(Q'|Q,\histSet\ind{k-1}) \propto \exp\prn*{-\lambda\cdot \frac{1}{H}\sum\limits_{s=1}^{k-1} \sum\limits_{h=1}^{H}\Delta_h\ind{s}(Q',Q)^2},\\
      L\ind{k}(Q) \ldef \frac{1}{H} \sum\limits_{h=1}^{H} \Delta_h\ind{k}(Q,Q)^2 + \frac{1}{\lambda} \log \En_{Q'\sim q\ind{k}(\cdot|Q)} \brk*{ \exp \prn*{-\lambda \cdot \frac{1}{H} \sum\limits_{h=1}^{H}\Delta_h\ind{k}(Q',Q)^2}  }    ,\\
      \mu\ind{k}(Q) \ldef  \mu\ind{k}(Q|\histSet\ind{k-1}) \propto \exp\prn*{ -\eta \sum\limits_{s=1}^{k-1} \prn*{ L\ind{s}(Q)  - \beta\cdot \frac{1}{H} \sum\limits_{h=1}^{H} \max_a Q(s_1\ind{s,h},a)} }.
    \end{gather*}
    \State Predict $\mu\ind{k}$.
    \For{$l=1,\dots,H$}
    \State \multiline{Play $\pi\ind{k,l}\sim p\ind{k}$ and obtain the trajectory $o\ind{k,l}=(s_1\ind{k,l},a_1\ind{k,l},r_1\ind{k,l}),\ldots, 
    (s_H\ind{k,l},a_H\ind{k,l},r_H\ind{k,l})$, where
    $p\ind{k}\in\Delta(\Pi)$ is a decision distribution produced by
    any batched algorithm (e.g., \cref{alg:main_batched}) that selects a
    decision adaptively based on $\mu\ind{k}$.}
    \EndFor
    \State Update $\histSet\ind{k} \gets \histSet\ind{k-1}\cup \bigcup_{l=1}^H\set{s_l\ind{k,l},a_l\ind{k,l},r_l\ind{k,l},s_{l+1}\ind{k,l}} \cup \bigcup_{l=1}^H \set{s_1\ind{k,l}}$.
    \EndFor
    \end{algorithmic}
    \end{algorithm}

\neurips{    
\subsection{Proofs from \cref{sec:complete}}

\cref{prop:ts3_rl} is an application of more general results given in
\cref{app:completeness}, which analyze a generalization of
\cref{alg:ts3_rl} for a more general online learning setting. To
\cref{prop:ts3_rl}, we simply apply these results to the reinforcement
learning framework.
\begin{proof}[\pfref{prop:ts3_rl}]
  Let the batch size $n=H$ be fixed, and let $K\ldef{}T/n$ be the number of epochs. Recall that for each step $k\in\brk{K}$, the estimation
  oracle is given a batch of examples $B\ind{k}=\crl*{(\pi\ind{k,l},r\ind{k,l},o\ind{k,l})}_{l=1}^{n}$ where
  $\pi\ind{k,l}\sim{}p\ind{k}$ and
  $(r\ind{k,l},o\ind{k,l})\sim\Mstar(\pi\ind{k,l})$. Each observation
  (trajectory) takes the form
  $o\ind{k,l}=(s_1\ind{k,l},a_1\ind{k,l},r_1\ind{k,l}),\ldots, 
  (s_H\ind{k,l},a_H\ind{k,l},r_H\ind{k,l})$. 
  We abbreviate $\Qstar=\Qmstarstar$.
  \paragraph{Estimation algorithm} For each step $k$, the randomized
  estimator $\mu\ind{k}$ selected as described in
  \pref{alg:ts3_rl}. This algorithm is an instantiation of
  \pref{alg:ts3} in the general online learning setting described in
  \cref{app:completeness}, with $\Gclass = \cQ$ and for all $h\in
  [H]$, $\Xclass_h = \cS \times \cA$, $\Yclass_h = \bbR\times \cS$ and
  $\Zclass = \ngedit{\cS^H}$. The unknown kernels are the transition
  distributions for the corresponding layers of the MDP $\Mstar$, and the loss functions are
  \begin{align*}
  \loss_{h,1}((s_h, a_h), Q) &\ldef Q_h(s_h, a_h) ,\\
  \loss_{h,2}((r_h,s_{h+1}),Q) &\ldef r_h +  \max_a Q_h(s_{h+1},a),\\
  \loss_3(\set{s_1\ind{l}}_{l\in [H]} ) &\ldef  - \frac{1}{H} \sum\limits_{l=1}^{H} \max_a Q_1(s_1\ind{l},a) 
  \end{align*}
  Finally, take $\x_h\ind{k} = (s_h\ind{k,h}, a_h\ind{k,h})$,
  $\y_h\ind{k} = (r_h\ind{k,h}, s_{h+1}\ind{k,h})$ and $\z\ind{k} =
  \set{s_1\ind{k,h}}_{h\in [H]}$. \jqedit{It is important to note that
    $s_h\ind{k,h}, a_h\ind{k,h},r_h\ind{k,h}, s_{h+1}\ind{k,h}$ are
    taken from different trajectories for $h\in [H]$, so
    $\y_h\ind{k}\mid{}\x_h\ind{k}$ are independent from one other for
    $h\in [H]$.} \ngedit{Moreover, note that the distributions
    $p\ind{k} \in \Delta(\Pi)$ play the role of nature: the
    distribution of the tuple ($x_h\ind{k}, y_h\ind{k}, w\ind{k})$ for $h \in [H]$ is determined by running a policy $\pi\ind{k} \sim p\ind{k}$ in the ground-truth MDP $M^\star$.} With this configuration, observe that in the notation of \cref{app:completeness}, we have\ngedit{, for any $Q$,} 
  \begin{align*}
    \En_{\x_{1:H}\ind{k}}\cE(Q,Q,\x_{1:H}\ind{k})^2 &= \frac{1}{H} \sum\limits_{h=1}^{H}\En_{\x_h\ind{k}}(\loss_{h,1}(\x_h\ind{k},Q) - \En[\loss_{h,2}(\y_h,Q)|\x_h\ind{k}])^2 \\
    &= \frac{1}{H} \sum\limits_{h=1}^{H}\En_{\pi\ind{k,h}\sim p\ind{k}}\Enmpi{\Mstar}{\pi\ind{k,h}}\brk*{\prn*{Q_h(s_h,a_h)-\brk{\cT\sups{\Mstar}_hQ_{h+1}}(s_h,a_h)}^2} \\
    &= \frac{1}{H} \En_{\pi\sim p\ind{k}}\Enmpi{\Mstar}{\pi}\brk*{ \sum\limits_{h=1}^{H}\prn*{Q_h(s_h,a_h)-\brk{\cT\sups{\Mstar}_hQ_{h+1}}(s_h,a_h)}^2}\\
    &= \frac{1}{H} \En_{\pi\sim p\ind{k}} \Dsbepi{Q}{\Mstar}.
  \end{align*}
and 
\begin{align*}
  \En_{\x_{1:H}\ind{k},\z\ind{k}}\iota\ind{k}(Q) &= \frac{1}{H} \sum\limits_{l=1}^{H} \En_{\pi\ind{k,h}\sim p\ind{k}}\Enmpi{\Mstar}{\pi\ind{k,h}} \brk*{ \max_a Q_1^*(s_1\ind{k,l},a) - \max_a Q_1(s_1\ind{k,l},a)}\\
  &= \fmstar(\pimstar)-\fq(\piq) .
\end{align*}

  \paragraph{Estimation error bound} We take $\alpha = 12\coef$, so that \pref{thm:main-OL} implies that with probability at least $1-\delta$, 
  \begin{align*}
    &\sum\limits_{k=1}^{K}  \En_{Q\sim \mu\ind{k}} \prn*{\frac{1}{H} \En_{\pi\sim p\ind{k}} \Dsbepi{Q}{\Mstar} + \alpha (\fmstar(\pimstar)-\fq(\piq) )  } \\
    &\approxleq \eta\alpha\log(\abs{\cQ}K/\delta)K  + \log\abs{\cQ}/\eta  + \alpha^2 K.
  \end{align*}
  Then by taking $\alpha = \frac{1}{\gamma H}$, this further implies that with probability at least $1-\delta$,
  \begin{align*}
    \EstOptSB &= \sum\limits_{k=1}^{K}  \En_{\pi\sim p\ind{k}}\En_{Q\sim \mu\ind{k}} \prn*{ \Dsbepi{Q}{\Mstar} + \frac{1}{\gamma} (\fmstar(\pimstar)-\fq(\piq) )  } \\
    &\approxleq H( \eta\alpha\log(\abs{\cQ}K/\delta)K  + \log\abs{\cQ}/\eta  + \alpha^2 K)\\
    &\approxleq{} \frac{H\log\abs{\cQ}}{\eta} + \frac{\eta\log (\abs{\cQ}K/\delta)K}{\gamma} + \frac{K}{\gamma^2 H}.
  \end{align*}
\end{proof}

\begin{proof}[\pfref{cor:regret_complete}] We choose $n=H$ and apply \cref{alg:ts3_rl} as the estimation oracle. 
We first consider the ``trivial'' parameter regime in which $ H d^{1/3}(\log(\abs{\cQ}K/\delta))^{-1/5}T^{-1/3} \geq 1/(2^{16}(\log (|\cQ|K/\delta) + 1))$. Here, $T \approxleq H d^{1/3} (\log(\abs{\cQ}K/\delta))^{4/5}T^{2/3}$, and thus
 \begin{align*}
  \RegDM
  \approxleq{} H d^{1/3} (\log(\abs{\cQ}K/\delta))^{4/5}T^{2/3}.
 \end{align*}
 When the case above, does not hold, we proceed as in the theorem statement, choosing $\eta = H d^{1/3}(\log(\abs{\cQ}K/\delta))^{-1/5}T^{-1/3} \leq  1/(2^{16}(\log (|\cQ|K/\delta) + 1))$. Combining \pref{thm:main_batched} and  \pref{prop:ts3_rl} then gives %
  \begin{align*}
    \RegDM
    \approxleq{} \ocompsbe(\cM)\cdot{}T 
    + \gamma{}\frac{H^2\log\abs{\cQ}}{\eta} + \eta\log (\abs{\cQ}K/\delta)T + \frac{K}{\gamma^2}
  \end{align*}
  with probability at least $1-\delta$.
    Next, using \pref{prop:dec_bilinear} to bound $\ocompsbe(\cM)$ in the above display, it follows that 
    \begin{align*}
          \RegDM
      \approxleq{} \frac{HdT}{\gamma} + \gamma{}\frac{H^2\log\abs{\cQ}}{\eta} + \eta\log (\abs{\cQ}K/\delta)T + \frac{K}{\gamma^2} .
    \end{align*}
    We choose $\gamma  = d^{2/3} (\log (\abs{\cQ}K/\delta))^{-2/5} T^{1/3}$ to obtain 
    \begin{align*}
      \RegDM
      \approxleq{} H d^{1/3} (\log(\abs{\cQ}K/\delta))^{4/5}T^{2/3} 
    \end{align*}
    with probability at least $1-\delta$.

  \end{proof}

 }

   }

  \section{Understanding the Role of Optimistic Estimation}
\label{sec:discussion}

We close with discussion and interpretation of our results.

\arxiv{\subsection{When Does Optimistic Estimation Help?}}
\neurips{\noindent\textbf{When does optimistic estimation help?}}
\label{sec:optimistic-estimation-help}
Perhaps the most pressing question at this point is to understand when
the regret bound
for \etdopt (\pref{thm:main}) improves upon the corresponding regret
bound for vanilla \etd (\pref{thm:prev}). In what follows, we show
that: \neurips{(1) For any divergence $D$, the \newcomp is equivalent to a variant
  of the \vanilla \CompShort which incorporates randomized estimators
  \citep{foster2021statistical}, but with the arguments to the
  divergence \emph{flipped}; (2)
 For divergences $D$ that satisfy a triangle inequality, this
  randomized \CompShort is equivalent to the \vanilla \CompShort
  itself.}
\arxiv{
\begin{enumerate}
\item For any divergence $D$, the \newcomp is equivalent to a variant
  of the \vanilla \CompShort which incorporates randomized estimators
  \citep{foster2021statistical}, but with the arguments to the
  divergence \emph{flipped}.
\item For divergences $D$ that satisfy a triangle inequality, this
  randomized \CompShort is equivalent to the \vanilla \CompShort itself.
\end{enumerate}
}Together these results show that the improvement given by the \newcomp is
limited to \emph{asymmetric} divergences such as the bilinear
divergence in \pref{sec:rl}; for more traditional
divergences such as Hellinger distance and squared error, the
optimistic approach offers no improvement. \arxiv{\\}%
Our results use the following
regularity assumption, satisfied by
all standard divergences.
\begin{assumption}
  \label{ass:continuity}
  \arxiv{  For all pairs of models $M,\Mbar\in\conv(\cM)$, we have
  $(\fmbar(\pi) - \fm(\pi))^2 \leq{}
  \Lcont^2\cdot\Dgenpi{\Mbar}{M}$ for a constant $\Lcont>0$.}
\neurips{For all $M,\Mbar\in\conv(\cM)$, 
  $(\fmbar(\pi) - \fm(\pi))^2 \leq{}
  \Lcont^2\cdot\Dgenpi{\Mbar}{M}$ for a constant $\Lcont>0$.}
\end{assumption}

Given a divergence $D$, we define the \emph{flipped divergence}, which
swaps the first and second arguments, by \neurips{$\Dflippi{\Mbar}{M}\ldef\Dgenpi{M}{\Mbar}$.}
\arxiv{\[
  \Dflippi{\Mbar}{M}\ldef\Dgenpi{M}{\Mbar}.
\]}
We define the \CompText for randomized estimators
\citep{foster2021statistical,chen2022unified}
\neurips{as $\compgenrand(\cM,\mu)
    = \inf_{p\in\Delta(\Pi)}\sup_{M\in\cM}\En_{\pi\sim{}p}\brk[\big]{
    \fm(\pim) -\fm(\pi)
    - \gamma\cdot\En_{\Mbar\sim\mu}\brk*{\Dgenpi{\Mbar}{M}}
  }$, with $\compgenrand(\cM)\ldef\sup_{\mu\in\Delta(\cM)}\compgenrand(\cM,\mu)$.
 }
\arxiv{as follows:
\begin{align}
    \compgenrand(\cM,\mu)
    = \inf_{p\in\Delta(\Pi)}\sup_{M\in\cM}\En_{\pi\sim{}p}\brk*{
    \fm(\pim) -\fm(\pi)
    - \gamma\cdot\En_{\Mbar\sim\mu}\brk*{\Dgenpi{\Mbar}{M}}
    },
  \end{align}
  with
  $\compgenrand(\cM)\ldef\sup_{\mu\in\Delta(\cM)}\compgenrand(\cM,\mu)$.}
  This definition is identical to \pref{eq:comp_general}, but allows
  \arxiv{the reference model }$\Mbar$ to be randomized.\loose
\begin{proposition}
  \label{prop:equiv_symmetry}
Whenever
  \pref{ass:continuity} holds, we have that for all $\gamma>0$,
  \begin{align}
\compgenrandbasic[\Dflipshort]_{3\gamma/2}(\cM)
    - \frac{\Lcont^2}{2\gamma} \leq    \ocompD(\cM) \leq{} \compgenrandbasic[\Dflipshort]_{\gamma/2}(\cM)
    + \frac{\Lcont^2}{2\gamma}.
  \end{align}
\end{proposition}
For settings in which there exists an estimation oracle for which the flipped
estimation error
\neurips{$\Est^{\Dflipshort}=\sum_{t=1}^{T}\En_{\pi\sim{}p\ind{t}}\En_{\Mhat\ind{t}\sim{}\mu\ind{t}}\brk*{\Dgenpi[\pi\ind{t}]{\Mstar}{\Mhat\ind{t}}}$}
\arxiv{\[\Est^{\Dflipshort}=\sum_{t=1}^{T}\En_{\pi\sim{}p\ind{t}}\En_{\Mhat\ind{t}\sim{}\mu\ind{t}}\brk*{\Dgenpi[\pi\ind{t}]{\Mstar}{\Mhat\ind{t}}}\]}
is controlled, this result shows that to match the guarantee in
\pref{thm:main}, optimism is not required, and it suffices to run a
variant of vanilla \etd that incorporates randomized estimators
(cf. \citet{foster2021statistical}, Section 4.3).

We now turn to the role of randomization. When $D$
is convex in the first argument, we have
$\compgenrand(\cM)\leq\sup_{\Mbar\in\conv(\cM)}\compgen(\cM,\Mbar)=\compgen(\cM)$, but it is not immediately
apparent whether the opposite direction of this inequality holds, and
one might hope that working with the randomized \CompShort in
\pref{prop:equiv_symmetry} would lead to improvements over the
non-randomized counterpart in \pref{thm:prev}. The
next result shows that this is not the case: Under mild assumptions on the divergence $D$,
randomization offers no improvement.\loose
\begin{proposition}
  \label{prop:randomized_equivalence}
Let $\Dgenshort$ be any bounded divergence\arxiv{ with the property that
for all models}\neurips{ such that for all} $M,M',\Mbar$ and $\pi\in\Pi$, \neurips{$  \Dgenpi{M}{M'}
  \leq{} C\prn*{
    \Dgenpi{\Mbar}{M}
    + \Dgenpi{\Mbar}{M'}
    }$.}
\arxiv{\begin{equation}
  \label{eq:triangle}
  \Dgenpi{M}{M'}
  \leq{} C\prn*{
    \Dgenpi{\Mbar}{M}
    + \Dgenpi{\Mbar}{M'}
    }.
\end{equation}}
Then for all $\gamma>0$,
\neurips{$
  \sup_{\Mbar}\compgen(\cM,\Mbar)
  \leq{} \compgenrandbasic_{\gamma/(2C)}(\cM)$.
       }
\arxiv{\begin{align}
  \sup_{\Mbar}\compgen(\cM,\Mbar)
  \leq{} \compgenrandbasic_{\gamma/(2C)}(\cM).
       \end{align}
       }

\end{proposition}

\fakepar{Implication for Hellinger distance}
Squared Hellinger distance is symmetric, satisfies
\pref{ass:continuity} with $\Lcont=1$ whenever
$\cR\subseteq\brk{0,1}$, and satisfies \arxiv{Condition \pref{eq:triangle} with
$C=2$}\neurips{the condition in \cref{prop:randomized_equivalence}
with $C=2$}. Hence, by combining \pref{prop:equiv_symmetry} with
\pref{prop:randomized_equivalence}, we obtain the following corollary.
\begin{corollary}
  If $\cR\subseteq\brk*{0,1}$, \neurips{then $    \ocompH[2\gamma](\cM) - \frac{1}{\gamma}
    \leq{} \sup_{\Mbar}\compH(\cM,\Mbar)
    \leq{} \ocompH[\gamma/6](\cM) + \frac{3}{\gamma}\;\;\forall{}\gamma>0$.}
  \arxiv{then for all
  $\gamma>0$,\begin{align*}
    \ocompH[2\gamma](\cM) - \frac{1}{\gamma}
    \leq{} \sup_{\Mbar}\compH(\cM,\Mbar)
    \leq{} \ocompH[\gamma/6](\cM) + \frac{3}{\gamma}.
  \end{align*}}
\end{corollary}
This shows that for Hellinger distance---at least from a statistical perspective---there is
no benefit to using the Optimistic \CompShort or randomized
\CompShort compared to the original version. In particular, this implies that regret bounds based on
the randomized \CompShort with Hellinger distance (such as those found
in recent work of \citet{chen2022unified}) do not offer improvement over the guarantees for vanilla \etd in \citet{foster2021statistical}.
      One caveat,
      though, is that working with
      the Optimistic \CompShort, as well as randomized estimators, has
      potential to give computational improvement, as
      computing a distribution $p\in\Delta(\Pi)$ that minimizes
      $\ocompH(\cM,\mu)$ might be simpler than computing a
      corresponding distribution for $\compH(\cM,\Mbar)$ with
      $\Mbar\in\conv(\cM)$. We are not currently aware of any examples
      where such an improvement occurs, as even maintaining a
      distribution $\mu\in\Delta(\cM)$ requires intractably large
      memory for most classes of interest.

      \arxiv{\paragraph{Implications for reinforcement learning}}
      \neurips{\fakepar{Implication for model-free RL}}
      The bilinear divergence $\Dbipi{Q}{M}
  = \sum_{h=1}^{H}\prn*{\Enmpi{M}{\pi}\brk*{
  \lest(Q;z_h)}}^2$ that we adopt in \pref{sec:rl} is asymmetric, as
are closely related divergences such as squared Bellman error. Here,
there are two reasons why optimistic estimation offers
meaningful advantages.

\noindent\textbf{1.} By \pref{prop:equiv_symmetry} ($\Dbishort$ satisfies
  \pref{ass:continuity} with $\Lcont=\bigoh(H)$), a natural
  alternative to optimistic estimation is to estimate with respect to
  the flipped divergence $\Dbiflipshort$, then appeal to Algorithm 3
  of \citet{foster2021statistical}. The issue with this approach is that minimizing the flipped
  estimation error, which takes the form
  \[
    \Est^{\Dbiflipshort}=\sum_{t=1}^{T}\En_{\pi\sim{}p\ind{t}}\En_{\Mhat\ind{t}\sim{}\mu\ind{t}}\brk*{
      \Dbipi[\pi\ind{t}]{\Qmstarstar}{\Mhat\ind{t}}
    }
    =\sum_{t=1}^{T}\En_{\pi\sim{}p\ind{t}}\En_{\Mhat\ind{t}\sim{}\mu\ind{t}}\brk*{
      \sum_{h=1}^{H}\prn*{\Enmpi{\Mhat\ind{t}}{\pi\ind{t}}\brk*{
  \lest(\Qmstarstar;z_h)}}^2
    }
  \]
  appears to be challenging in model-free settings; we are not
  aware of any algorithms that accomplish this\arxiv{
  task}.\footnote{While \citet{foster2021statistical} do give regret
    bounds for model-free RL using this divergence, they only bound
    Bayesian regret, and do not provide an explicit algorithm for
    the frequentist setting.}\loose

  \noindent\textbf{2.}
  On the other hand, a second alternative is to perform estimation with
  respect to the un-flipped divergence $\Dbishort$ (which can be
  accomplished with \pref{prop:estimation_bilinear} by taking
  $\gamma\to\infty$), and appeal to vanilla \etd (either
  \pref{alg:etd}, or Algorithm 3
  of \citet{foster2021statistical} if one wishes to incorporate
  randomized estimators). However, the following result shows that
  unlike the \newcomp, the \vanilla \CompShort with the divergence
  $\Dbishort$ does not admit a favorable bound, even for tabular
  reinforcement learning.
  \begin{proposition}
    \label{prop:bilinear_separation}
    Let $\cM$ be the class of all horizon-$H$ tabular MDPs with
    $\abs{\cS}=2$ and
    $\abs{\cA}=2$. Consider the discrepancy function $\lest(Q;z_h) =
    (Q_h(s_h,a_h)-r_h-\max_{a'\in\cA}Q_{h+1}(s_{h+1},a'))$. Then we
    have $\ocompbi(\cM)\approxleq{} \frac{H}{\gamma}$, yet there
    exists $\Mbar\in\cM$ for which $\compbi(\cM,\Mbar) \approxgeq \frac{2^{H}}{\gamma}\wedge{}1$.
  \end{proposition}

\arxiv{\subsection{Insufficiency of Posterior Sampling}}
\neurips{\fakepar{Insufficiency of posterior sampling}}
For contextual bandits, where
$\ocompSq(\cM)\approxleq\frac{\abs{\cA}}{\gamma}$
\citep{zhang2022feel}, and bilinear classes, where
$\ocompbi(\cM)\approxleq\frac{H\cdot\dimbi(\cM)}{\gamma}$
(\pref{prop:dec_bilinear}), a strategy that achieves the bound on
the \newcomp is \emph{posterior sampling} (this is also the approach
taken in
\citet{agarwal2022model,agarwal2022non,zhong2022posterior}). That is, given a
distribution $\mu\in\Delta(\cM)$, choosing
$p(\pi)=\mu(\crl{M\in\cM\mid{}\pim=\pi})$ in \pref{eq:optimistic_dec}
certifies the desired bound on $\ocompD(\cM,\mu)$ for these
examples. \etdopttext subsumes and generalizes posterior sampling, but
in light of the fact that this simple strategy succeeds for
large classes of problems, it is reasonable to ask if there is a
sense in which posterior sampling is universal, and whether it can
achieve the value of the \newcomp for any model class. This would be
desirable, since it would indicate that solving the minimax problem in
\pref{alg:main} is not necessary. The following sample shows that this
is not the case: there are model classes (specifically, MDPs with a
constant number of actions) for which the regret of posterior sampling
is exponentially large compared to the regret of \pref{alg:main}.
\begin{proposition}
  \label{prop:cheating}
  Consider the divergence $\Dhelspi{\cdot}{\cdot}$.
  For any $S\in\bbN$ and $H\geq{}\log_2(S)$, there exists a class of 
  horizon-$H$ MDPs $\cM$ with $\abs{\cS}=S$ and $\abs{\cA}=3$ that
  satisfies the following properties:     \neurips{\\
    $\bullet$ There exists an estimation oracle with
    $\EstOptHel\approxleq \log(S/\delta)$ \arxiv{with probability}\neurips{w.p.} at least
    $1-\delta$ for all $\gamma>0$.\\
  $\bullet$ \arxiv{The posterior sampling algorithm}\neurips{Posterior sampling}, which sets
    $p\ind{t}(\pi)=\mu\ind{t}(\crl{M \mid{} \pim=\pi})$, has
    $\En\brk*{\RegDM}\approxgeq{}S\wedge{}2^{\bigom(H)}$.\\
  $\bullet$ \arxiv{We have $\ocompH(\cM)\leq\frac{1}{\gamma}$ and consequently}
    \pref{alg:main} with divergence $D=\Dhelspi{\cdot}{\cdot}$
    has
    $\En\brk*{\RegDM}\leq\bigoht\prn[\big]{\sqrt{T\log(S)}}$.
}
  \arxiv{
  \begin{itemize}
  \item There exists an estimation oracle with
    $\EstOptHel\approxleq \log(S/\delta)$ \arxiv{with probability}\neurips{w.p.} at least
    $1-\delta$ for all $\gamma>0$.
  \item \arxiv{The frequentist posterior sampling algorithm}\neurips{Posterior sampling}, which sets
    $p\ind{t}(\pi)=\mu\ind{t}(\crl{M \mid{} \pim=\pi})$, has
    $\En\brk*{\RegDM}\approxgeq{}S\wedge{}2^{\bigom(H)}$.
      \item \arxiv{We have $\ocompH(\cM)\leq\frac{1}{\gamma}$ and consequently}
    \pref{alg:main} with divergence $D=\Dhelspi{\cdot}{\cdot}$
    has
    $\En\brk*{\RegDM}\leq\bigoht\prn[\big]{\sqrt{T\log(S)}}$.
  \end{itemize}}
\end{proposition}
This shows that posterior sampling does not provide a universal
mechanism for exploration, and highlights the need for deliberate
strategies such as \etd.

\arxiv{
\subsection*{Acknowledgements}
NG is supported by a Fannie \& John Hertz Foundation Fellowship and an
NSF Graduate Fellowship. AR acknowledges support from the NSF through
award DMS-2031883, from the ARO through award W911NF-21-1-0328, and
from the DOE through award DE-SC0022199.
}

\bibliography{refs} 

\clearpage

\appendix

        \neurips{
        	\onecolumn
        	\renewcommand{\contentsname}{Contents of Appendix}
        	\addtocontents{toc}{\protect\setcounter{tocdepth}{2}}
        	{
        		\hypersetup{hidelinks}
        		\tableofcontents
        	}
        }

\section{Additional Related Work}
In this section we discuss additional related work not already covered.

\citet{chen2022unified} build on the \CompShort framework by giving
regret bounds for a \etd that incorporates randomized estimators, but
not optimism. For the case of finite classes, their guarantees scale
as roughly
\[
  \RegDM \approxleq{}
  \compgenrand[\Dhelshort](\cM)\cdot{}T + \gamma\cdot\log(\abs{\cM}/\delta),
\]
where $\compgenrand[\Dhelshort](\cM)$ is the \CompShort for randomized
estimators defined in \cref{sec:discussion}. As discussed in
\cref{sec:discussion}, this regret bound cannot improve upon the
guarantees for the original \etd method in
\citet{foster2021statistical} beyond constants, as we have
$\sup_{\Mbar}\comp[4\gamma]^{\mathsf{H}}(\cM,\Mbar)\leq{}\compgenrand[\Dhelshort](\cM)$. In
addition, since the algorithm does not incorporate optimism, it cannot
be directly applied to model-free reinforcement learning settings.

\citet{foster2023tight} give upper and lower bounds on optimal regret based on a variant
of the \CompShort called the \emph{constrained \CompText}. These
results tighten the original regret bounds in
\citet{foster2021statistical}, but the upper bounds still scale
with $\EstHFull=\log(\abs{\cM}/\delta)$, rendering them unsuitable for
model-free RL. Nonetheless it would be interesting to explore
whether the techniques in this work can be combined with optimistic estimation.

        \neurips{
          \section{Additional Results}
          \label{sec:additional}
  \subsection{Optimistic Estimation-to-Decisions with Batching}
  \label{sec:batched}
\arxiv{\begin{algorithm}[t]}
  \neurips{\begin{algorithm}[htp]}
    \setstretch{1.3}
     \begin{algorithmic}[1]
       \State \textbf{parameters}:
       \Statex[1] Online estimation oracle $\AlgEst$ with batch size $n$.
       \Statex[1] Exploration parameter $\gamma>0$.
       \Statex[1] Divergence $\Dgen{\cdot}{\cdot}$ with sufficient
       statistic space $\Suff$.
       \State Let $K\ldef{}T/n$.
  \For{$k=1, 2, \cdots, K$}
  \State Receive randomized estimator $\mu\ind{k}\in\Delta(\Suff) = \AlgEst\ind{t}\prn[\big]{(B\ind{i})_{i=1}^{k-1} }$.
\State \mbox{\arxiv{Compute}\neurips{Get}
$p\ind{k}\gets\argmin_{p\in\Delta(\Act)}\sup_{M\in\cM}\En_{\act\sim{}p}\En_{\suffhat\sim\mu\ind{k}}\brk*{\fsuffhat(\pisuffhat)-\fm(\pi)
    -\gamma\cdot\Dgenpi{\suffhat}{M}}$.
  \algcommentlight{Eq. \pref{eq:optimistic_dec}.}}
\State{}\multiline{Sample batch
$B\ind{k}=\crl*{(\pi\ind{k,l},r\ind{k,l},o\ind{k,l})}_{l=1}^{n}$ where
$\pi\ind{k,l}\sim{}p\ind{k}$ and
$(r\ind{k,l},o\ind{k,l})\sim\Mstar(\pi\ind{k.l})$, and update estimation
oracle with $B\ind{k}$.}
\EndFor
\end{algorithmic}
\caption{Optimistic \etdtext (\etdopt) with Batching}
\label{alg:main_batched}
\end{algorithm}

For our application to reinforcement learning, it will be useful to
generalize \etdopt to accomodate estimation algorithms
that draw \emph{batches} of multiple samples from each distribution
$p\ind{t}$. Given a \emph{batch size} $n$, we break the
$T$ rounds of the decision making protocol into $K\ldef{}T/n$
contiguous epochs (or, ``iterations''). Within each epoch, the
learner's distribution $p\ind{k}$ is unchanged (we index by $k$ rather
than $t$ to reflect this), and we create a \emph{batch}
$B\ind{k}=\crl*{(\pi\ind{k,l},r\ind{k,l},o\ind{k,l})}_{l=1}^{n}$ by
sampling $\pi\ind{k,l}\sim{}p\ind{k}$ independently and observing
$(r\ind{k,l},o\ind{k,l})\sim\Mstar(\pi\ind{k,l})$ for each
$l\in\brk{n}$. We consider estimation oracles of the form
\neurips{$\mu\ind{k}=\AlgEst\ind{k}\prn*{ \crl*{B\ind{k}}_{i=1}^{k-1}}$,}
\arxiv{\[
\mu\ind{k}=\AlgEst\ind{k}\prn*{ \crl*{B\ind{k}}_{i=1}^{k-1}},
\]}
and measure estimation error via
\begin{equation}
  \EstOptD \ldef{}
\sum_{k=1}^{K}\En_{\act\ind{k}\sim{}p\ind{k}}\En_{\Mhat\ind{k}\sim\mu\ind{k}}\brk*{\Dgenpi{\Mhat\ind{k}}{\Mstar}
  +
  \gamma^{-1}(\fmstar(\pimstar)-\fmhatt(\pi\subs{\Mhat\ind{k}}))}.\label{eq:estimation_batched}
\end{equation}
We assume that the estimation oracle ensures that with probability at least $1-\delta$,
$\EstOptD\leq{}\EstOptDFullKn$, where $\EstOptDFullKn$ is a known
upper bound.

\pref{alg:main_batched} is a variant of \etdopt that incorporates
batching. The algorithm updates the distribution $p\ind{k}$ in the
same fashion as its non-batched counterpart, but does so only at the
beginning of each epoch. The main guarantee for this method as follows.
\begin{theorem}
  \label{thm:main_batched}
  Let $T\in\bbN$ be given, and let $n$ be the batch size. For any
  $\delta>0$, \pref{alg:main_batched} ensures that with probability at
  least $1-\delta$,
  \neurips{$    \RegDM \leq \ocompD(\cM)\cdot{}T + \gamma{}n\cdot{}\EstOptD(T/n,n,\delta)$.}
\arxiv{\begin{align}
  \label{eq:main_multi_sample}
    \RegDM \leq \ocompD(\cM)\cdot{}T + \gamma{}n\cdot{}\EstOptD(T/n,n,\delta).
       \end{align}
       }
\end{theorem}
\neurips{See \cref{app:main_proofs} for the proof.}
When working with divergences for which unbiased estimates are
unavailable, this approach can lead to stronger guarantees than
\pref{thm:main}. We refer to the proof of
\pref{prop:estimation_bilinear} for a concrete example.

}

\neurips{
  \subsection{Model-Free RL: Tighter Guarantees under Bellman Completeness}
  \label{sec:complete}
\arxiv{In the remainder of the section}\neurips{In this section}, we show how to derive tighter
estimation guarantees (and consequently tighter regret bounds) when
$\cQ$ satisfies a \emph{Bellman completeness} assumption \citep{,zanette2020learning,jin2021bellman}. For a given
MDP $M$, let $\cTm_h\brk*{f}(s,a) \ldef \En\sups{M}\brk*{r_h +
  \max_{a'}f(s_{h+1},a')\mid{}s_h=s,a_h=a}$ denote the Bellman
operator for layer $h$.
\begin{assumption}[Completeness]
  We assume that $\cQ=\cQ_1\times\cdots\times\cQ_H$ is a product
  class, and that for all $h$ and $Q_{h+1}\in\cQ_{h+1}$,
  $\brk{\cTmstar_hQ_{h+1}}\in\cQ_{h}$.
\end{assumption}
As before, we take $\Psi=\cQ$, $\suffmap(M)\ldef{}\Qmstar$, and $\fq(\piq)\ldef\En_{s_1\sim{}d_1}\brk*{Q_1(s_1,\piq(s_1))}$. For the divergence
$\Dgenshort$, we now appeal to \emph{squared Bellman error} (e.g., \cite{jin2021bellman,xie2022role}):
\begin{align}
  \Dsbepi{Q}{M}
  = \sum_{h=1}^{H}\Enmpi{M}{\pi}\brk*{\prn*{Q_h(s_h,a_h)-\brk{\cT\sups{M}_hQ_{h+1}}(s_h,a_h)}^2}.
  \label{eq:discrepancy}
\end{align}
We abbreviate $\EstOptSB=\EstOptSBlong$ and $\ocompsbe(\cM,\mu) =
\odec^{\Dsbeshort}_{\gamma}(\cM,\mu)$.

\paragraph{Estimation}
\cref{alg:ts3_rl} %
performs optimistic online estimation with squared Bellman error. The
algorithm is an adaptation of a two-timescale exponential weights strategy
originally introduced by \citet{agarwal2022non} within an optimistic posterior
sampling algorithm referred to as \tst. We show that this technique
leads to a self-contained online estimation guarantee
outside the context of the \tst algorithm.

\begin{proposition}[Estimation for square Bellman error]
  \label{prop:ts3_rl}
  Assume that $\cQ$ satisfies completeness relative to $M^{\star}$. Moreover assume $\sum_{h=1}^{H}\sup_{s,a} r_h(s,a) \leq 1$ and $\sup_{Q,h,s,a} Q_h(s,a) \leq 1$.
  Then for any $\gamma\geq{}1
  $ and $\eta\leq 1/(2^{16}(\log (|\cQ|K/\delta) + 1))$
  , with batch size $n=H$
  ($K\ldef{}T/n$), $\lambda=1/8$,  $\beta=(12\gamma H)^{-1}$ and $\delta>0$, 
  \cref{alg:ts3_rl} ensures that with probability at least $1-\delta$,
    \begin{align}
      \EstOptSB  \approxleq{}
      \frac{H\log\abs{\cQ}}{\eta} + \frac{\eta\log (\abs{\cQ}K/\delta)K}{\gamma} + \frac{K}{\gamma^2 H}.
    \end{align}
    whenever $\cQ$ satisfies completeness relative to
    $\Mstar$.\footnote{\citet{agarwal2022non} give a tighter estimation
      error bound of roughly $      \EstOptSB  \approxleq{}
      \log^2(\abs{\cQ}HK) + \frac{K}{\gamma^2}$, but this result takes advantage of a Bellman rank assumption on the underlying MDP. The estimation error bound we state here does not require any structural assumptions on the MDP under consideration, but gives a worse rate.
      }
  \end{proposition}
Note that \cref{prop:ts3_rl} does not make use of the bilinear class
assumption, and only requires that $\cQ$ satisfies completeness. \neurips{As
such, we expect that this result will find use more broadly.}

\paragraph{Regret bound for bilinear classes}
To provide regret bounds, we assume that $\Mstar$ satisfies the
bilinear class property relative to $\cQ$ as in \cref{sec:rl}. In addition to
assuming that $\Mstar$ is bilinear, we make the following
restrictions: (1) $\piestq=\piq$, i.e. estimation is on policy, (2)
$\lest(Q;z_h)=Q_h(s_h,a_h)-r_h-\max_{a'}Q_{h+1}(s_{h+1},a')$, so that 
$\Enmpi{M}{\pi}\brk*{\lest(Q;z_h)}$ is the average
Bellman error for $Q$ under $M$.\footnote{These restrictions
  correspond to restricting attention to \emph{$Q$-type Bellman rank}, a
  special case of the bilinear class property \citep{jin2021bellman,foster2021statistical}.}
  With this discrepancy function,
Jensen's inequality implies that $\ocompsbe(\cM)\leq{}\ocompbi(\cM)$,
so combining \pref{thm:main_batched}, \pref{prop:ts3_rl}, and
\pref{prop:dec_bilinear}, we obtain the following result.
\begin{corollary}[Regret bound under completeness]
  \label{cor:regret_complete}
Let $\cQ$ be given. Assume that $\Mstar\in\cM$, where $\cM$ is
bilinear relative to $\cQ$, and that
completeness holds. Moreover assume $\sum_{h=1}^{H}\sup_{s,a} r_h(s,a) \leq 1$ and $\sup_{Q,h,s,a} Q_h(s,a) \leq 1$. Abbreviate $d\equiv\dimbi(\cQ;\cM)$. For an appropriate choice of $n$ and
$\gamma$, \pref{alg:main_batched}, using \cref{alg:ts3_rl} (with
appropriate parameter choice) as an oracle, enjoys
the following guarantees with probability at least $1-\delta$:
\begin{align}
\RegDM \approxleq{} H d^{1/3} (\log(\abs{\cQ}K/\delta))^{4/5}T^{2/3} .
  \end{align}
\end{corollary}
This improves upon the $T^{3/4}$-type rate in \cref{cor:bilinear}.

  \begin{algorithm}[tp]
      \caption{Two-Timescale Exponential Weights for Bellman Complete Value Function Classes}
\label{alg:ts3_rl}
    \begin{algorithmic}[1] 
    \State Initialize $\histSet\ind{0}=\emptyset$.
    \For{$k=1,\dots,K$}
    \State For any $Q,Q'\in \cQ$ and $h\in [H]$, define 
    \begin{gather*}
      \Delta_h\ind{k}(Q',Q) \ldef Q'(s_h\ind{k,h},a_h\ind{k,h}) - r_h\ind{k,h} - Q(s_{h+1}\ind{k,h}),\\
      q\ind{k}(Q'|Q) \ldef q\ind{k}(Q'|Q,\histSet\ind{k-1}) \propto \exp\prn*{-\lambda\cdot \frac{1}{H}\sum\limits_{s=1}^{k-1} \sum\limits_{h=1}^{H}\Delta_h\ind{s}(Q',Q)^2},\\
      L\ind{k}(Q) \ldef \frac{1}{H} \sum\limits_{h=1}^{H} \Delta_h\ind{k}(Q,Q)^2 + \frac{1}{\lambda} \log \En_{Q'\sim q\ind{k}(\cdot|Q)} \brk*{ \exp \prn*{-\lambda \cdot \frac{1}{H} \sum\limits_{h=1}^{H}\Delta_h\ind{k}(Q',Q)^2}  }    ,\\
      \mu\ind{k}(Q) \ldef  \mu\ind{k}(Q|\histSet\ind{k-1}) \propto \exp\prn*{ -\eta \sum\limits_{s=1}^{k-1} \prn*{ L\ind{s}(Q)  - \beta\cdot \frac{1}{H} \sum\limits_{h=1}^{H} \max_a Q(s_1\ind{s,h},a)} }.
    \end{gather*}
    \State Predict $\mu\ind{k}$.
    \For{$l=1,\dots,H$}
    \State \multiline{Play $\pi\ind{k,l}\sim p\ind{k}$ and obtain the trajectory $o\ind{k,l}=(s_1\ind{k,l},a_1\ind{k,l},r_1\ind{k,l}),\ldots, 
    (s_H\ind{k,l},a_H\ind{k,l},r_H\ind{k,l})$, where
    $p\ind{k}\in\Delta(\Pi)$ is a decision distribution produced by
    any batched algorithm (e.g., \cref{alg:main_batched}) that selects a
    decision adaptively based on $\mu\ind{k}$.}
    \EndFor
    \State Update $\histSet\ind{k} \gets \histSet\ind{k-1}\cup \bigcup_{l=1}^H\set{s_l\ind{k,l},a_l\ind{k,l},r_l\ind{k,l},s_{l+1}\ind{k,l}} \cup \bigcup_{l=1}^H \set{s_1\ind{k,l}}$.
    \EndFor
    \end{algorithmic}
    \end{algorithm}

\neurips{    
\subsection{Proofs from \cref{sec:complete}}

\cref{prop:ts3_rl} is an application of more general results given in
\cref{app:completeness}, which analyze a generalization of
\cref{alg:ts3_rl} for a more general online learning setting. To
\cref{prop:ts3_rl}, we simply apply these results to the reinforcement
learning framework.
\begin{proof}[\pfref{prop:ts3_rl}]
  Let the batch size $n=H$ be fixed, and let $K\ldef{}T/n$ be the number of epochs. Recall that for each step $k\in\brk{K}$, the estimation
  oracle is given a batch of examples $B\ind{k}=\crl*{(\pi\ind{k,l},r\ind{k,l},o\ind{k,l})}_{l=1}^{n}$ where
  $\pi\ind{k,l}\sim{}p\ind{k}$ and
  $(r\ind{k,l},o\ind{k,l})\sim\Mstar(\pi\ind{k,l})$. Each observation
  (trajectory) takes the form
  $o\ind{k,l}=(s_1\ind{k,l},a_1\ind{k,l},r_1\ind{k,l}),\ldots, 
  (s_H\ind{k,l},a_H\ind{k,l},r_H\ind{k,l})$. 
  We abbreviate $\Qstar=\Qmstarstar$.
  \paragraph{Estimation algorithm} For each step $k$, the randomized
  estimator $\mu\ind{k}$ selected as described in
  \pref{alg:ts3_rl}. This algorithm is an instantiation of
  \pref{alg:ts3} in the general online learning setting described in
  \cref{app:completeness}, with $\Gclass = \cQ$ and for all $h\in
  [H]$, $\Xclass_h = \cS \times \cA$, $\Yclass_h = \bbR\times \cS$ and
  $\Zclass = \ngedit{\cS^H}$. The unknown kernels are the transition
  distributions for the corresponding layers of the MDP $\Mstar$, and the loss functions are
  \begin{align*}
  \loss_{h,1}((s_h, a_h), Q) &\ldef Q_h(s_h, a_h) ,\\
  \loss_{h,2}((r_h,s_{h+1}),Q) &\ldef r_h +  \max_a Q_h(s_{h+1},a),\\
  \loss_3(\set{s_1\ind{l}}_{l\in [H]} ) &\ldef  - \frac{1}{H} \sum\limits_{l=1}^{H} \max_a Q_1(s_1\ind{l},a) 
  \end{align*}
  Finally, take $\x_h\ind{k} = (s_h\ind{k,h}, a_h\ind{k,h})$,
  $\y_h\ind{k} = (r_h\ind{k,h}, s_{h+1}\ind{k,h})$ and $\z\ind{k} =
  \set{s_1\ind{k,h}}_{h\in [H]}$. \jqedit{It is important to note that
    $s_h\ind{k,h}, a_h\ind{k,h},r_h\ind{k,h}, s_{h+1}\ind{k,h}$ are
    taken from different trajectories for $h\in [H]$, so
    $\y_h\ind{k}\mid{}\x_h\ind{k}$ are independent from one other for
    $h\in [H]$.} \ngedit{Moreover, note that the distributions
    $p\ind{k} \in \Delta(\Pi)$ play the role of nature: the
    distribution of the tuple ($x_h\ind{k}, y_h\ind{k}, w\ind{k})$ for $h \in [H]$ is determined by running a policy $\pi\ind{k} \sim p\ind{k}$ in the ground-truth MDP $M^\star$.} With this configuration, observe that in the notation of \cref{app:completeness}, we have\ngedit{, for any $Q$,} 
  \begin{align*}
    \En_{\x_{1:H}\ind{k}}\cE(Q,Q,\x_{1:H}\ind{k})^2 &= \frac{1}{H} \sum\limits_{h=1}^{H}\En_{\x_h\ind{k}}(\loss_{h,1}(\x_h\ind{k},Q) - \En[\loss_{h,2}(\y_h,Q)|\x_h\ind{k}])^2 \\
    &= \frac{1}{H} \sum\limits_{h=1}^{H}\En_{\pi\ind{k,h}\sim p\ind{k}}\Enmpi{\Mstar}{\pi\ind{k,h}}\brk*{\prn*{Q_h(s_h,a_h)-\brk{\cT\sups{\Mstar}_hQ_{h+1}}(s_h,a_h)}^2} \\
    &= \frac{1}{H} \En_{\pi\sim p\ind{k}}\Enmpi{\Mstar}{\pi}\brk*{ \sum\limits_{h=1}^{H}\prn*{Q_h(s_h,a_h)-\brk{\cT\sups{\Mstar}_hQ_{h+1}}(s_h,a_h)}^2}\\
    &= \frac{1}{H} \En_{\pi\sim p\ind{k}} \Dsbepi{Q}{\Mstar}.
  \end{align*}
and 
\begin{align*}
  \En_{\x_{1:H}\ind{k},\z\ind{k}}\iota\ind{k}(Q) &= \frac{1}{H} \sum\limits_{l=1}^{H} \En_{\pi\ind{k,h}\sim p\ind{k}}\Enmpi{\Mstar}{\pi\ind{k,h}} \brk*{ \max_a Q_1^*(s_1\ind{k,l},a) - \max_a Q_1(s_1\ind{k,l},a)}\\
  &= \fmstar(\pimstar)-\fq(\piq) .
\end{align*}

  \paragraph{Estimation error bound} We take $\alpha = 12\coef$, so that \pref{thm:main-OL} implies that with probability at least $1-\delta$, 
  \begin{align*}
    &\sum\limits_{k=1}^{K}  \En_{Q\sim \mu\ind{k}} \prn*{\frac{1}{H} \En_{\pi\sim p\ind{k}} \Dsbepi{Q}{\Mstar} + \alpha (\fmstar(\pimstar)-\fq(\piq) )  } \\
    &\approxleq \eta\alpha\log(\abs{\cQ}K/\delta)K  + \log\abs{\cQ}/\eta  + \alpha^2 K.
  \end{align*}
  Then by taking $\alpha = \frac{1}{\gamma H}$, this further implies that with probability at least $1-\delta$,
  \begin{align*}
    \EstOptSB &= \sum\limits_{k=1}^{K}  \En_{\pi\sim p\ind{k}}\En_{Q\sim \mu\ind{k}} \prn*{ \Dsbepi{Q}{\Mstar} + \frac{1}{\gamma} (\fmstar(\pimstar)-\fq(\piq) )  } \\
    &\approxleq H( \eta\alpha\log(\abs{\cQ}K/\delta)K  + \log\abs{\cQ}/\eta  + \alpha^2 K)\\
    &\approxleq{} \frac{H\log\abs{\cQ}}{\eta} + \frac{\eta\log (\abs{\cQ}K/\delta)K}{\gamma} + \frac{K}{\gamma^2 H}.
  \end{align*}
\end{proof}

\begin{proof}[\pfref{cor:regret_complete}] We choose $n=H$ and apply \cref{alg:ts3_rl} as the estimation oracle. 
We first consider the ``trivial'' parameter regime in which $ H d^{1/3}(\log(\abs{\cQ}K/\delta))^{-1/5}T^{-1/3} \geq 1/(2^{16}(\log (|\cQ|K/\delta) + 1))$. Here, $T \approxleq H d^{1/3} (\log(\abs{\cQ}K/\delta))^{4/5}T^{2/3}$, and thus
 \begin{align*}
  \RegDM
  \approxleq{} H d^{1/3} (\log(\abs{\cQ}K/\delta))^{4/5}T^{2/3}.
 \end{align*}
 When the case above, does not hold, we proceed as in the theorem statement, choosing $\eta = H d^{1/3}(\log(\abs{\cQ}K/\delta))^{-1/5}T^{-1/3} \leq  1/(2^{16}(\log (|\cQ|K/\delta) + 1))$. Combining \pref{thm:main_batched} and  \pref{prop:ts3_rl} then gives %
  \begin{align*}
    \RegDM
    \approxleq{} \ocompsbe(\cM)\cdot{}T 
    + \gamma{}\frac{H^2\log\abs{\cQ}}{\eta} + \eta\log (\abs{\cQ}K/\delta)T + \frac{K}{\gamma^2}
  \end{align*}
  with probability at least $1-\delta$.
    Next, using \pref{prop:dec_bilinear} to bound $\ocompsbe(\cM)$ in the above display, it follows that 
    \begin{align*}
          \RegDM
      \approxleq{} \frac{HdT}{\gamma} + \gamma{}\frac{H^2\log\abs{\cQ}}{\eta} + \eta\log (\abs{\cQ}K/\delta)T + \frac{K}{\gamma^2} .
    \end{align*}
    We choose $\gamma  = d^{2/3} (\log (\abs{\cQ}K/\delta))^{-2/5} T^{1/3}$ to obtain 
    \begin{align*}
      \RegDM
      \approxleq{} H d^{1/3} (\log(\abs{\cQ}K/\delta))^{4/5}T^{2/3} 
    \end{align*}
    with probability at least $1-\delta$.

  \end{proof}

 }

   }

\section{Technical Tools}

\subsection{Preliminaries}
\begin{lemma}
\label{lem:technical}
For all $x\in [0,1]$, we have
\begin{align*}
  e^{-x} \leq 1 - (1-1/e) x \leq 1- x/2,\mathand
  e^x \leq  1 + (e-1)x \leq 1+2x.
\end{align*}
\end{lemma}

\begin{lemma}
\label{lem:technical-2}
For all $x\geq -1/8$, we have $e^{-x} \leq 1 - x + x^2$.
\end{lemma}

\begin{proof}[\pfref{lem:technical-2}]
Let $f(x) = e^{-x} - 1 + x - x^2$. We have $f''(x) = e^{-x} - 2 < 0$
for $x\geq -1/8$. Thus, $f'$ is monotonically decreasing on $x\geq
-1/8$, so for $x\in [-1/8, 0]$, $f'(x) \geq f'(0) = 0$. Hence, $f(x)$
is non-decreasing on $x\in [-1/8, 0]$. Furthermore, $f'(x) = -e^{-x} +
1 - 2x \leq 0$ for $x\geq 0$. Thus $f(x)$ obtains maximum value at
$x=0$, and $f(x) \leq f(0) = 0$.
\end{proof}

\subsection{Basic Online Learning Results}

In this section we state a technical lemma regarding the
performance of the exponential weights algorithm for online
learning. Let $\cG$ be an abstract set of hypotheses. We consider the following online
learning process.

For $t=1,\ldots,T$: %
\begin{itemize}
\item Learner predicts a (random) hypothesis $g\ind{t}\in\Gclass$.
\item Nature reveals $\loss\ind{t}\in\cL\ldef{}(\cG\to\bbR)$ and learner suffers loss
  $\loss\ind{t}(g\ind{t})$.
\end{itemize}
We define regret to the class $\cG$ via
\begin{equation}
  \label{eq:regret_ol}
    \RegOL = \sum_{t=1}^{T}\En_{g\ind{t}\sim\mu\ind{t}}\brk*{\loss\ind{t}(g\ind{t})} - \inf_{g\in\cG}\sum_{t=1}^{T}\loss\ind{t}(g),
  \end{equation}
  where $\mu\ind{t}\in\Delta(\cG)$ is the learner's randomization
  distribution for step $t$.

  \begin{lemma}
    \label{lem:exp_weights}
Consider the exponential weights update
    method with learning rate $\eta>0$, which sets
    \[
      \mu\ind{t}(g)\propto\exp\prn*{-\eta\sum_{i<t}\ls\ind{i}(g)}.
    \]
    For any sequence of non-negative loss functions
    $\ls\ind{1},\ldots,\ls\ind{T}$, this algorithm satisfies
    \begin{equation}
      \label{eq:exp_weights_1}
    \RegOL \leq{}
\frac{\eta}{2}\sum_{t=1}^{T}\En_{g\ind{t}\sim\mu\ind{t}}\brk*{(\loss\ind{t}(g\ind{t}))^2}
+ \frac{\log\abs{\cG}}{\eta}.
\end{equation}
In addition, for any sequence of loss functions
$\ls\ind{1},\ldots,\ls\ind{T}$ with $\ls\ind{t}(g)\in\brk*{-L,L}$ for
all $g\in\cG$, if $\eta\leq{}(2L)^{-1}$, then
    \begin{equation}
      \label{eq:exp_weights_2}
\RegOL \leq{}
2\eta\sum_{t=1}^{T}\En_{g\ind{t}\sim\mu\ind{t}}\brk*{(\loss\ind{t}(g\ind{t})-\En_{g'\sim\mu\ind{t}}\brk*{\loss\ind{t}(g')})^2}
+ \frac{\log\abs{\cG}}{\eta} \\
\leq{}
4\eta\sum_{t=1}^{T}\En_{g\ind{t}\sim\mu\ind{t}}\brk*{(\loss\ind{t}(g\ind{t}))^2}
+ \frac{\log\abs{\cG}}{\eta}
\end{equation}

  \end{lemma}
  \begin{proof}[\pfref{lem:exp_weights}]
A standard telescoping argument combined with the fact that $-\inf_{g \in \mathcal{G}} \sum_{t=1}^T \ell\ind{t}(g) \leq \frac{1}{\eta} \log \left( \exp \left( \sum_{g \in \mathcal{G}} -\eta \sum_{t=1}^T \ell\ind{t}(g)\right)\right)$ (e.g., \citet{PLG}) gives that for any choice $\eta>0$
and any sequence of loss functions, exponential weights has
\begin{align}
  \RegOL
  &\leq{} \sum_{t=1}^{T}\En_{g\sim\mu\ind{t}}\brk*{\loss\ind{t}(g)}
  +
  \frac{1}{\eta}\sum_{t=1}^{T}\log\prn*{\sum_{g\in\cG}\mu\ind{t}(g)\exp\prn*{-\eta\loss\ind{t}(g)}}
    + \frac{\log\abs{\cG}}{\eta} \\
    &= \frac{1}{\eta}\sum_{t=1}^{T}\log\prn*{\sum_{g\in\cG}\mu\ind{t}(g)\exp\prn*{-\eta\prn*{\loss\ind{t}(g)-\En_{g'\sim\mu\ind{t}}\brk*{\loss\ind{t}(g')}}}}
      + \frac{\log\abs{\cG}}{\eta}.
      \label{eq:exp_weights_basic}
\end{align}
We first prove \pref{eq:exp_weights_1}. Using that $\log(x)\leq{}x-1$ for $x\geq{}0$ and
$\exp(-x)\leq{}1-x+\frac{x^2}{2}$ for $x\geq{}0$, we have
\[
  \log\prn*{\sum_{g\in\cG}\mu\ind{t}(g)\exp\prn*{-\eta\loss\ind{t}(g)}}
  \leq{} -\eta\En_{g\sim\mu\ind{t}}\brk*{\loss\ind{t}(g)}
  + \frac{\eta^2}{2}\En_{g\sim\mu\ind{t}}\brk*{(\loss\ind{t}(g))^2},
\]
so that
\[
\RegOL \leq{}
\frac{\eta}{2}\sum_{t=1}^{T}\En_{g\sim\mu\ind{t}}\brk*{(\loss\ind{t}(g))^2}
+ \frac{\log\abs{\cG}}{\eta}.
\]

To prove \pref{eq:exp_weights_2}, we use that $\log(x)\leq{}x-1$ for $x\geq{}0$ and
$\exp(-x)\leq{}1-x+2x^2$ whenever $\abs{x}\leq{}1$ to get
\[
  \log\prn*{\sum_{g\in\cG}\mu\ind{t}(g)\exp\prn*{-\eta\prn*{\loss\ind{t}(g)-\En_{g'\sim\mu\ind{t}}\brk*{\loss\ind{t}(g')}}}}
  \leq{} 2\eta^2\En_{g\sim\mu\ind{t}}\brk*{\prn*{\loss\ind{t}(g)-\En_{g'\sim\mu\ind{t}}\brk*{\loss\ind{t}(g')}}^2},
\]
so that
\[
\RegOL \leq{}
2\eta\sum_{t=1}^{T}\En_{g\sim\mu\ind{t}}\brk*{(\loss\ind{t}(g)-\En_{g'\sim\mu\ind{t}}\brk*{\loss\ind{t}(g')})^2}
+ \frac{\log\abs{\cG}}{\eta}.
\]

\end{proof}

\subsection{Online Learning with Completeness}
\label{app:completeness}

In this section, we give guarantees for an online learning algorithm
\cref{alg:ts3}, which generalizes the two-timescale exponential
weights algorithm (\cref{alg:ts3_rl}) of \citet{agarwal2022non}. We describe and analyze the
algorithm in a general online learning framework, which abstracts away
the core problem solved by \cref{alg:ts3_rl}: value function
estimation using a Bellman complete value function class.

Let $\cG$ be an abstract set of hypotheses. We consider and online
learning process parameterized by a positive integer $H$ and $\alpha \in [0,1]$. 
\begin{itemize}
    \item There are $2H+1$ outcome spaces,
      $\set{\Xclass_h}_{h\in\brk{H}}$, $\crl{\Yclass_h}_{h\in [H]}$,
      and $\Zclass$.
    \item There are $H$ unknown probability kernels $\set{\Kernel_h:\Xclass_h\to\Yclass_h}_{h\in [H]}$.
    \item There are $2H+1$ known loss functions
      $\set{\loss_{h,1}:\Xclass_h\times \cG \to [0,1]}_{h\in\brk{H}}$,
      $\crl{\loss_{h,2}:\Yclass_h\times \cG \to [0,1]}_{h\in H}$, and $l_3:\Zclass\times \cG \to [0,1]$.
    \end{itemize}
    Define $\histSet\ind{0} = \emptyset$. We consider the following
    process. For $t=1,\ldots,T$: 
\begin{itemize}
\item Learner predicts a (randomized) hypothesis $f\ind{t}\in\Gclass$.
\item Nature reveals $\set{\x_h\ind{t},\y_h\ind{t}}_{h\in [H]}$ and
  $\z\ind{t}$. The outcomes $\crl{\x_h\ind{t}}_{h\in\brk{H}}$ can be
  chosen adaptively, but they are mutually
independent given $\histSet\ind{t-1}$. The outcome $\z\ind{t}$ can be
chosen adaptively based on $\histSet\ind{t-1}$ and $\set{\x_h\ind{t}}_{h\in [H]}$. 
Each outcome $\y_h\ind{t}\sim \Kernel_h(\x_h\ind{t})$ is drawn independently for each $h \in [H]$. %
\item The history is updated via $\histSet\ind{t} \gets \histSet\ind{t-1}\bigcup \set{\x_h\ind{t},\y_h\ind{t}}_{h\in [H]}\bigcup\set{\z\ind{t}}$.
\item The learner suffers loss
\begin{align}
  \frac{1}{H}\sum\limits_{h=1}^{H}(\loss_{h,1}(\x_h\ind{t},f\ind{t}) - \En[\loss_{h,2}(\y_h,f\ind{t})\mid{}\x_h\ind{t}])^2 + \alpha l_3(\z\ind{t},f\ind{t}). \label{eq:ol_loss}
\end{align}
\end{itemize}

The learner's goal is to minimize a form of regret for the cumulative
loss given in \cref{eq:ol_loss}. This loss function reflects two objective
. The first objective involves the $H$ losses
$\set{\loss_{1,h}}_{h\in [H]}$ with corresponding outcomes
$\set{\x_{h}}_{h\in [H]}$, as well as the $H$ losses
$\set{\loss_{2,h}}_{h\in [H]}$ tied to outcomes $\set{\y_{h}}_{h\in
  [H]}$, which are generated stochastically based on $\set{\x_h}_{h\in [H]}$. The primary objective is to minimize the primary error
$$\frac{1}{H}\sum\limits_{t=1}^{T}\sum\limits_{h=1}^{H}(\loss_{h,1}(\x_h\ind{t},f\ind{t}) - \En[\loss_{h,2}(\y_h,f\ind{t})\mid{}\x_h\ind{t}])^2$$
The secondary objective is to minimize $ \sum\limits_{t=1}^{T}
l_3(\z\ind{t},f\ind{t})$, and the ultimate goal is to minimize a weighted sum of the two objectives.

This online learning setup, adapted from \citet{agarwal2022non},
generalizes the reinforcement learning setting in which
\cref{alg:ts3_rl} operates. The adaptively chosen outcome $\x_h\ind{t}$
corresponds to the state-action pair at the $h$-th step,
$(s_h\ind{t}, a_h\ind{t})$, with the policy at time $t$ chosen in an
adaptive, potentially adversarial fashion. The conditionally
stochastic outcome $\y_h\ind{t}$ corresponds to the reward and the
next state $(r_h\ind{t}, s_{h+1}\ind{t})$, which is sampled
independently from the MDP's reward distribution and transition distribution at step
$h$, and is conditionally independent given $(s_h\ind{t},
a_h\ind{t})$. The learner's objective in the RL framework is to predict a value
function $f\ind{t} = Q\ind{t}$ that minimizes the squared Bellman error,
realized by selecting the losses $\loss_{1,h}(\x_h\ind{t}, f\ind{t}) =
Q_h\ind{t}(s_h\ind{t},a_h\ind{t})$ and $\loss_{2,h}(\y_h\ind{t},
f\ind{t}) = r_h\ind{t} + \max_a Q_{h+1}\ind{t}(s_{h+1}\ind{t},a)$. The
secondary objective is to predict the value function optimistically,
with $\loss_3(\z\ind{t},f\ind{t}) = -\max_{a}Q(s_1\ind{t},a)$ and $\z\ind{t} = s_1\ind{t}$.

\begin{algorithm}[htp]
  \caption{Two-Timescale Exponential Weights (adapted from \citet{agarwal2022non})}
\label{alg:ts3}
\begin{algorithmic}[1]
  \State Initialize $S_0\gets\emptyset$.
\For{$t=1,2,\dots,T$}
\State For all $f,g\in \Gclass$, define 
\begin{gather*}
  \Delta_h\ind{t}(g,f) \ldef \loss_{h,1}(\x_h\ind{t},g) - \loss_{h,2}(\y_h\ind{t},f),\\
  q\ind{t}(g\mid{}f) \ldef q\ind{t}(g\mid{}f,\histSet\ind{t-1})  \propto \exp\prn*{-\lambda \cdot \frac{1}{H}\sum\limits_{s=1}^{t-1} \sum\limits_{h=1}^{H} \Delta_h\ind{s}(g,f)^2},\\
  L\ind{t}(f) \ldef \frac{1}{H}\sum\limits_{h=1}^{H}\Delta_h\ind{t}(f,\ngedit{f})^2 + \frac{1}{\lambda} \log \prn*{\En_{g\sim q\ind{t}(\cdot\mid{}f)} \brk*{ \exp \prn*{-\lambda\cdot \frac{1}{H} \sum\limits_{h=1}^{H}\Delta_h\ind{t}(g,f)^2}  }}   ,\\ %
  p\ind{t}(f) \ldef p\ind{t}(f\mid{}\histSet\ind{t-1}) \propto \exp\prn*{ -\eta \sum\limits_{s=1}^{t-1} (\coef l_3(\z\ind{s},\ngedit{f}) + L\ind{s}(f))  }. %
\end{gather*}
\State Sample and predict $f\ind{t}\sim p\ind{t}$.
\State Observe $\set{\x_h\ind{t},\y_h\ind{t}}_{h\in [H]}$, $\z\ind{t}$.  and update $S_t \gets S_{t-1}\bigcup \set{\x_h\ind{t},\y_h\ind{t}}_{h\in [H]}\bigcup\set{\z\ind{t}}$.
\EndFor
\end{algorithmic}
\end{algorithm}

To analyze \cref{alg:ts3}, we make a generalized realizability assumption and
a generalized completeness assumption; these assumptions abstract away the
notions of realizability and completeness in RL.

\begin{assumption}[Realizability for online learning]
\label{ass:realizability-OL}
There exists $f^*\in \Gclass$ such that for all $h\in [H]$ and $\x_h\in \Xclass_h$, we have
\begin{align*}
\loss_{h,1}(\x_h,f^*) = \En[\loss_{h,2}(\y_h,f^*)\mid{}\x_h]. 
\end{align*}
\end{assumption}

\begin{assumption}[Completeness for online learning]
\label{ass:completeness-OL}
For any $f\in\cG$, there exists $g\in\cG$ such that for all $h\in [H]$ and $\x_h\in \Xclass_h$, we have
\begin{align*}
\loss_{h,1}(\x_h,g) = \En[\loss_{h,2}(\y_h,f)\mid{}\x_h].
\end{align*}
For any $f\in\cG$, we denote the corresponding $g\in\cG$ satisfying
this property by $g = \cT f$.
\end{assumption}

For functions $g$ and $f$ and outcome $\x_h$, define
\begin{align*}
\cE_h(g,f,\x_h) \ldef \loss_{h,1}(\x_h,g) - \En[\loss_{h,2}(\y_h,f)\mid{}\x_h];
\end{align*}
this quantity generalizes the notion of Bellman error for
reinforcement learning. Recalling that $\Delta_h\ind{t}(f,g) \ldef \loss_{h,1}(\x_h\ind{t},f) -
\loss_{h,2}(\y_h\ind{t},g)$, it follows immediately
\begin{align*}
\En [\Delta_h\ind{t}(g,f) \mid{} \x_h\ind{t}]=  \En [ (\loss_{h,1}(\x_h\ind{t},g) - \loss_{h,2}(\y_h\ind{t},f))\mid{}\x_h\ind{t}] =  \cE_h(g,f,\x_h\ind{t}).
\end{align*}
In addition, let us define
\begin{align*}
    \cE(g,f,x_{1:H})^2 &\ldef
                         \frac{1}{H}\sum\limits_{h=1}^{H}\prn*{\cE_h(g,f,\x_h)}^2,
                         \intertext{and}
    \iota\ind{t}(f) &\ldef  l_3(\z\ind{t},f) - l_3(\z\ind{t},f^*).                         
\end{align*}

The following result is the main theorem concerning the performance of \cref{alg:ts3}.

\begin{theorem}
\label{thm:main-OL}
Let $\lambda = 1/8$, $\eta < 1/(2^{16}(\log (|\Gclass|T/\delta) + 1))$, and $0<\coef<1$. Under \pref{ass:realizability-OL} and \pref{ass:completeness-OL}, 
for any $\delta \in (0,1)$, with probability at least $1-\delta$,
\begin{align*}
  &\frac{1}{H} \sum\limits_{t=1}^{T}\sum\limits_{h=1}^{H}\En_{t-1} \brk*{\En_{f\sim p\ind{t}}\ngedit{\left[(\loss_{h,1}(\x_h\ind{t},f) - \En[\loss_{h,2}(\y_h,f)\mid{}\x_h\ind{t}])^2 + 12\coef l_3(\z\ind{t},f)\right] } \mid{} \x_h\ind{t}}\\
  &- \prn*{\frac{1}{H}\sum\limits_{t=1}^{T}\sum\limits_{h=1}^{H}(\loss_{h,1}(\x_h\ind{t},f^*) - \En[\loss_{h,2}(\y_h,f^*)\mid{}\x_h\ind{t}])^2 + 12\coef l_3(\z\ind{t},f^*)} \\
  &=\sum\limits_{t=1}^{T}  \En_{t-1}\brk*{\En_{f\sim p\ind{t}} \brk*{\cE(f,f,x_{1:H}\ind{t})^2 + 12\coef \iota\ind{t}(f)} \mid{} \x_{1:H}\ind{t}}\\
  &\leq 2^{16}( \eta\coef\log(|\Gclass|T/\delta)T  + \log(|\Gclass|)/\eta  + \coef^2 T).
\end{align*}
\end{theorem}

\subsubsection{Proof of \cref{thm:main-OL}}
For our analysis, it will be useful to consider the following offset version of the loss:
\begin{align*}
    \delta_h\ind{t}(g,f) &\ldef \Delta_h\ind{t}(g,f)^2 - (\En[\loss_{h,2}(\y_h,f)\mid{}\x_h\ind{t}] - \loss_{h,2}(\y_h\ind{t},f) )^2.
\end{align*}
We use $x_{1:H}$ as a shorthand for $\set{\x_h}_{h\in [H]}$ and further define 
\begin{align*}
    \delta\ind{t}(g,f) &\ldef  \frac{1}{H} \sum\limits_{h=1}^{H}\delta_h\ind{t}(g,f),\\
    \delta\ind{t}(f) &\ldef \En_{g\sim{}q\ind{t}(\cdot\mid{}f)} \delta\ind{t}(g,f),\\
    Z\ind{t}(f) &\ldef - \frac{1}{\lambda}\log\En_{g\sim q\ind{t}(\cdot\mid{}f)} \exp\prn*{-\lambda  \delta\ind{t}(g,f)},\\
    Z\ind{t} &= - \frac 1\eta \log \E_{f \sim p\ind{t}} \exp \left( - \eta \cdot [\beta \iota\ind{t}(f) + \delta\ind{t}(f,f) - Z\ind{t}(f)]\right).
\end{align*}
Recall from \cref{alg:ts3} that
\begin{align*}
  q\ind{t}(g\mid{}f) &\ldef q\ind{t}(g\mid{}f,\histSet\ind{t-1})  \propto \exp\prn*{-\lambda \cdot \frac{1}{H}\sum\limits_{s=1}^{t-1} \sum\limits_{h=1}^{H} \Delta_h\ind{s}(g,f)^2},\\
  p\ind{t}(f) &\ldef p\ind{t}(f\mid{}\histSet\ind{t-1}) \propto \exp\prn*{ -\eta \sum\limits_{s=1}^{t-1} (\coef l_3(\z\ind{s},\ngedit{f}) + L\ind{s}(f))  }. 
\end{align*}
Thus, we can verify the following relationships:%
\begin{align}
q\ind{t}(g\mid{}f) &= \frac{\exp\prn*{-\lambda \sum\limits_{s=1}^{t-1}  \delta\ind{s}(g,f)}}{\sum\limits_{g'\in \Gclass} \exp\prn*{-\lambda \sum\limits_{s=1}^{t-1}  \delta\ind{s}(g',f)} }, \nonumber\\
p\ind{t}(f) &=  \frac{\exp\prn*{-\eta \sum\limits_{s=1}^{t-1} [\coef \iota\ind{s}(f) + \delta\ind{s}(f,f) -Z\ind{s}(f) ]  }}{\sum\limits_{f'\in \Gclass}\exp\prn*{-\eta \sum\limits_{s=1}^{t-1}[ \coef \iota\ind{s}(f') + \delta\ind{s}(f',f') -Z\ind{s}(f')  ] }},\label{eq:ptf-rephrasing}\\ %
q\ind{t+1}(g\mid{}f) &= q\ind{t}(g\mid{}f) \cdot e^{-\lambda [ \delta\ind{t}(g,f) - Z\ind{t}(f) ]},\nonumber\\
p\ind{t+1}(f) &= p\ind{t}(f) \cdot e^{-\eta[ \coef \iota\ind{t}(f) + \delta\ind{t}(f,f) -Z\ind{t}(f) - Z\ind{t} ]}.\nonumber
\end{align}
In what follows, we use $\En_{t-1}\brk{\cdot}$ to \dfedit{abbreviate
$\En[\cdot\mid{}\filt\ind{t-1}]$, where $\filt\ind{t-1}\ldef{}\sigma(S\ind{t-1})$.}

\begin{proof}[\pfref{thm:main-OL}]
  Under \pref{ass:realizability-OL}, $(\loss_{h,1}(\x_h\ind{t},f^*) - \En[\loss_{h,2}(\y_h,f^*)\mid{}\x_h\ind{t}])^2  = 0$. Thus, the first equality holds by definition as
\begin{align*}
  &\frac{1}{H} \sum\limits_{t=1}^{T}\sum\limits_{h=1}^{H}\En_{t-1} \brk*{\En_{f\sim p\ind{t}}\ngedit{\left[(\loss_{h,1}(\x_h\ind{t},f) - \En[\loss_{h,2}(\y_h,f)\mid{}\x_h\ind{t}])^2 + 12\coef l_3(\z\ind{t},f) \right]} \mid{} \x_h\ind{t}}\\
  &- \prn*{\frac{1}{H}\sum\limits_{t=1}^{T}\sum\limits_{h=1}^{H}(\loss_{h,1}(\x_h\ind{t},f^*) - \En[\loss_{h,2}(\y_h,f^*)\mid{}\x_h\ind{t}])^2 + 12\coef l_3(\z\ind{t},f^*)} \\
  &=\sum\limits_{t=1}^{T}  \En_{t-1}\brk*{\En_{f\sim p\ind{t}} \prn*{\cE(f,f,x_{1:H}\ind{t})^2 + 12\coef \iota\ind{t}(f)} \mid{} \x_{1:H}\ind{t}}.
\end{align*}

To prove the result, we appeal to two technical lemmas,
\pref{lem:refined-exp-weight-guarantee} and \pref{lem:final-z-term},
stated below. Plugging the bound from \pref{lem:final-z-term} into \pref{lem:refined-exp-weight-guarantee}, we have
\begin{align*}
    &\sum\limits_{t=1}^{T} \En_{t-1} \brk*{\En_{f\sim p\ind{t}} \cE(f,f,x_{1:H}\ind{t})^2\mid{} x_{1:H}\ind{t} } + 6\coef\sum\limits_{t=1}^{T}  \En_{t-1} \brk*{\En_{f\sim p\ind{t}} \iota(f)\mid{} x_{1:H}\ind{t} } \\
    &\leq 64 \prn*{ \frac{1}{128}\sum\limits_{t=1}^{T} \En_{t-1} \brk*{\En_{f\sim p\ind{t}}\brk*{ \cE(f,f,x_{1:H}\ind{t})^2}\mid{} x_{1:H}\ind{t} } +(48\eta^2\coef^2 + 8\eta\coef)\log (|\Gclass|T/\delta)T  + 8\eta\coef^2 T + 16 \log \abs*{\Gclass}  }\\
    &\hspace{1.5in} + \frac{6}{\eta}\log |\Gclass| + 18\coef^2T.
\end{align*}
Rearranging, we obtain
\begin{align*}
  \sum\limits_{t=1}^{T}  \En_{t-1} \brk*{\En_{f\sim p\ind{t}} \prn*{\cE(f,f,x_{1:H}\ind{t})^2 + 12\coef \iota\ind{t}(f)}\mid{} x_{1:H}\ind{t} } &\leq 2^{16}( \eta\coef\log(|\Gclass|T/\delta)T  + \log(|\Gclass|)/\eta  + \coef^2 T).
\end{align*}
\end{proof}

We now state the technical lemmas, \pref{lem:refined-exp-weight-guarantee} and \pref{lem:final-z-term}, used in the proof above

\begin{lemma}
\label{lem:refined-exp-weight-guarantee} 
Under \pref{ass:realizability-OL},
for any $0 \leq \eta\leq 1/24$ and $0<\coef <1$, we have
\begin{align*}
  \hspace{2in}&\hspace{-2in}\sum\limits_{t=1}^{T} \En_{t-1} \brk*{\En_{f\sim p\ind{t}}\cE(f,f,x_{1:H}\ind{t})^2\mid{} x_{1:H}\ind{t} }  + 6\coef\sum\limits_{t=1}^{T}  \En_{t-1} \brk*{\En_{f\sim p\ind{t}} \iota\ind{t}(f) \mid{} x_{1:H}\ind{t} } \\
    &\leq 64 \sum\limits_{t=1}^{T}\En_{t-1} \brk*{\En_{f\sim p\ind{t}} Z\ind{t}(f)\mid{} x_{1:H}\ind{t} } + 18\eta\coef^2 T  + \frac{6}{\eta}\log |\Gclass|.
\end{align*}
\end{lemma}

\begin{lemma}
\label{lem:final-z-term}
Let $\lambda = 1/8$ and $\eta < 1/(2^{16}(\log (|\Gclass|T/\delta) + 1))$. Under \pref{ass:completeness-OL}, for any $0<\delta<1$, with probability at least $1-\delta$, we have
\begin{align*}
     \sum\limits_{t=1}^{T}\En_{t-1} \brk*{\En_{f\sim p\ind{t}}  Z\ind{t}(f) \mid{} x_{1:H}\ind{t} } &\leq  \frac{1}{128}\sum\limits_{t=1}^{T} \En_{t-1} \brk*{\En_{f\sim p\ind{t}}  \cE(f,f,x_{1:H}\ind{t})^2\mid{} x_{1:H}\ind{t} }\\
     &\hspace{1in}   +(48\eta^2\coef^2 + 8\eta\coef)\log (|\Gclass|T/\delta)T  \\
     &\hspace{1in}+ 8\eta\coef^2 T + 16 \log \abs*{\Gclass} .
\end{align*}
\end{lemma}

The remainder of the proof is organized as follows.
\begin{itemize}
\item \cref{sec:complete_basic} presents basic technical lemmas.
\item \cref{subsubsec:exp-guarantee} presents the proof of
  \cref{lem:refined-exp-weight-guarantee}.
\item \cref{subsubsec:self-bound-Z} presents the proof of
  \cref{lem:final-z-term}.
\item \cref{sec:termi,sec:termii} contain additional technical
    lemmas used in the proof of \cref{lem:final-z-term}.
\end{itemize}

\subsubsection{Basic properties}
\label{sec:complete_basic}

In this section, we present basic technical results that will be used
within the proof of \pref{lem:final-z-term} and
\pref{lem:refined-exp-weight-guarantee}.

In this section, we will present some basic properties of $\delta\ind{t}(g,f)$, $\cE(g,f,x_{1:H}\ind{t})^2$ and $Z\ind{t}(f)$ for any $t\in [T],g,f\in \Gclass, x_{1:H}\ind{t}\in \bigcup_{h\in[H]} \Xclass_h$. These properties are mainly due to a sub-Gaussian term in the definition of $\delta\ind{t}(g,f)$. Recall
\begin{align*}
\delta\ind{t}(g,f) &= \frac{1}{H} \sum\limits_{h=1}^{H} (\loss_{h,1}(\x_h\ind{t},g) - \loss_{h,2}(\y_h\ind{t},f))^2 - (\En[\loss_{h,2}(\y_h,f)\mid{}\x_h\ind{t}] - \loss_{h,2}(\y_h\ind{t},f))^2 \\
&= \frac{1}{H} \sum\limits_{h=1}^{H}\Large( (\loss_{h,1}(\x_h\ind{t},g)- \En[\loss_{h,2}(\y_h,f)\mid{}\x_h\ind{t}] )^2 \\
&\quad + 2 (\En[\loss_{h,2}(\y_h,f)\mid{}\x_h\ind{t}]- \loss_{h,2}(\y_h\ind{t},f) ) (\loss_{h,1}(\x_h\ind{t},g)  - \En[\loss_{h,2}(\y_h,f)\mid{}\x_h\ind{t}] )\Large)\\
&=   \cE(g,f,x_{1:H}\ind{t})^2 + \frac{1}{H} \sum\limits_{h=1}^{H}  2 (\En[\loss_{h,2}(\y_h,f)\mid{}\x_h\ind{t}]- \loss_{h,2}(\y_h\ind{t},f) ) \cE_h(g,f,\x_h\ind{t}).
\end{align*}
Note that the quantity $\En[\loss_{h,2}(\y_h,f)\mid{}\x_h\ind{t}]-
\loss_{h,2}(\y_h\ind{t},f) $, conditioned on $\x_h\ind{t}$, is a mean $0$ random variable bounded by $[-1,1]$. Thus, we have
\begin{align}
\label{eq:expectation-delta}
\En\brk*{ \delta\ind{t}(g,f) \mid{} x_{1:H}\ind{t} }= \En\brk*{ \cE(g,f,x_{1:H}\ind{t})^2\mid{} x_{1:H}\ind{t} }.
\end{align}
Furthermore, since $\En[\loss_{h,2}(\y_h,f)\mid{}\x_h\ind{t}]- \loss_{h,2}(\y_h\ind{t},f) $ is sub-Gaussian with variance proxy $1/2$ by Hoeffding's inequality, we have the following three lemmas.

\begin{lemma}
\label{lem:exp-control} For any $t\leq T$, $g,f\in \cG$ and $c\in \bbR$, we have
\begin{align*}
    \En\brk*{\exp \prn*{-c \delta\ind{t}(g,f)  }\mid{} x_{1:H}\ind{t} } &\leq \En\brk*{\exp \prn*{-c (1-2c) \cE(g,f,x_{1:H}\ind{t})^2  }\mid{} x_{1:H}\ind{t} }.
\end{align*}
In addition, for any $0<c< \frac{1}{2}$, we have $\En \brk*{\exp \prn*{-c \delta\ind{t}(g,f)  } \mid{} x_{1:H}\ind{t} } \leq 1$.
\end{lemma}
\begin{proof}[\pfref{lem:exp-control}]
By the $1/2$-sub-Gaussianity of $\En[\loss_{h,2}(\y_h,f)\mid{}\x_h\ind{t}]- \loss_{h,2}(\y_h\ind{t},f) $, we have
\begin{align*}
    & \En \brk*{\exp \prn*{-c \delta\ind{t}(g,f)  } \mid{} x_{1:H}\ind{t} } \\
    &=  \En \brk*{ \exp (-c\cE(g,f,x_{1:H}\ind{t})^2 )\mid{} x_{1:H}\ind{t} }\En \brk*{ \exp\prn*{- \frac{2c}{H} \sum\limits_{h=1}^{H}   (\En[\loss_{h,2}(\y_h,f)\mid{}\x_h\ind{t}]- \loss_{h,2}(\y_h\ind{t},f) ) \cE_h(g,f,\x_h\ind{t})}\mid{} x_{1:H}\ind{t} }\\
    &\leq \En \brk*{\exp \prn*{-c (1-2c/H) \cE(g,f,x_{1:H}\ind{t})^2   } \mid{} x_{1:H}\ind{t} } \\
    &\leq \En\brk*{\exp \prn*{-c (1-2c) \cE(g,f,x_{1:H}\ind{t})^2   } \mid{} x_{1:H}\ind{t} }.
\end{align*}
\end{proof}

\begin{lemma}
\label{lem:4-th-moment-controlled-by-the-2-nd-moment}
For any $t\leq T$, $g$, $f$ and $x_{1:H}\ind{t}$, we have $ \En\brk*{\delta\ind{t}(g,f)^2\mid{} x_{1:H}\ind{t} } \leq  5\En\brk*{ \cE(g,f,x_{1:H}\ind{t})^2\mid{} x_{1:H}\ind{t} }$.
\end{lemma}

\begin{proof}[\pfref{lem:4-th-moment-controlled-by-the-2-nd-moment}]
The result follows by writing %
\begin{align*}
    \En\brk*{ \delta\ind{t}(g,f)^2\mid{} x_{1:H}\ind{t} } &= \En\brk*{\prn*{\cE(g,f,x_{1:H}\ind{t})^2 + \frac{1}{H} \sum\limits_{h=1}^{H}  2 (\En[\loss_{h,2}(\y_h,f)\mid{}\x_h\ind{t}]- \loss_{h,2}(\y_h\ind{t},f) ) \cE_h(g,f,\x_h\ind{t})}^2 \mid{} x_{1:H}\ind{t} }\\
    &=      \En\brk*{\cE(g,f,x_{1:H}\ind{t})^4  +  \frac{4}{H^2} \prn*{\sum\limits_{h=1}^{H} (\En[\loss_{h,2}(\y_h,f)\mid{}\x_h\ind{t}]- \loss_{h,2}(\y_h\ind{t},f) ) \cE_h(g,f,\x_h\ind{t})}^2\mid{} x_{1:H}\ind{t} }  \\
    &\leq   \En\brk*{\cE(g,f,x_{1:H}\ind{t})^2  +  \frac{4}{H^2} \sum\limits_{h=1}^{H} \prn*{(\En[\loss_{h,2}(\y_h,f)\mid{}\x_h\ind{t}]- \loss_{h,2}(\y_h\ind{t},f) ) \cE_h(g,f,\x_h\ind{t})}^2\mid{} x_{1:H}\ind{t} } \\
    &\leq 5  \En\brk*{  \cE(g,f,x_{1:H}\ind{t})^2\mid{} x_{1:H}\ind{t} },
\end{align*}
where the first equality is by defintion, the second equality is by expanding the terms and notice that the cross terms have zero mean, the first inequality is obtained using the fact that $|\cE_h(g,f,x_h)| \leq 1$ and that the cross terms of the expansion of the second term have zero mean and the final inequality is by  $| \En[\ell_{h,2}(y_h, f) \mid x_h\ind{t}] - \ell_{h,2}(y_h\ind{t}, f)| \leq 1$.
\end{proof}

\begin{lemma}
\label{lem:g-Qstar-positivity}
For all $t\leq T$, as long as $0\leq \lambda \leq 1/8$, we have that
for all $f$,
\begin{align*}
    4\En_{t-1}[ Z\ind{t}(f)\mid{}x_{1:H}\ind{t}] \ge \En_{g\sim q\ind{t}(\cdot\mid{}f,\histSet\ind{t-1})} \brk*{\cE(g,f,x_{1:H}\ind{t})^2}  \geq 0
\end{align*}
almost surely. In particular, $ \En_{t-1}[ (Z\ind{t}(f^*))\mid{}x_{1:H}\ind{t}] \geq 0$.
\end{lemma}
\begin{proof}[\pfref{lem:g-Qstar-positivity}]
Recall $Z\ind{t}(f) = \ngedit{-}\frac{1}{\lambda} \log\En_{g\sim q\ind{t}(\cdot\mid{}f)} \exp(-\lambda \delta\ind{t}(g,f))$, thus we have
\begin{align*}
-\lambda \En_{t-1} [ Z\ind{t}(f)\mid{}x_{1:H}\ind{t}] &=  \En_{t-1} [ \log\En_{g\sim q\ind{t}(\cdot\mid{}f)} \exp(-\lambda \delta\ind{t}(g,f)) \mid{} x_{1:H}\ind{t}]\\
&\leq  \log \En_{t-1} [ \En_{g\sim q\ind{t}(\cdot\mid{}f)} \exp(-\lambda \delta\ind{t}(g,f)) \mid{} x_{1:H}\ind{t}] \tag{Jensen}\\
&\leq  \log \En_{t-1} \brk*{\En_{g\sim q\ind{t}(\cdot\mid{}f)}\exp \prn*{-\lambda (1-2\lambda) \cE(g,f,x_{1:H}\ind{t})^2  }\mid{} x_{1:H}\ind{t} } \tag{\pref{lem:exp-control}}\\
&\leq   \log \En_{t-1} \brk*{\En_{g\sim q\ind{t}(\cdot\mid{}f)} \prn*{1- \frac{1}{2} \lambda (1-2\lambda) \cE(g,f,x_{1:H}\ind{t})^2 }\mid{} x_{1:H}\ind{t} }  \tag{\pref{lem:technical}} \\
&\leq -\frac{\lambda}{4}\En_{t-1} \brk*{ \En_{g\sim q\ind{t}(\cdot\mid{}f)} \cE(g,f,x_{1:H}\ind{t})^2\mid{} x_{1:H}\ind{t} } = -\frac{\lambda}{4}\En_{g\sim q\ind{t}(\cdot\mid{}f)} \cE(g,f,x_{1:H}\ind{t})^2.
\end{align*}

\end{proof}
Recalling that  $Z\ind{t}(f) = -\frac{1}{\lambda}\log\En_{g\sim q\ind{t}(\cdot\mid{}f)} \exp(-\lambda \delta\ind{t}(g,f))$ and $ \delta\ind{t}(f) = \En_{g\sim q\ind{t}(\cdot\mid{}f)} \delta\ind{t}(g,f)$, we obtain the following relationship via Jensen's inequality and second order expansion of the exponential function.
\begin{lemma}
\label{lem:z-to-delta}
For any $t\leq T$ and realization of $\histSet\ind{t-1}$, we have that
as long as $\lambda < 1/8$, $ Z\ind{t}(f) \leq \delta\ind{t}(f) $ and
$\abs*{Z\ind{t}(f)} \leq   \frac{9}{8} \En_{g\sim
  q\ind{t}(\cdot\mid{}f)} \abs*{\delta\ind{t}(g,f)}$ for all $g,f\in\cG$.
\end{lemma}

\begin{proof}[\pfref{lem:z-to-delta}]
By Jensen's inequality,
\begin{align*}
Z\ind{t}(f) = - \frac{1}{\lambda} \log \En_{g\sim q\ind{t}(\cdot\mid{}f)}\exp \prn*{-\lambda \delta\ind{t}(g,f)} \leq \En_{g\sim q\ind{t}(\cdot\mid{}f)} \delta\ind{t}(g,f) = \delta\ind{t}(f). 
\end{align*}
On the other hand, applying \pref{lem:technical-2}, we have
\begin{align*}
-\lambda Z\ind{t}(f) &= \log \En_{g\sim q\ind{t}(\cdot\mid{}f)} \exp \prn*{-\lambda \delta\ind{t}(g,f)} \\
&\leq  \En_{g\sim q\ind{t}(\cdot\mid{}f)} \prn*{-\lambda \delta\ind{t}(g,f) + \lambda^2\delta\ind{t}(g,f)^2  }. 
\end{align*}
Thus 
\begin{align*}
\abs*{Z\ind{t}(f)}  \leq  \abs*{\delta\ind{t}(f)} + \lambda\En_{g\sim q\ind{t}(\cdot\mid{}f)} \delta\ind{t}(g,f)^2 \leq  \frac{9}{8} \En_{g\sim q\ind{t}(\cdot\mid{}f)} \abs*{\delta\ind{t}(g,f)}.
\end{align*}
\end{proof}

Another important fact we will use is that $f^*$ satisfies $\delta\ind{t}(f^*,f^*) = 0$. Furthermore, for any $f\in \cG$, the pair $\cT f,f$ always has $\delta\ind{t}(\cT f,f) = 0$ for all $t\in [T]$.

\begin{lemma}
\label{lem:delta_Ta_a}
For any $t\leq T$, under \pref{ass:realizability-OL} we have
$\delta\ind{t}(f^*,f^*) = 0$. In addition, if \pref{ass:completeness-OL} is satisfied, then for any $f$, $\delta\ind{t}(\cT f,f) = 0$.
\end{lemma}

\begin{proof}[\pfref{lem:delta_Ta_a}]
By the definition of $f^*$,
\begin{align*}
  \delta\ind{t}(f^*,f^*) &= \frac{1}{H}\sum\limits_{h=1}^{H} (\loss_{h,1}(\x_h\ind{t},f^*) - \loss_{h,2}(\y_h\ind{t},f^*))^2 - (\En[\loss_{h,2}(\y_h,f^*)\mid{}\x_h\ind{t}] - \loss_{h,2}(\y_h\ind{t},f^*))^2 = 0.
\end{align*}
Likewise, by the definition of $\cT f$, 
\begin{align*}
    \delta\ind{t}(\cT f,f) &= \frac{1}{H}\sum\limits_{h=1}^{H} (\loss_{h,1}(\x_h\ind{t},\cT f) - \loss_{h,2}(\y_h\ind{t},f))^2 - (\En[\loss_{h,2}(\y_h,f)\mid{}\x_h\ind{t}] - \loss_{h,2}(\y_h\ind{t},f))^2 = 0.
\end{align*}
\end{proof}

\subsubsection{\pfref{lem:refined-exp-weight-guarantee}}
\label{subsubsec:exp-guarantee}
In this section, we prove \cref{lem:refined-exp-weight-guarantee}, which
gives a guarantee for the outer exponential weights update used within \cref{alg:ts3}.

\begin{proof}[\pfref{lem:refined-exp-weight-guarantee}] The definition
  of the exponential weights update (in particular, \cref{eq:ptf-rephrasing}) implies that 
\begin{align*}
\hspace{1.5in}&\hspace{-1.5in} -\eta \cdot  \min_{f} \prn*{ \sum\limits_{t=1}^{T} \coef \iota\ind{t}(f) +  \delta\ind{t}(f,f)-   Z\ind{t}(f) } - \log |\Gclass|\\
&\leq \sum\limits_{t=1}^{T}\log \prn*{ \En_{f\sim p\ind{t}} \exp\prn*{ -\eta \prn*{  \coef\iota\ind{t}(f) + \delta\ind{t}(f,f) -   Z\ind{t}(f) }  }} \\
&\leq  \frac{1}{3}\sum\limits_{t=1}^{T}\log\prn*{ \En_{f\sim p\ind{t}} \exp\prn*{ -3\eta  \coef\iota\ind{t}(f)}} + \frac{1}{3}\sum\limits_{t=1}^{T} \log \prn*{ \En_{f\sim p\ind{t}} \exp\prn*{ -3\eta  \delta\ind{t}(f,f)} }\\
&\hspace{1.75in} +\frac{1}{3}\sum\limits_{t=1}^{T} \log \prn*{\En_{f\sim p\ind{t}} \exp(3 \eta Z\ind{t}(f) )},
\end{align*}
\ngedit{where the final inequality holds due to the fact that $\En[XYZ] \leq \sqrt[3]{\En[X^3] \En[Y^3] \En[Z^3]}$ for} \jqedit{positive} \ngedit{random variables $X,Y,Z$ (which in turn can be shown via a repeated application of H\"{o}lder's inequality). }
We further have
\begin{align*}
  \hspace{1in}&\hspace{-1in}\En_{t-1} \brk*{ \log \prn*{ \En_{f\sim p\ind{t}} \exp\prn*{ -3\eta  \delta\ind{t}(f,f)}} \mid{} x_{1:H}\ind{t} }\\
    &\leq\log\prn*{ \En_{t-1} \brk*{\En_{f\sim p\ind{t}} \exp\prn*{ -3\eta  \delta\ind{t}(f,f)} \mid{} x_{1:H}\ind{t} }}\tag{Jensen}\\
    &\leq\log\prn*{ \En_{t-1} \brk*{\En_{f\sim p\ind{t}}  \exp(-  \eta (1- 6\eta)  \cE(f,f,x_{1:H}\ind{t})^2 )  \mid{} x_{1:H}\ind{t} }}\tag{\pref{lem:exp-control}}\\
    &\leq \log\prn*{ \En_{t-1} \brk*{ \En_{f\sim p\ind{t}} (1 -  \eta (1- 6\eta)  \cE(f,f,x_{1:H}\ind{t})^2)  \mid{} x_{1:H}\ind{t} }}\tag{\pref{lem:technical}}\\
    &\leq - \frac{\eta}{2} \En_{t-1} \brk*{\En_{f\sim p\ind{t}}\cE(f,f,x_{1:H}\ind{t})^2\mid{} x_{1:H}\ind{t} } .
\end{align*}
Meanwhile, using \pref{lem:technical-2}, we get
\begin{align*}
\log \prn*{\En_{t-1} \brk*{\En_{f\sim p\ind{t}} \exp\prn*{ -3\eta \coef\iota\ind{t}(f) }\mid{} x_{1:H}\ind{t} } }
  \leq  \En_{t-1} \brk*{\En_{f\sim p\ind{t}}  (-3\eta \coef\iota\ind{t}(f) +  9\eta^2\coef^2 (\iota\ind{t}(f))^2)\mid{} x_{1:H}\ind{t} }. 
\end{align*}
Moreover, under \pref{ass:realizability-OL}, the benchmark term is negative by
\begin{align*}
  \hspace{1.5in}&\hspace{-1.5in}\En_{t-1} \brk*{\min_{f} \prn*{ \sum\limits_{t=1}^{T} \delta\ind{t}(f,f)-   Z\ind{t}(f)  + \coef\iota\ind{t}(f)} \mid{} x_{1:H}\ind{t} }\\
  &\leq  \En_{t-1} \brk*{\sum\limits_{t=1}^{T} \delta\ind{t}(f^*,f^*)-   Z\ind{t}(f^*) + \coef\iota\ind{t}(f^*) \mid{} x_{1:H}\ind{t} }\\
    &= - \sum\limits_{t=1}^{T} \En_{t-1} \brk*{Z\ind{t}(f^*)\mid{} x_{1:H}\ind{t} } \tag{\pref{lem:delta_Ta_a} and the definition of $\iota\ind{t}$} \\
    &\leq 0. \tag{\pref{lem:g-Qstar-positivity}}
\end{align*}

Furthermore, the log-exponential term is---up to a constant---bounded by its first order expansion:
\begin{align*}
  \hspace{1in}&\hspace{-1in}\sum\limits_{t=1}^{T} \log\prn*{ \En_{t-1} \brk*{\En_{f\sim p_t} \exp\prn*{ 3\eta Z\ind{t}(f) } \mid{} x_{1:H}\ind{t} }}\\
    &\leq \sum\limits_{t=1}^{T} \log\prn*{ \En_{t-1} \brk*{\En_{f\sim p_t} \exp\prn*{ 3\eta \delta\ind{t}(f) } \mid{} x_{1:H}\ind{t} }}\tag{\pref{lem:z-to-delta}}\\
    &=  \sum\limits_{t=1}^{T} \log\prn*{ \En_{t-1} \brk*{\En_{f\sim p\ind{t}}\exp\prn*{ 3\eta \En_{g\sim q\ind{t}(\cdot\mid{}f)}  \delta\ind{t}(g,f) }\mid{} x_{1:H}\ind{t} } }\\
    &\leq \sum\limits_{t=1}^{T} \log\prn*{ \En_{t-1} \brk*{\En_{f\sim p\ind{t}}\En_{g\sim q\ind{t}(\cdot\mid{}f)}\exp\prn*{ 3\eta   \delta\ind{t}(g,f) } \mid{} x_{1:H}\ind{t} }}\tag{by Jensen} \\
    &\leq   \sum\limits_{t=1}^{T} \log\prn*{ \En_{t-1} \brk*{\En_{f\sim p\ind{t}}\En_{g\sim q\ind{t}(\cdot\mid{}f)}  \exp \prn*{3\eta(1+6\eta)\cE(g,f,x_{1:H}\ind{t})^2}\mid{} x_{1:H}\ind{t} }} \tag{\pref{lem:exp-control}}\\
    &\leq \sum\limits_{t=1}^{T} \En_{t-1} \brk*{\En_{f\sim p\ind{t}}\En_{g\sim q\ind{t}(\cdot\mid{}f)}  6\eta(1+6\eta)\cE(g,f,x_{1:H}\ind{t})^2 \mid{} x_{1:H}\ind{t} }  \tag{\pref{lem:technical}}\\
    &\leq 32\eta \sum\limits_{t=1}^{T}\En_{t-1} \brk*{\En_{f\sim p\ind{t}} Z\ind{t}(f)\mid{} x_{1:H}\ind{t} } \tag{\pref{lem:g-Qstar-positivity}}.
\end{align*}
\ngedit{
By using the fact that $|\iota\ind{t}(f)| \leq 1$ and combining the above displays, we obtain that
\begin{align*}
-\log |\MG| \leq & -\frac{\eta}{6} \En_{t-1} \left[ \En_{f \sim p\ind{t}} \cE(f, f, x_{1:H}\ind{t})^2 \mid x_{1:H}\ind{t} \right] + \frac{32\eta}{3} \sum_{t=1}^T \E_{t-1} [\E_{f \sim p\ind{t}} Z\ind{t}(f) \mid x_{1:H}\ind{t}] + 3 \eta^2 \beta^2 - \E_{t-1}[\E_{f \sim p\ind{t}}[ \eta \beta \iota\ind{t}(f)]]\nonumber.
\end{align*}
Rearranging yields the desired statement.
}
\end{proof}

\subsubsection{\pfref{lem:final-z-term}}
\label{subsubsec:self-bound-Z}

We state three technical lemmas,
\cref{lem:Z-term,lem:term-I,lem:term-II}, which are proven in subsequent
subsections, then prove \cref{lem:final-z-term} as a consequence.

\begin{lemma}
\label{lem:Z-term}
Almost surely with respect to the draw of $\histSet\ind{T}$, we have
\begin{align*}
  \hspace{1in}&\hspace{-1in}\sum\limits_{t=1}^{T} \En_{t-1} \brk*{ \En_{f\sim p\ind{t}(\cdot\mid{}\histSet\ind{t-1})} Z\ind{t}(f)\mid{} x_{1:H}\ind{t} } \\
&\leq \frac{1}{\lambda}
\underbrace{\sum\limits_{t=1}^{T} \En_{t-1} \brk*{ \En_{p\ind{t}}\brk*{ \prn*{ e^{-\eta[ \coef\iota\ind{t}(f)+ \delta\ind{t}(f,f) -Z\ind{t}(f) - Z\ind{t} ]} -1 }\log \frac{1}{q\ind{t}(\cT f\mid{}f)}}\mid{} x_{1:H}\ind{t} } }_{\mathrm{Term~I}}  \\
&\quad - \underbrace{\sum\limits_{t=1}^{T} \En_{t-1} \brk*{\En_{p\ind{t}} \brk*{\prn*{ e^{-\eta[ \coef\iota\ind{t}(f)+  \delta\ind{t}(f,f) -Z\ind{t}(f) - Z\ind{t} ]} -1 }   Z\ind{t}(f)}\mid{} x_{1:H}\ind{t} }}_{\mathrm{Term~II}} + \frac{1}{\lambda} \log \abs*{\Gclass}  .
\end{align*}
\end{lemma}

\begin{lemma}
\label{lem:term-I}
Under \pref{ass:completeness-OL}, for any $0<\delta<1$, with probability at least
$1-\delta$ over the draw of $\histSet\ind{T}$, we have for all $t\leq
T$,
\begin{align*}
  \hspace{1in}&\hspace{-1in}\En_{t-1}\brk*{\En_{f\sim p\ind{t}} \prn*{   e^{-\eta\brk*{\coef\iota\ind{t}(f) + \delta\ind{t}(f,f) -Z\ind{t}(f) - Z\ind{t}  }}  -1 } \log \frac{1}{q\ind{t}(\cT f\mid{}f)}\mid{} x_{1:H}\ind{t} }\\
    &\leq 32\eta \En_{t-1}\brk*{\En_{f\sim p\ind{t}} \brk*{ \prn*{Z\ind{t}(f)+ \cE(f,f,x_{1:H}\ind{t})^2  }\log (\abs*{\Gclass}T/\delta)  } \mid{} x_{1:H}\ind{t} }\\
    &\hspace{1in}  +(24\eta^2\coef^2 + 4\eta\coef)\log (\abs*{\Gclass}T/\delta). 
\end{align*}
\end{lemma}

\begin{lemma}
\label{lem:term-II}
For any $t\leq T$, the following bound holds almost surely with respect to the draw of $\histSet\ind{t}$:
 \begin{align*}
  \hspace{1.5in}&\hspace{-1.5in}\En_{t-1}\brk*{ \En_{f\sim p\ind{t}} \brk*{\prn*{   e^{-\eta\brk*{\coef\iota\ind{t}(f) +\delta\ind{t}(f,f) -Z\ind{t}(f) - Z\ind{t}  }}  -1 }   Z\ind{t}(f)}\mid{} x_{1:H}\ind{t} } \\
    &\leq \eta \En_{t-1}\brk*{ \En_{f\sim p\ind{t}}\brk*{15 \cE(f,f,x_{1:H}\ind{t})^2  +280 Z\ind{t}(f)   }\mid{} x_{1:H}\ind{t} } + 4\eta\beta^2.
 \end{align*}
\end{lemma}

\begin{proof}[\pfref{lem:final-z-term}]
    By \pref{lem:Z-term}, we have
\begin{align*}
    \sum\limits_{t=1}^{T} \En_{t-1}\brk*{ \En_{f\sim p\ind{t}}  Z\ind{t}(f) \mid{} x_{1:H}\ind{t} } &\leq  \frac{1}{\lambda} \cdot (\text{Term I}) - (\text{Term II})  + \frac{1}{\lambda} \log \abs*{\Gclass}. 
\end{align*}
Thus under \pref{ass:completeness-OL}, using the bound in
\pref{lem:term-I} for Term I and using the bound in \pref{lem:term-II}
for Term II, we have with probability at least $1-\delta$,
\begin{align*}
  \hspace{.5in}&\hspace{-.5in}\sum\limits_{t=1}^{T} \En_{t-1}\brk*{ \En_{f\sim p\ind{t}}  Z\ind{t}(f)  \mid{} x_{1:H}\ind{t} }\\
    &\leq \frac{1}{\lambda} \cdot \sum\limits_{t=1}^{T} 32\eta  \En_{t-1}\brk*{ \En_{f\sim p\ind{t}} \brk*{ \prn*{Z\ind{t}(f)+ \cE(f,f,x_{1:H}\ind{t})^2  } \log (\abs*{\Gclass}T/\delta)  }  \mid{} x_{1:H}\ind{t} } \\
    &\hspace{1in}  +(24\eta^2\coef^2 + 4\eta\coef)\log (\abs*{\Gclass}T/\delta)T\\
    &\hspace{1in} + \eta \sum\limits_{t=1}^{T}  \En_{t-1}\brk*{  \En_{f\sim p\ind{t}}\brk*{15 |\cE(f,f,x_{1:H}\ind{t})|^2  +280 Z\ind{t}(f)   } \mid{} x_{1:H}\ind{t} }\\
    &\hspace{1in} + 4\eta\coef^2 T + \frac{1}{\lambda} \log \abs*{\Gclass}.
\end{align*}
Reorganizing the terms, we have
\begin{align*}
  \hspace{1in}&\hspace{-1in}\prn*{1 - \frac{32\eta}{\lambda}\log (\abs*{\Gclass}T/\delta) -280 \eta   } \sum\limits_{t=1}^{T} \En_{t-1}\brk*{ \En_{f\sim p\ind{t}} Z\ind{t}(f) \mid{} x_{1:H}\ind{t} } \\
&\leq \prn*{\frac{32\eta}{\lambda}\log (\abs*{\Gclass}T/\delta)  + 15\eta }\sum\limits_{t=1}^{T}  \En_{t-1}\brk*{ \En_{f\sim p\ind{t}} \cE(f,f,x_{1:H}\ind{t})^2 \mid{} x_{1:H}\ind{t} } \\
&\quad \hspace{1in} +(24\eta^2\coef^2 + 4\eta\coef)\log (\abs*{\Gclass}T/\delta)T + 4\eta\coef^2 T +  \frac{1}{\lambda} \log \abs*{\Gclass} .
\end{align*}
Using the assumed bound on $\eta$ and $\lambda$ gives
\begin{align*}
   \frac{1}{2} \sum\limits_{t=1}^{T} \En_{t-1}\brk*{ \En_{f\sim p\ind{t}} Z\ind{t}(f)  \mid{} x_{1:H}\ind{t} } &\leq \frac{1}{2^8} \sum\limits_{t=1}^{T}  \En_{t-1}\brk*{ \En_{f\sim p\ind{t}} \brk*{ \cE(f,f,x_{1:H}\ind{t})^2} \mid{} x_{1:H}\ind{t} } \\
   &\hspace{.7in}   +(24\eta^2\coef^2 + 4\eta\coef)\log (\abs*{\Gclass}T/\delta)T  + 4\eta\coef^2 T + 8 \log \abs*{\Gclass} .
\end{align*}

\end{proof}

\subsubsection{Proof of \cref{lem:Z-term} (Error Decomposition)}

    \begin{proof}[\pfref{lem:Z-term}]
Recalling the relationships $q\ind{t+1}(g\mid{}f) = q\ind{t}(g\mid{}f) \cdot e^{-\lambda [ \delta\ind{t}(g,f) - Z\ind{t}(f) ]}$ and 
$p\ind{t+1}(f) = p\ind{t}(f) \cdot e^{-\eta[ \coef \iota\ind{t}(f) +
  \delta\ind{t}(f,f) -Z\ind{t}(f) - Z\ind{t} ]}$, the proof begins
with the following manipulation:
\begin{align*}
  \hspace{1.5in}&\hspace{-1.5in}\En_{f\sim p\ind{t+1}(\cdot\mid{}\histSet\ind{t})} \brk*{\log \frac{1}{q\ind{t+1}(\cT f\mid{}f)}} - \En_{f\sim p\ind{t}(\cdot\mid{}\histSet\ind{t-1})}\brk*{ \log \frac{1}{q\ind{t}(\cT f\mid{}f)}}\notag\\
    &=  \En_{p\ind{t+1}-p_{t}} \brk*{\log \frac{1}{q\ind{t}(\cT f\mid{}f)}} + \En_{p\ind{t+1}-p_{t}}\brk*{ \log \frac{q\ind{t}(\cT f\mid{}f)}{q\ind{t+1}(\cT f\mid{}f)}} +  \En_{p_{t}}\brk*{ \log \frac{q\ind{t}(\cT f\mid{}f)}{q\ind{t+1}(\cT f\mid{}f)}}\notag\\
    &= \En_{p\ind{t}}\brk*{\prn*{ e^{-\eta \brk*{ \coef\iota\ind{t}(f)+ \delta\ind{t}(f,f) - Z\ind{t}(f)  - Z\ind{t}}} -1}\log \frac{1}{q\ind{t}(\cT f\mid{}f)}}   \\
    &+ \En_{p_{t}}\brk*{\prn*{ e^{-\eta \brk*{ \coef\iota\ind{t}(f)+ \delta\ind{t}(f,f) - Z\ind{t}(f)  - Z\ind{t}}} -1}\cdot \lambda \cdot  \prn*{ \delta\ind{t}(\cT f,f)   - Z\ind{t}(f)} }   \\
    &+ \lambda\cdot \En_{p_{t}}\brk*{ \delta\ind{t}(\cT f,f)   - Z\ind{t}(f)}.
\end{align*}
Thus, \ngedit{taking expectation with respect to $\En_{t-1}[\cdot]$, dividing by $\lambda$, rearranging}, summing over $t$ on both sides, \ngedit{and taking advatnage of telescoping}, we obtain
\begin{align*}
  \hspace{1in}&\hspace{-1in}\sum\limits_{t=1}^{T} \En_{t-1} \brk*{ \En_{f\sim p\ind{t}(\cdot\mid{}\histSet\ind{t-1})} Z\ind{t}(f)\mid{} x_{1:H}\ind{t} } \\
&\leq \frac{1}{\lambda} \sum\limits_{t=1}^{T} \En_{t-1} \brk*{ \En_{p\ind{t}}\brk*{ \prn*{ e^{-\eta[ \coef\iota\ind{t}(f)+ \delta\ind{t}(f,f) -Z\ind{t}(f) - Z\ind{t} ]} -1 }\log \frac{1}{q\ind{t}(\cT f\mid{}f)}}\mid{} x_{1:H}\ind{t} }   \\
&\quad +\sum\limits_{t=1}^{T} \En_{t-1} \brk*{\En_{p\ind{t}}\brk*{\prn*{ e^{-\eta \brk*{ \coef\iota\ind{t}(f)+ \delta\ind{t}(f,f) - Z\ind{t}(f)  - Z\ind{t}}} -1}\cdot \prn*{ \delta\ind{t}(\cT f,f)   - Z\ind{t}(f)}}\mid{} x_{1:H}\ind{t} } \\
&\quad + \sum\limits_{t=1}^{T} \En_{t-1} \brk*{\En_{p_{t}}\brk*{ \delta\ind{t}(\cT f,f)}\mid{} x_{1:H}\ind{t} } \\
&\quad + \frac{1}{\lambda} \prn*{\En_{0} \brk*{\En_{f\sim p\ind{1}} \log \frac{1}{q_1(\cT f\mid{}f)} \mid{}x_{1:H}\ind{1}}  -  \En_{T} \brk*{\En_{f\sim p\ind{T+1} } \log \frac{1}{q_{T+1}(\cT f\mid{}f)}\mid{}x_{1:H}\ind{T+1}}  }.
\end{align*}
Finally, we use that $\delta\ind{t}(\cT f, f) = 0$, $ \log \frac{1}{q_{T+1}(\cT f\mid{}f)} \ge 0$, and $\log \frac{1}{q_1(\cT f\mid{}f)} = \log \abs*{\Gclass} $.

\end{proof}

\subsubsection{Proof of \cref{lem:term-I} (Bound on Term I)}
\label{sec:termi}

Toward proving \cref{lem:term-I}, we state and prove a series of
technical lemmas, \cref{lem:good-event,lem:z-square-controlled-by-z,lem:log-partition-control,lem:distribution-shift}.

\begin{lemma}
\label{lem:good-event}
Under \pref{ass:completeness-OL}, for any $0<\delta < 1$, the event
$\Event\ind{t}$ defined below holds \jqedit{simultaneously} for all $t\leq T$ with
probability at least $1-\delta$ over the draw of $\histSet\ind{T}$:
\begin{align*}
\Event\ind{t} = \crl*{ \sup_{f\in \Gclass} ~ \log \frac{1}{q\ind{t+1}(\cT f\mid{}f)} \leq 2\log (\abs*{\Gclass}T/\delta )}.
\end{align*}
\end{lemma}

\begin{proof}[\pfref{lem:good-event}]
Recalling that $q\ind{t+1}(g\mid{}f) = \frac{\exp\prn*{-\lambda \sum\limits_{s=1}^{t}  \delta\ind{s}(g,f)}}{\sum\limits_{g'\in \Gclass} \exp\prn*{-\lambda \sum\limits_{s=1}^{t}  \delta\ind{s}(g',f)} }$, we have
\begin{align*}
\log \prn*{\En_{\histSet\ind{t}}\brk*{ \sup_{f} \frac{1}{q\ind{t+1}(\cT f\mid{}f)}} }&=   \log \prn*{\En_{\histSet\ind{t}} \brk*{\sup_{f}\sum\limits_{g} \exp \prn*{-\lambda \sum\limits_{s=1}^{t} \delta\ind{s}(g,f)}  }} \tag{\pref{lem:delta_Ta_a}}\\
    &\leq   \log \prn*{\En_{\histSet\ind{t}} \brk*{ \sum\limits_{f,g} \exp \prn*{-\lambda \sum\limits_{s=1}^{t} \delta\ind{s}(g,f) } }}\\
    &=    \log  \prn*{\sum\limits_{f,g}  \En_{\histSet\ind{t}} \brk*{\exp \prn*{-\lambda \sum\limits_{s=1}^{t} \delta\ind{s}(g,f)  }}}\\
    &\leq   2\log \abs*{\Gclass} , \tag{\pref{lem:exp-control}}
\end{align*}
\ngedit{where the final inequality uses $0 < \lambda < 1/2$ as well as the fact that for any fixed choice of $f,g$, conditioned on $x_{1:H}\ind{s}$, $s \leq t$, the random variables $\delta\ind{s}(g,f)$ (for $s \leq t$) are independent.} 
 Then, by Markov's inequality and the union bound, we have the desired result.
\end{proof}

\begin{lemma}
\label{lem:z-square-controlled-by-z}
For any $t\leq T$, $0\le \lambda\le 1/8$, $f$, almost surely with respect to the draw of $\histSet\ind{t}$ and $x_{1:H}\ind{t}$, we have \[\En_{t-1}\brk[\big]{ \abs*{Z\ind{t}(f)}^2 \mid{} x_{1:H}\ind{t} }\leq 40\cdot \En_{t-1}\brk*{ Z\ind{t}(f) \mid{} x_{1:H}\ind{t} }.\]
\end{lemma}

\begin{proof}[\pfref{lem:z-square-controlled-by-z}]
Using \pref{lem:z-to-delta}, we get
\begin{align*}
    \En_{t-1}\brk*{ \abs*{Z\ind{t}(f)}^2\mid{} x_{1:H}\ind{t} } &\le 2 \En_{t-1}\brk*{  \prn*{\En_{g\sim q\ind{t}(\cdot\mid{}f)} \abs*{\delta\ind{t}(g,f)}}^2\mid{} x_{1:H}\ind{t} } \\
&\leq 2\En_{t-1}\brk*{  \En_{g\sim q\ind{t}(\cdot\mid{}f)} \abs*{\delta\ind{t}(g,f)}^2 \mid{} x_{1:H}\ind{t} } \tag{by Jensen}\\
&\le 10\En_{t-1}\brk*{  \En_{g\sim q\ind{t}(\cdot\mid{}f)} \cE(g,f,x_{1:H}\ind{t})^2 \mid{} x_{1:H}\ind{t} } \tag{\pref{lem:4-th-moment-controlled-by-the-2-nd-moment}} \\
&\le 40 \En_{t-1}\brk*{ Z\ind{t}(f)\mid{} x_{1:H}\ind{t} }. \tag{\pref{lem:g-Qstar-positivity}}
\end{align*}

\end{proof}

\begin{lemma}
\label{lem:log-partition-control}
For any $t\leq T$, almost surely with respect to the draw of $\histSet\ind{t}$ and $x_{1:H}\ind{t}$, we have
\begin{align*}
    \En_{t-1}\brk*{ \abs*{Z\ind{t}}^2 \mid{} x_{1:H}\ind{t} } \leq \En_{t-1}\brk*{ \En_{f\sim p\ind{t}(\cdot \mid{}\histSet\ind{t-1})} [10\cE(f,f,x_{1:H}\ind{t})^2 + 80Z\ind{t}(f)]\mid{} x_{1:H}\ind{t} }  +  \jqedit{3\beta^2}.
\end{align*}
\end{lemma}

\begin{proof}[\pfref{lem:log-partition-control}]
Using Jensen's inequality, we have
\begin{align*}
    \abs*{Z\ind{t}} &=  \abs*{-\frac{1}{\eta} \log \prn*{ \En_{f\sim p\ind{t}} \exp (-\eta [\jqedit{\beta \iota\ind{t}(f)} + \delta\ind{t}(f,f) - Z\ind{t}(f)  ] )} }\\
    &\leq  \En_{f\sim p\ind{t}} \brk*{ \jqedit{\abs*{\beta \iota\ind{t}(f)}}+ \abs*{ \delta\ind{t}(f,f) } + \abs*{ Z\ind{t}(f) }}.
\end{align*}
Thus by \pref{lem:4-th-moment-controlled-by-the-2-nd-moment} and \pref{lem:z-square-controlled-by-z}, we have
\begin{align*}
    \En_{t-1}\brk*{ \abs*{Z\ind{t}}^2\mid{} x_{1:H}\ind{t} } &\leq 3 \En_{t-1}\brk*{   \En_{f\sim p\ind{t}}\brk*{ \jqedit{\abs*{\beta \iota\ind{t}(f)}^2} +  \abs*{ \delta\ind{t}(f,f) }^2 + \abs*{ Z\ind{t}(f) }^2}\mid{} x_{1:H}\ind{t} }\\
    &\leq \En_{t-1}\brk*{   \En_{f\sim p\ind{t}} \brk*{ 10\cE(f,f,x_{1:H}\ind{t})^2 + 80 Z\ind{t}(f)}\mid{} x_{1:H}\ind{t} } +  \jqedit{3\beta^2}. 
\end{align*}

\end{proof}

\begin{lemma}
\label{lem:distribution-shift}
For any $0\le \eta \le 1/240$, $f$, almost surely with respect to the draw of $\histSet\ind{t}$ and $x_{1:H}\ind{t}$, we have
\begin{align*}
    \hspace{1in}&\hspace{-1in}\En_{t-1} \brk*{   e^{-\eta\brk*{ \coef\iota\ind{t}(f) +\delta\ind{t}(f,f) -Z\ind{t}(f) - Z\ind{t}  }}  -1 \mid{} x_{1:H}\ind{t} }\\
    &\leq   \En_{t-1}\brk*{16\eta Z\ind{t}(f) +4\eta \En_{f'\sim p\ind{t}(\cdot\mid{}\histSet\ind{t-1})} \cE(f',f',x_{1:H}\ind{t})^2\mid{} x_{1:H}\ind{t} }\\
    &\hspace{1in}+ 12\eta^2\coef^2 + 2\eta\coef.
\end{align*}
\end{lemma}

\begin{proof}[\pfref{lem:distribution-shift}]
  \ngedit{Jensen's inequality gives that
    \begin{align*}
      Z\ind{t} &\leq \En_{f' \sim p\ind{t}} \left[ \beta \iota\ind{t}(f') + \delta\ind{t}(f',f') - Z\ind{t}(f') \right]\\
      Z\ind{t}(f) &\leq \E_{g \sim q\ind{t}(\cdot | f)}\left[ \delta\ind{t}(g,f) \right] = \delta\ind{t}(f).
      \end{align*}}
    We may then write
\begin{align*}
    &\En_{t-1}\brk*{   e^{-\eta\brk*{\coef\iota\ind{t}(f) + \delta\ind{t}(f,f) -Z\ind{t}(f) - Z\ind{t}  }}  -1 \mid{} x_{1:H}\ind{t} } \\
    &\leq \En_{t-1}\brk*{   e^{-\eta\coef\iota\ind{t}(f) -\eta\delta\ind{t}(f,f) + \eta \delta\ind{t}(f) +\eta  \En_{f'\sim p\ind{t}}[ \coef\iota\ind{t}(f') + \delta\ind{t}(f',f') - Z\ind{t}(f')  ]  } -1\mid{} x_{1:H}\ind{t} }   \\
    &\leq \frac{1}{6} \En_{t-1}\brk*{ \En_{f'\sim p\ind{t}} \brk*{ e^{-6\eta\coef\iota\ind{t}(f)} + e^{-6\eta\delta\ind{t}(f,f)}  + e^{6\eta \delta\ind{t}(f)}  + e^{6\eta\coef\iota\ind{t}(f')} + e^{6\eta \delta\ind{t}(f',f')} + e^{-6\eta Z\ind{t}(f')} -6 }\mid{} x_{1:H}\ind{t} },
\end{align*}
\ngedit{
where the second equality uses Jensen's inequality to pull the $\En_{f' \sim p\ind{t}}[\cdot]$ outside of the exponential and then the fact that $a_1 \cdots a_6 \leq \frac 16 \sum_{i=1}^6 a_i^6$ for real numbers $a_1, \ldots, a_6$.} 
We control the six terms on the right hand side separately as follows:

\begin{enumerate}[label=$\bullet$,leftmargin=5mm]
    \item Term (a): Using \pref{lem:technical-2} we get 
    \begin{align*}
      \En_{t-1}\brk*{ e^{-6\eta\coef\iota\ind{t}(f)} \mid{} x_{1:H}\ind{t} }  \leq 1 + 6\eta\coef + 36\eta^2\coef^2. 
  \end{align*}  
  \item Term (b): Using \pref{lem:exp-control} we get 
  \begin{align*}
    \En_{t-1}\brk*{ e^{-6\eta\delta\ind{t}(f,f)} \mid{} x_{1:H}\ind{t} }  \leq 1.
  \end{align*}  
  \item Term (c): 
  \begin{align*}
    \En_{t-1}\brk*{ e^{6\eta \delta\ind{t}(f)}\mid{} x_{1:H}\ind{t} } &= \En_{t-1}\brk*{ \exp (6\eta  \En_{g\sim{}q\ind{t}(\cdot\mid{}f)} \delta\ind{t}(g,f)  )\mid{} x_{1:H}\ind{t} } \\
    &\leq \En_{t-1}\brk*{ \En_{g\sim{}q\ind{t}(\cdot\mid{}f)} \exp (6\eta  \delta\ind{t}(g,f)  )\mid{} x_{1:H}\ind{t} } \\
    &\leq \En_{t-1}\brk*{ \En_{g\sim{}q\ind{t}(\cdot\mid{}f)} \exp\prn*{6\eta(1+12\eta)\cE(g,f,x_{1:H}\ind{t})^2 }\mid{} x_{1:H}\ind{t} } \tag{\pref{lem:exp-control}} \\
    &\leq  1 + 16\eta\En_{t-1}\brk*{ \En_{g\sim{}q\ind{t}(\cdot\mid{}f)} \cE(g,f,x_{1:H}\ind{t})^2 \mid{} x_{1:H}\ind{t} }\tag{\pref{lem:technical}}\\
    &\leq 1 +  64\eta \En_{t-1}\brk*{ Z\ind{t}(f)\mid{} x_{1:H}\ind{t} }. \tag{\pref{lem:g-Qstar-positivity}}
\end{align*}
\item Term (d): Using \pref{lem:technical-2}, we have
\begin{align*}
  \En_{t-1}\brk*{  \En_{f'\sim p\ind{t}} e^{6\eta\coef\iota\ind{t}(f')}\mid{} x_{1:H}\ind{t} }&\leq 1 + 6 \eta\coef + 36\eta^2\coef^2. 
\end{align*}
\item Term (e): 
\begin{align*}
  \hspace{1in}&\hspace{-1in}\En_{t-1}\brk*{  \En_{f'\sim p\ind{t}} e^{6\eta\delta\ind{t}(f',f')}\mid{} x_{1:H}\ind{t} }   \\
    &\leq  \En_{t-1}\brk*{ \En_{f'\sim p\ind{t}} \exp\prn*{6\eta(1+12\eta ) \cE(f',f',x_{1:H}\ind{t})^2 }\mid{} x_{1:H}\ind{t} }  \tag{\pref{lem:exp-control}} \\
    &\leq 1+ 16 \eta   \En_{t-1}\brk*{ \En_{f'\sim p\ind{t}} \cE(f',f',x_{1:H}\ind{t})^2\mid{} x_{1:H}\ind{t} }. \tag{\pref{lem:technical}}
\end{align*} 
\item Term (f): 
\begin{align*}
  \hspace{1in}&\hspace{-1in}\En_{t-1}\brk*{\En_{f'\sim p\ind{t}}   e^{-6\eta Z\ind{t}(f')} \mid{} x_{1:H}\ind{t} }  \\
    &\leq  \En_{t-1}\brk*{\En_{f'\sim p\ind{t}}   1 - 6\eta Z\ind{t}(f') + 36\eta^2(Z\ind{t}(f'))^2 \mid{} x_{1:H}\ind{t} }\tag{\pref{lem:technical-2}} \\
    &\leq   1 +  (-6\eta + 1440\eta^2) \En_{t-1}\brk*{ \En_{f'\sim p\ind{t}}   Z\ind{t}(f') \mid{} x_{1:H}\ind{t} }\tag{\pref{lem:z-square-controlled-by-z}}\\
    &\leq 1. \tag{\pref{lem:g-Qstar-positivity}}
\end{align*} 
\end{enumerate}

\end{proof}

\begin{proof}[\pfref{lem:term-I}]
  Combining the preceding lemmas, we have with probability at least $1-\delta$ \ngedit{over the draw of $S\ind{T}$}, \jqedit{the event $A\ind{t}$ defined in \cref{lem:good-event} holds for all $t\leq T$ simultaneously, and thus}
    \begin{align*}
        &\En_{t-1}\brk*{\En_{f\sim p\ind{t}} \prn*{   e^{-\eta\brk*{\coef\iota\ind{t}(f)+\delta\ind{t}(f,f) -Z\ind{t}(f) - Z\ind{t}  }}  -1 } \log \frac{1}{q\ind{t}(\cT f\mid{}f)}\mid{} x_{1:H}\ind{t} }\\
        &\leq  \En_{f\sim p\ind{t}} \brk*{\prn*{ 16\eta  \En_{t-1}\brk*{ Z\ind{t}(f) +\En_{f'\sim p\ind{t}} \cE(f',f',x_{1:H}\ind{t})^2\mid{} x_{1:H}\ind{t} } + 12\eta^2\coef^2+2\eta\coef }   \log \frac{1}{q\ind{t}(\cT f\mid{}f)}}\\
        &\leq 16\eta \En_{t-1}\brk*{\En_{f\sim p\ind{t}} \brk*{ \prn*{Z\ind{t}(f)+ \cE(f,f,x_{1:H}\ind{t})^2 }\cdot 2\log (\abs*{\Gclass}T/\delta)  }\mid{} x_{1:H}\ind{t} } \\
        &\hspace{1in} + (24\eta^2\coef^2 + 4\eta\coef)\log (\abs*{\Gclass}T/\delta), 
    \end{align*}
    where the first inequality is derived from
    \pref{lem:distribution-shift}, and the second inequality follows from
    \pref{lem:g-Qstar-positivity} and \pref{lem:good-event}.
\end{proof}

\subsubsection{Proof of \cref{lem:term-II} (Bound on Term II)}
\label{sec:termii}

\begin{proof}[\pfref{lem:term-II}]
  We compute
\begin{align*}
  \hspace{1in}&\hspace{-1in}\En_{t-1}\brk*{ \En_{f\sim p\ind{t}} \brk*{\prn*{   e^{-\eta\brk*{\coef\iota\ind{t}(f) + \delta\ind{t}(f,f) -Z\ind{t}(f) - Z\ind{t}  }}  -1 }   Z\ind{t}(f)}\mid{} x_{1:H}\ind{t} }\\
    &\le \En_{t-1}\brk*{ \En_{f\sim p\ind{t}} \brk*{\prn*{   e^{\eta\prn*{\coef|\iota\ind{t}(f)| + |\delta\ind{t}(f,f)| + |Z\ind{t}(f)| + |Z\ind{t}|  }    } - 1}  \abs*{ Z\ind{t}(f)}}\mid{} x_{1:H}\ind{t} } \\
    &\le 2 \eta \En_{t-1}\brk*{ \En_{f\sim p\ind{t}} \brk*{ \brk*{\coef|\iota\ind{t}(f)| + |\delta\ind{t}(f,f)| + |Z\ind{t}(f)| + |Z\ind{t}|  } |Z\ind{t}(f)|    }\mid{} x_{1:H}\ind{t} } \\
    &\le  \eta \En_{t-1}\brk*{ \En_{f\sim p\ind{t}} \brk*{ |\delta\ind{t}(f,f)|^2 + 5|Z\ind{t}(f)|^2 + |Z\ind{t}|^2  } \mid{} x_{1:H}\ind{t} } + \eta \coef^2\\
    &\le   \eta \En_{t-1}\brk*{ \En_{f\sim p\ind{t}} \brk*{15 \cE(f,f,x_{1:H}\ind{t})^2 + 280 Z\ind{t}(f)} \mid{} x_{1:H}\ind{t} } + 4\eta\coef^2, 
\end{align*}
where the second inequality applies \pref{lem:technical} \ngedit{(together with the fact that $\eta \in (0,1/240)$)}, \ngedit{the third inequality uses Young's inequality and the fact that $|\iota\ind{t}(f)| \leq 1$ for all $f$, and the final inequality applies \pref{lem:z-square-controlled-by-z} (to bound $\E_{t-1}[|Z\ind{t}(f)|^2]$),  \pref{lem:log-partition-control} (to bound $\E_{t-1}[|Z\ind{t}|^2]$), and \cref{lem:4-th-moment-controlled-by-the-2-nd-moment} (to bound $\E_{t-1}[\delta\ind{t}(f,f) \mid x_{1:H}\ind{t}]$)
}.
\end{proof}

\section{Proofs from \cref{sec:main} and \cref{sec:discussion}}
\label{app:main}

\subsection{Proofs from \cref{sec:main}}
\label{app:main_proofs}

\begin{proof}[Proof of \pref{thm:main} and \pref{thm:main_batched}]
We prove \pref{thm:main_batched}; \pref{thm:main} is the special case
of this result in which $n=1$. Observe that for
\pref{alg:main_batched}, we have
\begin{align*}
  \RegDM
  =
  n\cdot\sum_{k=1}^{K}\En_{\pi\ind{k}\sim{}p\ind{k}}\brk*{\fmstar(\pimstar)
  - \fmstar(\pi\ind{k})}.
\end{align*}
We can rewrite this sum as
\begin{align*}
  &\sum_{k=1}^{K}\En_{\pi\ind{k}\sim{}p\ind{k}}\brk*{\fmstar(\pimstar)
  - \fmstar(\pi\ind{k})} \\
  & =   \sum_{k=1}^{K}\En_{\pi\ind{k}\sim{}p\ind{k}}\En_{\suffhat\ind{k}\sim\mu\ind{k}}\brk*{f^{\suffhat\ind{k}}(\pi_{\suffhat\ind{k}})
    - \fmstar(\pi\ind{k})} + \En_{\suffhat\ind{k}\sim\mu\ind{k}}
    \brk*{(\fmstar(\pimstar) -
    f^{\suffhat\ind{k}}(\pi_{\suffhat\ind{k}}))}\\
    & =   \sum_{k=1}^{K}\En_{\pi\ind{k}\sim{}p\ind{k}}\En_{\suffhat\ind{k}\sim\mu\ind{k}}\brk*{f^{\suffhat\ind{k}}(\pi_{\suffhat\ind{k}})
      - \fmstar(\pi\ind{k})
      -\gamma\cdot\Dgenpi[\pi\ind{k}]{\suffhat\ind{k}}{\Mstar}
      } \\
  &~~~~~~~~~~~~~+ \En_{\suffhat\ind{k}\sim\mu\ind{k}}
    \brk*{\gamma\cdot\Dgenpi[\pi\ind{k}]{\suffhat\ind{k}}{\Mstar} +
    (\fmstar(\pimstar) -
    f^{\suffhat\ind{k}}(\pi_{\suffhat\ind{k}}))}\\
  & =   \sum_{k=1}^{K}\En_{\pi\ind{k}\sim{}p\ind{k}}\En_{\suffhat\ind{k}\sim\mu\ind{k}}\brk*{f^{\suffhat\ind{k}}(\pi_{\suffhat\ind{k}})
      - \fmstar(\pi\ind{k})
      -\gamma\cdot\Dgenpi[\pi\ind{k}]{\suffhat\ind{k}}{\Mstar}
      }  + \gamma\cdot{}\EstOptD.
\end{align*}
For each step $k$, by the choice of $p\ind{k}$, we have
\begin{align*}
  &\En_{\pi\ind{k}\sim{}p\ind{k}}\En_{\suffhat\ind{k}\sim\mu\ind{k}}\brk*{f^{\suffhat\ind{k}}(\pi_{\suffhat\ind{k}})
      - \fmstar(\pi\ind{k})
    -\gamma\cdot\Dgenpi[\pi\ind{k}]{\suffhat\ind{k}}{\Mstar}} \\
  &\leq\sup_{M\in\cM}\En_{\pi\ind{k}\sim{}p\ind{k}}\En_{\suffhat\ind{k}\sim\mu\ind{k}}\brk*{f^{\suffhat\ind{k}}(\pi_{\suffhat\ind{k}})
      - \fm(\pi\ind{k})
    -\gamma\cdot\Dgenpi[\pi\ind{k}]{\suffhat\ind{k}}{M}}\\
&=\inf_{p\in\Delta(\Pi)}\sup_{M\in\cM}\En_{\pi\sim{}p}\En_{\suffhat\sim\mu\ind{k}}\brk*{f^{\suffhat}(\pi_{\suffhat})
      - \fm(\pi)
                                                            -\gamma\cdot\Dgenpi[\pi]{\suffhat}{M}}\\
  &= \ocompD(\cM,\mu\ind{k})\leq\ocompD(\cM).
\end{align*}
Finally, we use that probability at least $1-\delta$, $\EstOptD\leq{}\EstOptDFullKn$.

\end{proof}

We next prove \pref{prop:estimation_bilinear}, which gives a bound on
\begin{align}
\EstOptBi = \sum_{k=1}^{K}\En_{\act\ind{k}\sim{}p\ind{k}}\En_{\Mhat\ind{k}\sim\mu\ind{k}}\brk*{\sum_{h=1}^{H}\prn*{\Enmpi{\Mstar}{\pi}\brk*{
  \lest(\wh{Q}\ind{k};z_h)}}^2
  +
  \gamma^{-1}(\fmstar(\pimstar)-\fmhatt(\pi\subs{\Mhat\ind{k}}))}\nonumber.
\end{align}

\begin{algorithm}[htp]
    \setstretch{1.3}
    \begin{algorithmic}[1]
      \State \textbf{parameters}:
      \Statex[1] $\bullet$ Number of rounds $T$, Batch size $n$
      \Statex[1] $\bullet$ Learning rate $\eta>0$.
         \Statex[1] $\bullet$ Discrepancy function $\lest(Q;z_h\ind{k,l})$
         Exploration parameter $\gamma>0$.
         \State Let $K=T/n$ and $B\ind{0}\ldef\emptyset$.
  \For{$k=1, 2, \cdots, K$}
  \State Form the randomized
estimator $\mu\ind{k}$ via
$\mu\ind{k}(Q)
  \propto\exp\prn*{
  -\eta\sum_{i<k}\ls\ind{i}(Q)
}$, where
\[
  \ls\ind{i}(Q) \ldef \sum_{h=1}^{H}\prn*{\frac{1}{n}\sum_{l=1}^{n}\lest(Q;z_h\ind{i,l})}^2
  - \frac{1}{8\gamma}\cdot{}\fq(\piq)
  \]
\State Receive batch of samples $B\ind{k}=\crl*{(\pi\ind{k,l},r\ind{k,l},o\ind{k,l})}_{l=1}^{n}$ where
$\pi\ind{k,l}\sim{}p\ind{k}$ and
$(r\ind{k,l},o\ind{k,l})\sim\Mstar(\pi\ind{k,l})$.
\Statex[1]\hfill\algcommentlight{$p\ind{k}$ is the decision
  distribution produced by \etdopt (cf. \cref{alg:main_batched}).}
\State Let
$z_h\ind{k,l}=(s_h\ind{k,l},a_h\ind{k,l},r_h\ind{k,l},s_{h+1}\ind{k,l})$,
where we recall $o\ind{k,l}=(s_1\ind{k,l},a_1\ind{k,1},r_1\ind{k,l}),\ldots, (s_H\ind{k,l},a_H\ind{k,1},r_H\ind{k,l})$.
\EndFor
\end{algorithmic}
\caption{Optimistic Estimation for Bilinear Classes}
\label{alg:bilinear_est}
\end{algorithm}
\begin{proof}[\pfref{prop:estimation_bilinear}]%
  \neurips{Throughout this proof, we use the batched estimation notation from
  \cref{sec:batched}. }Let $n$ be fixed, and let $K\ldef{}T/n$ be the number of epochs. Recall that for each step $k\in\brk{K}$, the estimation
oracle is given a batch of examples $B\ind{k}=\crl*{(\pi\ind{k,l},r\ind{k,l},o\ind{k,l})}_{l=1}^{n}$ where
$\pi\ind{k,l}\sim{}p\ind{k}$ and
$(r\ind{k,l},o\ind{k,l})\sim\Mstar(\pi\ind{k,l})$. Each observation
(trajectory) takes the form
$o\ind{k,l}=(s_1\ind{k,l},a_1\ind{k,l},r_1\ind{k,l}),\ldots, 
(s_H\ind{k,l},a_H\ind{k,l},r_H\ind{k,l})$. Throughout the proof, we
use the notation
\[
  \cE_h\ind{k}(Q) = \Enmpi{\Mstar}{\pi\ind{k}}\brk*{\lest(Q;z_h)},
\]
and abbreviate $\Qstar=\Qmstarstar$.
\paragraph{Estimation algorithm}
Define
\[
\cEhat_h\ind{k}(Q) = \frac{1}{n}\sum_{l=1}^{n}\lest(Q;z_h\ind{k,l}),
\]
where
$z_h\ind{k,l}\ldef{}(s_h\ind{k,l},a_h\ind{k,l},r_h\ind{k,l},s_{h+1}\ind{k,l})$. Let
$\eta>0$ and $\alpha>0$ be parameters whose values will be chosen at
the end of the proof. Defining
\[
\ls\ind{k}(Q) \ldef \sum_{h=1}^{H}(\cEhat_h\ind{k}(Q))^2 - \alpha\fq(\piq),
\]
\cref{alg:bilinear_est} chooses
\begin{align*}
  \mu\ind{k}(Q)
  \propto\exp\prn*{
  -\eta\sum_{i<k}\ls\ind{i}(Q)
  }
\end{align*}
as the randomized estimator for epoch $k$.
\paragraph{Estimation error bound}
Let us abbreviate $L=\Lbi(\cQ;\Mstar)$. Observe that for all $k$,
$\abs{\ls\ind{k}(Q)}\leq{}R\ldef{}HL^2 + \alpha$ almost
surely. Hence, \pref{lem:exp_weights} implies that as long as
$\eta\leq{}1/2R$, we have
\begin{align}
  \sum_{k=1}^{K}\En_{Q\sim\mu\ind{k}}\brk*{\ls\ind{k}(Q)}
  -   \sum_{k=1}^{K}\ls\ind{k}(\Qstar)
  \leq{}
  4\eta\sum_{k=1}^{K}\En_{Q\sim\mu\ind{k}}\brk*{(\ls\ind{k}(Q))^2}
  + \frac{\log\abs{\cQ}}{\eta}
  \label{eq:bilinear_est1}
\end{align}
For each $k\in\brk{K}$, we have that for all $Q\in\cQ$,
\begin{align*}
(\ls\ind{k}(Q))^2
  \leq{} 2\prn*{\sum_{h=1}^{H}(\cEhat_h\ind{k}(Q))^2}^2 + 2\alpha^2
  \leq{} 2HL^2\sum_{h=1}^{H}(\cEhat_h\ind{k}(Q))^2 + 2\alpha^2.
\end{align*}
As a result, \pref{eq:bilinear_est1} implies that
\begin{align}
&  \sum_{k=1}^{K}\sum_{h=1}^{H}\En_{Q\sim\mu\ind{k}}\brk*{\prn*{\cEhat_h\ind{k}(Q)}^2}
  +
                \alpha\sum_{k=1}^{K}\En_{Q\sim\mu\ind{k}}\brk*{(\fmstar(\pimstar)-\fq(\piq)} \\
  &\leq{}
  \sum_{k=1}^{K}\sum_{h=1}^{H}\prn*{\cEhat_h\ind{k}(\Qstar)}^2
+   8\eta{}HL^2\sum_{k=1}^{K}\sum_{h=1}^{H}\En_{Q\sim\mu\ind{k}}\brk*{\prn*{\cEhat_h\ind{k}(Q)}^2}
  + 8\eta{}\alpha^2K
  + \frac{\log\abs{\cQ}}{\eta}.
  \label{eq:bilinear_est2}
\end{align}
Whenever $\eta\leq{}\frac{1}{16HL^2}$, rearranging gives
\begin{align}
&  \frac{1}{2}\sum_{k=1}^{K}\sum_{h=1}^{H}\En_{Q\sim\mu\ind{k}}\brk*{\prn*{\cEhat_h\ind{k}(Q)}^2}
  +
                \alpha\sum_{k=1}^{K}\En_{Q\sim\mu\ind{k}}\brk*{(\fmstar(\pimstar)-\fq(\piq)} \\
  &\leq{}
  \sum_{k=1}^{K}\sum_{h=1}^{H}\prn*{\cEhat_h\ind{k}(\Qstar)}^2
  + 8\eta{}\alpha^2K
  + \frac{\log\abs{\cQ}}{\eta}.
  \label{eq:bilinear_est3}
\end{align}
We now appeal to the following lemma.
\begin{lemma}
  \label{lem:bilinear_conc}
  With probability at least $1-\delta$, it holds that for all
  $k\in\brk{K}$, $h\in\brk{H}$, and $Q\in\cQ$,
  \begin{align}
    \label{eq:bilinear_conc}
    \frac{1}{2}(\cE_h\ind{k}(Q))^2 - \vepsconcs  \leq  (\cEhat_h\ind{k}(Q))^2 \leq{} 2 (\cE_h\ind{k}(Q))^2 + 2\vepsconcs,
  \end{align}
  where $\vepsconc\ldef{}L\sqrt{\frac{2\log(\abs{\cQ}KH/\delta)}{n}}$.
\end{lemma}
Going forward, we condition on the event in
\pref{lem:bilinear_conc}. Applying the inequality
\pref{eq:bilinear_conc} within \pref{eq:bilinear_est3}, we have
\begin{align}
&  \frac{1}{4}\sum_{k=1}^{K}\sum_{h=1}^{H}\En_{Q\sim\mu\ind{k}}\brk*{\prn*{\cE_h\ind{k}(Q)}^2}
  +
                \alpha{}\sum_{k=1}^{K}\En_{Q\sim\mu\ind{k}}\brk*{(\fmstar(\pimstar)-\fq(\piq)} \\
  &\leq{}
    2\sum_{k=1}^{K}\sum_{h=1}^{H}\prn*{\cE_h\ind{k}(\Qstar)}^2
  + 8\eta{}\alpha^2K
    + \frac{\log\abs{\cQ}}{\eta} + 3HK\vepsconcs\notag\\
  &=8\eta{}\alpha^2K
    + \frac{\log\abs{\cQ}}{\eta} + 3HK\vepsconcs,
  \label{eq:bilinear_est4}
\end{align}
where the last equality uses that $\cE_h\ind{k}(\Qstar)=0$ for all
$h\in\brk{H}$, and $k\in\brk{K}$, which is a consequence of
\pref{def:bilinear}. Next, we recall that
\[
  \sum_{k=1}^{K}\sum_{h=1}^{H}\En_{Q\sim\mu\ind{k}}\brk*{\prn*{\cE_h\ind{k}(Q)}^2}
  = \sum_{k=1}^{K}\En_{Q\sim\mu\ind{k}}\brk*{\Dbipi[\pi\ind{k}]{Q}{\Mstar}}.
\]
A standard application of Freedman's inequality (c.f Lemma A.3 in
\citet{foster2021statistical}) implies that with probability at least
$1-\delta$,
\begin{align*}
  \sum_{k=1}^{K}\En_{Q\sim\mu\ind{k}}\brk*{\Dbipi[\pi\ind{k}]{Q}{\Mstar}}
  \geq{}
  \frac{1}{2}\sum_{k=1}^{K}\En_{\pi\sim{}p\ind{k}}\En_{Q\sim\mu\ind{k}}\brk*{\Dbipi[\pi\ind{k}]{Q}{\Mstar}}
  - \bigoh(HL^2\log(H/\delta)).
\end{align*}
Putting everything together, we have that %
\begin{align*}
  \frac{1}{8}\sum_{k=1}^{K}\En_{\pi\ind{k}\sim{}p\ind{k}}\En_{Q\sim\mu\ind{k}}\brk*{\Dbipi[\pi\ind{k}]{Q}{\Mstar}
  + 8\alpha (\fmstar(\pimstar)-\fq(\piq)}
  \approxleq{} \eta{}\alpha^2K
    + \frac{\log\abs{\cQ}}{\eta} + HK\vepsconcs + HL^2\log(H/\delta).
\end{align*}
Choosing $\alpha=\frac{1}{8\gamma}$, this gives
\begin{align*}
\EstOptBi &\approxleq  \eta{}\alpha^2K
  + \frac{\log\abs{\cQ}}{\eta} + HK\vepsconcs + HL^2\log(H/\delta) \\
          &\approxleq{}
            \eta{}\alpha^2K
            + \frac{\log\abs{\cQ}}{\eta} + HL^2\log(\abs{\cQ}KH/\delta)\prn*{1+\frac{K}{n}}
\end{align*}
We choose
$\eta=\sqrt{\frac{\log\abs{\cQ}}{\alpha^2K}}\wedge{}\frac{1}{16R}$,
which satisfies the constraints described earlier in the proof, and
gives
\begin{align*}
  \EstOptBi           &\approxleq{}
                        \frac{\sqrt{K\log\abs{\cQ}}}{\gamma}
                        + R\log\abs{\cQ}
                        + 
                        HL^2\log(\abs{\cQ}KH/\delta)\prn*{1+\frac{K}{n}}
  \\
   &\approxleq{}
     \frac{\sqrt{K\log\abs{\cQ}}}{\gamma}
     + 
                        HL^2\log(\abs{\cQ}KH/\delta)\prn*{1+\frac{1}{\gamma}+\frac{K}{n}}.
\end{align*}
  
\end{proof}

\begin{proof}[\pfref{lem:bilinear_conc}]%
  For any fixed $k\in\brk{K}$ and $Q\in\cQ$, Hoeffding's inequality
  implies that with probability at least $1-\delta$,
  \begin{align*}
    \abs*{\cE_h\ind{k}(Q)-\cEhat_h\ind{k}(Q)}
    \leq{} L\sqrt{\frac{2\log(1/\delta)}{n}}.
  \end{align*}
  By a standard union bound it follows that with probability at least
  $1-\delta$, for all $k\in\brk{K}$,
  $h\in\brk{H}$, and $Q\in\cQ$ simultaneously,
    \begin{align*}
    \abs*{\cE_h\ind{k}(Q)-\cEhat_h\ind{k}(Q)}
    \leq{} L\sqrt{\frac{2\log(\abs{\cQ}KH/\delta)}{n}} \rdef \vepsconc.
    \end{align*}
    Whenever this event occurs, the AM-GM inequality implies that
    \[
      (\cE_h\ind{k}(Q))^2 \leq{} 2 (\cEhat_h\ind{k}(Q))^2 + 2\vepsconcs
    \]
    and likewise $      (\cEhat_h\ind{k}(Q))^2 \leq{} 2 (\cE_h\ind{k}(Q))^2 + 2\vepsconcs$.
\end{proof}

\begin{proof}[\pfref{prop:dec_bilinear}]
  Let $\mu\in\Delta(\cQ)$ be fixed. Fix $\alpha\in(0,1)$, and let
  $\pialphaq$ be the randomized policy that---for each
  $h$---independently plays $\pi_{Q,h}$ with probability
$1-\alpha/H$ and $\piest_{Q,h}$ with probability $\alpha/H$. Let
$p\in\Delta(\Pi)$ be the distribution induced by sampling $Q\in\mu$
and playing $\pialphaq$. Translated to our notation, \citet{foster2021statistical} (cf. Proof
  of Theorem 7.1, Eq. (152)) shows that for all MDPs $M$, this
  strategy guarantees that for all $\eta>0$,
  \begin{align*}
    \En_{\pi\sim{}p}\En_{Q\sim\mu}\brk*{
    \fq(\piq) - \fm(\pi)
    }
    &\leq{} \alpha + 
    \frac{H\cdot{}\dimbi(\cQ;M)}{2\eta}
  + \frac{\eta{}}{2}\sum_{h=1}^{H}\En_{Q,Q'\sim\mu}\brk[\Big]{\tri*{X_h(Q;M),W_h(Q';M)
      }^{2}}\\
        &= \alpha + 
    \frac{H\cdot{}\dimbi(\cQ;M)}{2\eta}
          + \frac{\eta{}}{2}\sum_{h=1}^{H}\En_{Q,Q'\sim\mu}\brk*{
\prn*{\Enmpi{M}{\pi^{\vphantom{\mathrm{est}}}_{Q}\circ_{h}\piest_{Q}}\brk*{
        \lest(Q';z_h)}
        }^2
          }.
  \end{align*}
  
  In the on-policy case in which $\piestq=\piq$, it suffices to set
  $\alpha=0$, which gives
  \begin{align*}
\sum_{h=1}^{H}\En_{Q,Q'\sim\mu}\brk*{
\prn*{\Enmpi{M}{\pi^{\vphantom{\mathrm{est}}}_{Q}\circ_{h}\piest_{Q}}\brk*{
        \lest(Q';z_h)}
        }^2
    } 
    & =
      \sum_{h=1}^{H}\En_{Q,Q'\sim\mu}\brk*{
\prn*{\Enmpi{M}{\piq}\brk*{
        \lest(Q';z_h)}
        }^2
    }\\
    &= \sum_{h=1}^{H}\En_{\pi\sim{}p}\En_{Q'\sim\mu}\brk*{
\prn*{\Enmpi{M}{\pi}\brk*{
        \lest(Q';z_h)}
        }^2
      }
      = \En_{\pi\sim{}p}\En_{Q\sim\mu}\brk*{\Dbipi{Q}{M}},
  \end{align*}
  where the last equality relabels $Q'\gets{}Q$. Setting $\eta=2\gamma$ yields the result.

    In the general case where $\piestq\neq{}\piq$, we have that for
    all $Q'\in\cQ$,\footnote{Note that this result uses that the
      quantity $\Enmpi{M}{\pi}\brk*{
        \lest(Q';z_h)}$ only depends on the policy $\pi$ through $a_1,\ldots,a_h$.}
    \begin{align*}
      \En_{\pi\sim{}p}\brk*{
\prn*{\Enmpi{M}{\pi}\brk*{
        \lest(Q';z_h)}
        }^2}
      \geq{}\frac{\alpha}{H}(1-\alpha/H)^{H-1}
      \En_{Q\sim\mu}\brk*{
\prn*{\Enmpi{M}{\pi^{\vphantom{\mathrm{est}}}_{Q}\circ_{h}\piest_{Q}}\brk*{
        \lest(Q';z_h)}
        }^2}.
    \end{align*}
    We have
    $\frac{\alpha}{H}(1-\alpha/H)^{H-1}\geq{}\frac{\alpha}{2H}$
    whenever $\alpha\leq{}1/2$, which gives
      \begin{align*}
    \En_{\pi\sim{}p}\En_{Q\sim\mu}\brk*{
    \fq(\piq) - \fm(\pi)
    }
        &\leq{} \alpha + 
    \frac{H\cdot{}\dimbi(\cQ;M)}{2\eta}
          +
          \frac{\eta{}H}{\alpha}\cdot{}\En_{\pi\sim{}p}\En_{Q\sim\mu}\brk*{\Dbipi{Q}{M}}\\
        &= \alpha + 
    \frac{H^2\cdot{}\dimbi(\cQ;M)}{2\gamma\alpha}
          + \gamma\cdot{}\En_{\pi\sim{}p}\En_{Q\sim\mu}\brk*{\Dbipi{Q}{M}},
      \end{align*}
      where the last line chooses $\eta=\gamma\alpha/H$. To conclude,
      we set $\alpha = \sqrt{H^2\dimbi(\cQ;\cM)/4\gamma}$, which is
      admissible whenever $\gamma\geq{}H^2\dimbi(\cQ;\cM)$.

    \end{proof}

    \begin{proof}[\pfref{cor:bilinear}]
 In both of the cases in the theorem statement (on-policy and off-policy), whenever $\gamma\geq{}1$ and $n\leq\sqrt{T}$ (to ensure that $K/n \geq 1$), combining \pref{thm:main_batched} and  \pref{prop:estimation_bilinear} gives %
  \begin{align*}
    \RegDM
    \approxleq{} \ocompbi(\cM)\cdot{}T 
    + \gamma{}HL^2\log(\abs{\cQ}KH/\delta)\cdot\frac{T}{n} + \sqrt{nT\log\abs{\cQ}}.
  \end{align*}
  We choose $n=\sqrt{T}$, which
  is admissible whenever $T$ is sufficiently large, and gives
    \begin{align*}
    \RegDM
    \approxleq{} \ocompbi(\cM)\cdot{}T + \gamma HL^2\log(\abs{\cQ}KH/\delta) \sqrt{T} + \sqrt{\log\abs{\cQ}}\cdot T^{3/4}.
    \end{align*}
    In the on-policy case, we have, from \pref{prop:dec_bilinear}, that 
    \begin{align*}
          \RegDM
      \approxleq{} \frac{HdT}{\gamma} + \gamma HL^2\log(\abs{\cQ}KH/\delta) \sqrt{T} + \sqrt{\log\abs{\cQ}}\cdot T^{3/4},
    \end{align*}
    and in the off-policy case, we have
        \begin{align*}
          \RegDM
      \approxleq{} \sqrt{\frac{H^2d}{\gamma}}{T} + \gamma HL^2\log(\abs{\cQ}KH/\delta) \sqrt{T} + \sqrt{\log\abs{\cQ}}\cdot T^{3/4}.
        \end{align*}
        Choosing $\gamma$ to balance yields the result.
      \end{proof}

      \arxiv{    
\subsubsection{Proofs from \cref{sec:complete}}

\cref{prop:ts3_rl} is an application of more general results given in
\cref{app:completeness}, which analyze a generalization of
\cref{alg:ts3_rl} for a more general online learning setting. To
\cref{prop:ts3_rl}, we simply apply these results to the reinforcement
learning framework.
\begin{proof}[\pfref{prop:ts3_rl}]
  Let the batch size $n=H$ be fixed, and let $K\ldef{}T/n$ be the number of epochs. Recall that for each step $k\in\brk{K}$, the estimation
  oracle is given a batch of examples $B\ind{k}=\crl*{(\pi\ind{k,l},r\ind{k,l},o\ind{k,l})}_{l=1}^{n}$ where
  $\pi\ind{k,l}\sim{}p\ind{k}$ and
  $(r\ind{k,l},o\ind{k,l})\sim\Mstar(\pi\ind{k,l})$. Each observation
  (trajectory) takes the form
  $o\ind{k,l}=(s_1\ind{k,l},a_1\ind{k,l},r_1\ind{k,l}),\ldots, 
  (s_H\ind{k,l},a_H\ind{k,l},r_H\ind{k,l})$. 
  We abbreviate $\Qstar=\Qmstarstar$.
  \paragraph{Estimation algorithm} For each step $k$, the randomized
  estimator $\mu\ind{k}$ selected as described in
  \pref{alg:ts3_rl}. This algorithm is an instantiation of
  \pref{alg:ts3} in the general online learning setting described in
  \cref{app:completeness}, with $\Gclass = \cQ$ and for all $h\in
  [H]$, $\Xclass_h = \cS \times \cA$, $\Yclass_h = \bbR\times \cS$ and
  $\Zclass = \ngedit{\cS^H}$. The unknown kernels are the transition
  distributions for the corresponding layers of the MDP $\Mstar$, and the loss functions are
  \begin{align*}
  \loss_{h,1}((s_h, a_h), Q) &\ldef Q_h(s_h, a_h) ,\\
  \loss_{h,2}((r_h,s_{h+1}),Q) &\ldef r_h +  \max_a Q_h(s_{h+1},a),\\
  \loss_3(\set{s_1\ind{l}}_{l\in [H]} ) &\ldef  - \frac{1}{H} \sum\limits_{l=1}^{H} \max_a Q_1(s_1\ind{l},a) 
  \end{align*}
  Finally, take $\x_h\ind{k} = (s_h\ind{k,h}, a_h\ind{k,h})$,
  $\y_h\ind{k} = (r_h\ind{k,h}, s_{h+1}\ind{k,h})$ and $\z\ind{k} =
  \set{s_1\ind{k,h}}_{h\in [H]}$. \jqedit{It is important to note that
    $s_h\ind{k,h}, a_h\ind{k,h},r_h\ind{k,h}, s_{h+1}\ind{k,h}$ are
    taken from different trajectories for $h\in [H]$, so
    $\y_h\ind{k}\mid{}\x_h\ind{k}$ are independent from one other for
    $h\in [H]$.} \ngedit{Moreover, note that the distributions
    $p\ind{k} \in \Delta(\Pi)$ play the role of nature: the
    distribution of the tuple ($x_h\ind{k}, y_h\ind{k}, w\ind{k})$ for $h \in [H]$ is determined by running a policy $\pi\ind{k} \sim p\ind{k}$ in the ground-truth MDP $M^\star$.} With this configuration, observe that in the notation of \cref{app:completeness}, we have\ngedit{, for any $Q$,} 
  \begin{align*}
    \En_{\x_{1:H}\ind{k}}\cE(Q,Q,\x_{1:H}\ind{k})^2 &= \frac{1}{H} \sum\limits_{h=1}^{H}\En_{\x_h\ind{k}}(\loss_{h,1}(\x_h\ind{k},Q) - \En[\loss_{h,2}(\y_h,Q)|\x_h\ind{k}])^2 \\
    &= \frac{1}{H} \sum\limits_{h=1}^{H}\En_{\pi\ind{k,h}\sim p\ind{k}}\Enmpi{\Mstar}{\pi\ind{k,h}}\brk*{\prn*{Q_h(s_h,a_h)-\brk{\cT\sups{\Mstar}_hQ_{h+1}}(s_h,a_h)}^2} \\
    &= \frac{1}{H} \En_{\pi\sim p\ind{k}}\Enmpi{\Mstar}{\pi}\brk*{ \sum\limits_{h=1}^{H}\prn*{Q_h(s_h,a_h)-\brk{\cT\sups{\Mstar}_hQ_{h+1}}(s_h,a_h)}^2}\\
    &= \frac{1}{H} \En_{\pi\sim p\ind{k}} \Dsbepi{Q}{\Mstar}.
  \end{align*}
and 
\begin{align*}
  \En_{\x_{1:H}\ind{k},\z\ind{k}}\iota\ind{k}(Q) &= \frac{1}{H} \sum\limits_{l=1}^{H} \En_{\pi\ind{k,h}\sim p\ind{k}}\Enmpi{\Mstar}{\pi\ind{k,h}} \brk*{ \max_a Q_1^*(s_1\ind{k,l},a) - \max_a Q_1(s_1\ind{k,l},a)}\\
  &= \fmstar(\pimstar)-\fq(\piq) .
\end{align*}

  \paragraph{Estimation error bound} We take $\alpha = 12\coef$, so that \pref{thm:main-OL} implies that with probability at least $1-\delta$, 
  \begin{align*}
    &\sum\limits_{k=1}^{K}  \En_{Q\sim \mu\ind{k}} \prn*{\frac{1}{H} \En_{\pi\sim p\ind{k}} \Dsbepi{Q}{\Mstar} + \alpha (\fmstar(\pimstar)-\fq(\piq) )  } \\
    &\approxleq \eta\alpha\log(\abs{\cQ}K/\delta)K  + \log\abs{\cQ}/\eta  + \alpha^2 K.
  \end{align*}
  Then by taking $\alpha = \frac{1}{\gamma H}$, this further implies that with probability at least $1-\delta$,
  \begin{align*}
    \EstOptSB &= \sum\limits_{k=1}^{K}  \En_{\pi\sim p\ind{k}}\En_{Q\sim \mu\ind{k}} \prn*{ \Dsbepi{Q}{\Mstar} + \frac{1}{\gamma} (\fmstar(\pimstar)-\fq(\piq) )  } \\
    &\approxleq H( \eta\alpha\log(\abs{\cQ}K/\delta)K  + \log\abs{\cQ}/\eta  + \alpha^2 K)\\
    &\approxleq{} \frac{H\log\abs{\cQ}}{\eta} + \frac{\eta\log (\abs{\cQ}K/\delta)K}{\gamma} + \frac{K}{\gamma^2 H}.
  \end{align*}
\end{proof}

\begin{proof}[\pfref{cor:regret_complete}] We choose $n=H$ and apply \cref{alg:ts3_rl} as the estimation oracle. 
We first consider the ``trivial'' parameter regime in which $ H d^{1/3}(\log(\abs{\cQ}K/\delta))^{-1/5}T^{-1/3} \geq 1/(2^{16}(\log (|\cQ|K/\delta) + 1))$. Here, $T \approxleq H d^{1/3} (\log(\abs{\cQ}K/\delta))^{4/5}T^{2/3}$, and thus
 \begin{align*}
  \RegDM
  \approxleq{} H d^{1/3} (\log(\abs{\cQ}K/\delta))^{4/5}T^{2/3}.
 \end{align*}
 When the case above, does not hold, we proceed as in the theorem statement, choosing $\eta = H d^{1/3}(\log(\abs{\cQ}K/\delta))^{-1/5}T^{-1/3} \leq  1/(2^{16}(\log (|\cQ|K/\delta) + 1))$. Combining \pref{thm:main_batched} and  \pref{prop:ts3_rl} then gives %
  \begin{align*}
    \RegDM
    \approxleq{} \ocompsbe(\cM)\cdot{}T 
    + \gamma{}\frac{H^2\log\abs{\cQ}}{\eta} + \eta\log (\abs{\cQ}K/\delta)T + \frac{K}{\gamma^2}
  \end{align*}
  with probability at least $1-\delta$.
    Next, using \pref{prop:dec_bilinear} to bound $\ocompsbe(\cM)$ in the above display, it follows that 
    \begin{align*}
          \RegDM
      \approxleq{} \frac{HdT}{\gamma} + \gamma{}\frac{H^2\log\abs{\cQ}}{\eta} + \eta\log (\abs{\cQ}K/\delta)T + \frac{K}{\gamma^2} .
    \end{align*}
    We choose $\gamma  = d^{2/3} (\log (\abs{\cQ}K/\delta))^{-2/5} T^{1/3}$ to obtain 
    \begin{align*}
      \RegDM
      \approxleq{} H d^{1/3} (\log(\abs{\cQ}K/\delta))^{4/5}T^{2/3} 
    \end{align*}
    with probability at least $1-\delta$.

  \end{proof}

 }

\subsection{Proofs from \cref{sec:discussion}}

\begin{proof}[\pfref{prop:equiv_symmetry}]
  We begin with the upper bound. First, note that by
  \pref{ass:continuity} and the AM-GM inequality, we have
  \begin{align}
    \ocompD(\cM) &= \sup_{\mu\in\Delta(\cM)}\inf_{p\in\Delta(\Pi)}\sup_{M\in\cM}\En_{\pi\sim{}p}\En_{\Mbar\sim\mu}\brk*{
  \fmbar(\pimbar) - \fm(\pi)  - \gamma\cdot{}\Dgenpi{\Mbar}{M}
                   } \notag\\
    &\leq{} \sup_{\mu\in\Delta(\cM)}\inf_{p\in\Delta(\Pi)}\sup_{M\in\cM}\En_{\pi\sim{}p}\En_{\Mbar\sim\mu}\brk*{
  \fmbar(\pimbar) - \fmbar(\pi)  - \frac{\gamma}{2}\cdot{}\Dgenpi{\Mbar}{M}
  } + \frac{\Lcont^2}{2\gamma}.\label{eq:equiv1}
  \end{align}
  Consider any fixed choice for $\mu\in\Delta(\cM)$. By Sion's minimax theorem (see \citet{foster2021statistical} for
  details), we have
  \begin{align*}
    &\inf_{p\in\Delta(\Pi)}\sup_{M\in\cM}\En_{\pi\sim{}p}\En_{\Mbar\sim\mu}\brk*{
  \fmbar(\pimbar) - \fmbar(\pi)  - \frac{\gamma}{2}\cdot{}\Dgenpi{\Mbar}{M}
    } \\
    & =
      \inf_{p\in\Delta(\Pi)}\sup_{\nu\in\Delta(\cM)}\En_{\pi\sim{}p}\En_{\Mbar\sim\mu}\En_{M\sim\nu}\brk*{
  \fmbar(\pimbar) - \fmbar(\pi)  - \frac{\gamma}{2}\cdot{}\Dgenpi{\Mbar}{M}
      } \\
        & =
          \sup_{\nu\in\Delta(\cM)}\inf_{p\in\Delta(\Pi)}\En_{\pi\sim{}p}\En_{\Mbar\sim\mu}\En_{M\sim\nu}\brk*{
  \fmbar(\pimbar) - \fmbar(\pi)  - \frac{\gamma}{2}\cdot{}\Dgenpi{\Mbar}{M}
    },
  \end{align*}
  so that the main term in \pref{eq:equiv1} is equal to 
  \begin{align*}
    &
      \sup_{\mu\in\Delta(\cM)}\sup_{\nu\in\Delta(\cM)}\inf_{p\in\Delta(\Pi)}\En_{\pi\sim{}p}\En_{\Mbar\sim\mu}\En_{M\sim\nu}\brk*{
  \fmbar(\pimbar) - \fmbar(\pi)  - \frac{\gamma}{2}\cdot{}\Dgenpi{\Mbar}{M}
      } \\
        &\leq{}
      \sup_{\nu\in\Delta(\cM)}\inf_{p\in\Delta(\Pi)}\sup_{\mu\in\Delta(\cM)}\En_{\pi\sim{}p}\En_{\Mbar\sim\mu}\En_{M\sim\nu}\brk*{
  \fmbar(\pimbar) - \fmbar(\pi)  - \frac{\gamma}{2}\cdot{}\Dgenpi{\Mbar}{M}
          } \\
            &=
\sup_{\nu\in\Delta(\cM)}\inf_{p\in\Delta(\Pi)}\sup_{\Mbar\in\cM}\En_{\pi\sim{}p}\En_{M\sim\nu}\brk*{
  \fmbar(\pimbar) - \fmbar(\pi)  - \frac{\gamma}{2}\cdot{}\Dgenpi{\Mbar}{M}
    }.
  \end{align*}
  Relabeling, this is equal to
  \begin{align*}
    \sup_{\mu\in\Delta(\cM)}\inf_{p\in\Delta(\Pi)}\sup_{M\in\cM}\En_{\pi\sim{}p}\En_{\Mbar\sim\mu}\brk*{
  \fm(\pim) - \fm(\pi)  - \frac{\gamma}{2}\cdot{}\Dgenpi{M}{\Mbar}
    } =     \sup_{\mu\in\Delta(\cM)}\compgenrandbasic[\Dflipshort]_{\gamma/2}(\cM,\mu).
  \end{align*}

We now prove the lower bound. Using \pref{ass:continuity} and the
AM-GM inequality once more, we have
  \begin{align}
    \ocompD(\cM) &= \sup_{\mu\in\Delta(\cM)}\inf_{p\in\Delta(\Pi)}\sup_{M\in\cM}\En_{\pi\sim{}p}\En_{\Mbar\sim\mu}\brk*{
  \fmbar(\pimbar) - \fm(\pi)  - \gamma\cdot{}\Dgenpi{\Mbar}{M}
                   } \notag\\
    &\geq{} \sup_{\mu\in\Delta(\cM)}\inf_{p\in\Delta(\Pi)}\sup_{M\in\cM}\En_{\pi\sim{}p}\En_{\Mbar\sim\mu}\brk*{
  \fmbar(\pimbar) - \fmbar(\pi)  - \frac{3\gamma}{2}\cdot{}\Dgenpi{\Mbar}{M}
      } - \frac{\Lcont^2}{2\gamma}.\notag\\
                 &\geq{} \sup_{\mu\in\Delta(\cM)}\sup_{\nu\in\Delta(\cM)}\inf_{p\in\Delta(\Pi)}\En_{\pi\sim{}p}\En_{\Mbar\sim\mu}\En_{M\sim\nu}\brk*{
  \fmbar(\pimbar) - \fmbar(\pi)  - \frac{3\gamma}{2}\cdot{}\Dgenpi{\Mbar}{M}
      } - \frac{\Lcont^2}{2\gamma}.\label{eq:equiv2}
  \end{align}
  Using the minimax theorem in the same fashion as before, the main term in \pref{eq:equiv2} 
  is equal to 
  \begin{align*}
    &\sup_{\nu\in\Delta(\cM)}\sup_{\mu\in\Delta(\cM)}\inf_{p\in\Delta(\Pi)}\En_{\pi\sim{}p}\En_{\Mbar\sim\mu}\En_{M\sim\nu}\brk*{
  \fmbar(\pimbar) - \fmbar(\pi)  - \frac{3\gamma}{2}\cdot{}\Dgenpi{\Mbar}{M}
      } \\
    &=\sup_{\nu\in\Delta(\cM)}\inf_{p\in\Delta(\Pi)}\sup_{\mu\in\Delta(\cM)}\En_{\pi\sim{}p}\En_{\Mbar\sim\mu}\En_{M\sim\nu}\brk*{
  \fmbar(\pimbar) - \fmbar(\pi)  - \frac{3\gamma}{2}\cdot{}\Dgenpi{\Mbar}{M}
      } \\
    &=\sup_{\nu\in\Delta(\cM)}\inf_{p\in\Delta(\Pi)}\sup_{\Mbar\in\cM}\En_{\pi\sim{}p}\En_{M\sim\nu}\brk*{
  \fmbar(\pimbar) - \fmbar(\pi)  - \frac{3\gamma}{2}\cdot{}\Dgenpi{\Mbar}{M}
      }.
  \end{align*}
  Relabeling, this is equal to
  \begin{align*}
    \sup_{\mu\in\Delta(\cM)}\inf_{p\in\Delta(\Pi)}\sup_{M\in\cM}\En_{\pi\sim{}p}\En_{\Mbar\sim\mu}\brk*{
  \fm(\pim) - \fm(\pi)  - \frac{3\gamma}{2}\cdot{}\Dgenpi{M}{\Mbar}
      } =  \sup_{\mu\in\Delta(\cM)}\compgenrandbasic[\Dflipshort]_{3\gamma/2}(\cM,\mu).
  \end{align*}
  
\end{proof}

\begin{proof}[\pfref{prop:randomized_equivalence}]
  Let $\Mbar$ be arbitrary. By Sion's minimax theorem (see \citet{foster2021statistical} for
  details), we have
  \begin{align*}
    \compgen(\cM,\Mbar)
  = \sup_{\mu\in\Delta(\cM)}\inf_{p\in\Delta(\Pi)}\En_{\pi\sim{}p,M\sim\mu}\brk*{\fm(\pim)-\fm(\pi)-\gamma\cdot\Dgenpi{\Mbar}{M}}.
  \end{align*}
  By the assumed triangle inequality for $\Dgenshort$, we have that for all $\pi\in\Pi$,
  \begin{align*}
    \En_{M,M'\sim\mu}\brk*{\Dgenpi{M}{M'}}
    &\leq{}     C\En_{M\sim\mu}\brk*{\Dgenpi{\Mbar}{M}}
    +    C\En_{M'\sim\mu}\brk*{\Dgenpi{\Mbar}{M'}} \\
    &= 2C\En_{M\sim\mu}\brk*{\Dgenpi{\Mbar}{M}}.
  \end{align*}
  Applying this bound above, we have that
  \begin{align*}
        \compgen(\cM,\Mbar)
    &\leq{}
      \sup_{\mu\in\Delta(\cM)}\inf_{p\in\Delta(\Pi)}\En_{\pi\sim{}p,M\sim\mu}\brk*{\fm(\pim)-\fm(\pi)-\frac{\gamma}{2C}\cdot\En_{M'\sim\mu}\brk*{\Dgenpi{M'}{M}}}\\
    &\leq{}
      \sup_{\nu\in\Delta(\cM)}\sup_{\mu\in\Delta(\cM)}\inf_{p\in\Delta(\Pi)}\En_{\pi\sim{}p,M\sim\mu}\brk*{\fm(\pim)-\fm(\pi)-\frac{\gamma}{2C}\cdot\En_{M'\sim\nu}\brk*{\Dgenpi{M'}{M}}}\\
    &\leq{}
      \sup_{\nu\in\Delta(\cM)}\inf_{p\in\Delta(\Pi)}\sup_{M\in\cM}\En_{\pi\sim{}p}\brk*{\fm(\pim)-\fm(\pi)-\frac{\gamma}{2C}\cdot\En_{M'\sim\nu}\brk*{\Dgenpi{M'}{M}}}\\
    &=\sup_{\nu\in\Delta(\cM)}\compgenrandbasic_{\gamma/(2C)}(\cM,\nu).
  \end{align*}
  
\end{proof}

\begin{proof}[\pfref{prop:bilinear_separation}]%
      \newcommand{\tfrak}{\mathfrak{t}}%
    \newcommand{\sfrak}{\mathfrak{s}}%
    \newcommand{\afrak}{\mathfrak{a}}%
    \newcommand{\bfrak}{\mathfrak{b}}%
    \newcommand{\avec}{\vec{a}}%
    \newcommand{\Mavec}{M_{\avec}}%
    \newcommand{\fmavec}{f\sups{\Mavec}}%
    \newcommand{\pimavec}{\pi\subs{\Mavec}}%
    \newcommand{\Qmbar}{Q^{\sss{\Mbar},\star}}%
  The upper bound on $\ocompbi(\cM)$ follows from
  \pref{prop:dec_bilinear}, so it remains to prove the lower bound on
  $\compbi(\cM)$.

  For a model $M$, define the Bellman operator
  $\cT\sups{M}_h$ via
  \begin{align}
    [\cT_h\ind{\sss{M}} Q](s, a) = \En_{s_{h+1}\sim{}P_h\sups{M}(\cdot \mid{} s,a), r_h\sim R\sups{M}_h(\cdot \mid{} s,a)}\brk*{r_h + \max_{a'\in\cA} Q(s_{h+1}, a')}.
  \end{align}
  Let 
  \[
    \Dsbepi{Q}{M}
    \ldef{}\sum_{h=1}^{H}\Enmpi{M}{\pi}\brk*{
      \prn*{Q_h(s_h,a_h) - \brk*{\cTm_hQ_{h+1}}(s_h,a_h)}^2
      }
    \]
    and $\Dsbepi{\Mbar}{M}\ldef\Dsbepi{\Qmbar}{M}$.
    By Jensen's inequality, it suffices to lower bound $\compsbe(\cM)$. 
   
    Let $\cS=\crl{\sfrak,\tfrak}$ and $\cA=\crl{\afrak,\bfrak}$. We
    consider a sub-family $\cM'\subset\cM$ of deterministic
    combination lock MDPs parameterized by
    $\avec=(\avec_1,\ldots,\avec_H)\in\cA^{H}$, with $\Mavec$ defined
    as follows.
    \begin{itemize}
    \item The initial state is $s_1=\sfrak$.
    \item For each $h=1,\ldots,H-1$, if $s_h=\sfrak$, then selecting
      $a_h=\avec_h$ transitions to $s_{h+1}=\sfrak$, and selecting
      $a_h\neq{}\avec_h$ transitions to $s_{h+1}=\tfrak$.
      $\tfrak$ is a
      self-looping terminal state: if $s_h=\tfrak$, then
      $s_{h+1}=\tfrak$ regardless of the action taken.
    \item If $s_H=\sfrak$ and $a_H=\avec_H$, then $r_H=\Delta>0$; all
      other state-action tuples have zero reward.
    \end{itemize}
We choose $\Mbar$ such that $\Qmbar_h(s,a)=0$ for all $(h,s,a)$. We calculate that for
all $\avec\in\cA^{H}$, and all policies $\pi\in\PiRNS$,
\begin{itemize}
\item $\fmavec(\pimavec)-\fmavec(\pi)
  = \Delta\cdot\bbP^{\sss{\Mavec},\pi}(a_{1:H}\neq\avec)$.
\item $\Dsbepi{\Qmbar}{\Mavec}
  = \Delta^{2}\cdot\bbP^{\sss{\Mavec},\pi}(a_{1:H}=\avec)$.
\end{itemize}
It follows that
\begin{align*}
  \compsbe(\cM,\Mbar)
  &\geq{}
  \inf_{p\in\Delta(\Pi)}
  \max_{\avec\in\cA^{H}}\En_{\pi\sim{}p}\brk*{
  \Delta\cdot\bbP^{\sss{\Mavec},\pi}(a_{1:H}\neq\avec)
  - \gamma\cdot{}\Delta^{2}\cdot\bbP^{\sss{\Mavec},\pi}(a_{1:H}=\avec)
    }\\
    &\geq{}
  \inf_{p\in\Delta(\Pi)}
      \En_{\avec\sim\unif(\cA^{H})}\En_{\pi\sim{}p}\brk*{
  \Delta\cdot\bbP^{\sss{\Mavec},\pi}(a_{1:H}\neq\avec)
  - \gamma\cdot{}\Delta^{2}\cdot\bbP^{\sss{\Mavec},\pi}(a_{1:H}=\avec)
  }.
\end{align*}
It is straightforward to see by induction that for all $\pi\in\PiRNS$,
$\En_{\avec\sim\unif(\cA^{H})}\brk*{
\bbP^{\sss{\Mavec},\pi}(a_{1:H}=\avec)
  }\leq{}2^{-H}$ and $\En_{\avec\sim\unif(\cA^{H})}\brk*{
\bbP^{\sss{\Mavec},\pi}(a_{1:H}\neq\avec)
}\geq{} \frac{1}{2}$, so we have
\begin{align*}
  \compsbe(\cM,\Mbar)
  \geq{} \frac{\Delta}{2} - \gamma\frac{\Delta^2}{2^{H}}.
\end{align*}
The result follows by choosing $\Delta$ appropriately.

\end{proof}

\begin{proof}[\pfref{prop:cheating}]
  \newcommand{\sfrak}{\mathfrak{s}}%
  \newcommand{\afrak}{\mathfrak{a}}%
  \newcommand{\bfrak}{\mathfrak{b}}%
  \newcommand{\om}[1][M]{o\sups{#1}}%
  \newcommand{\omstar}[1][\Mstar]{o\sups{#1}}%
  \newcommand{\ombar}[1][\Mbar]{o\sups{#1}}%
  \newcommand{\omhat}[1][\Mhat]{o\sups{#1}}%
  \newcommand{\omhatt}[1][\Mhat\ind{t}]{o\sups{#1}}%
  Let $H$ be given, and assume without loss of generality that
  $S=2^{H-1}+2^{H-2}$. Let $N\ldef{}2^{H-2}$. We construct a family of MDPs
$\cM=\crl{M_1,\ldots,M_N}$ with deterministic rewards and transitions as follows.
\begin{itemize}
\item $\cA=\crl{\afrak,\bfrak}$.
\item The state space is $\cS=\cT\cup\cE\cup\crl{\sfrak}$. $\sfrak$ is
  a deterministic initial state which appears only in layer $1$. $\cT$
  represents a depth $H-2$ binary tree
  (with the root at layer $h=2$), which has $N=2^{H-2}$ leaf states in
  layer $H$, labeled by $\crl{1,\ldots,N}$, and $2^{H-1}-1$ states in
  total. $\cE=\crl{N+1,\ldots,2N}$ is an auxiliary collection of
  self-looping terminal states, each of which can appears in layers $h=2,\ldots,H$. In total,
  we have $\abs{\cS}=2^{H-1}+2^{H-2}$.
\end{itemize}
The dynamics and rewards for MDP $M_i$ are as follows.
\begin{itemize}
  \item At layer $h=1$, $s_1=\sfrak$ deterministically.
  \item If $a_1=\afrak$, we transition to $s_2=N+i\in\cE$, and if
    $a_1=\bfrak$, we transition to the root state for the tree
    $\cT$. We receive no reward.
  \item For $h\geq{}2$, all states $s\in\cE$ are self-looping and have no
    reward (i.e. $s_{h+1}=s_h$ if $s_h\in\cE$ for $h>1)$.
  \item For $h\geq{}2$, states in $\cT$ follow a standard
    deterministic binary tree
    structure (e.g., \citep{osband2016lower,domingues2021episodic}). Beginning from the root node at $h=2$, action $\afrak$
    transitions to the left successor, while action $\bfrak$
    transitions to the right successor. For $M_i$, we receive reward
    $1$ for reaching the leaf node $i\in\cT$ at layer $H$, and receive
    zero reward for all other states. Note that the transition
    probabilities for this portion of the MDP do not depend on $i$.
\end{itemize}

\paragraph{Online estimation}
We first construct an online estimation algorithm for the class
$\cM$. Recall that we adopt Hellinger distance, given by $\Dhelspi{\Mhat}{M}=\Dhels{\Mhat(\pi)}{M(\pi)}$.

We first note that since all $M\in\cM$ have $\fm(\pim)=1$, the
optimistic estimation error is equal to the (non-optimistic) estimation error
\[
\EstH = \sum_{t=1}^{T}\En_{\pi\ind{t}\sim{}p\ind{t}}\En_{\Mhat\ind{t}\sim\mu\ind{t}}\brk*{\Dhels{\Mhat\ind{t}(\pi\ind{t})}{\Mstar(\pi\ind{t})}}.
\]
Observe that since all MDPs $M\in\cM$ have deterministic rewards and
transitions, there exists a function $\om(\pi)$ such that
$o\sim{}M(\pi)$ has $o=\om(\pi)$ almost surely. It follows that
\[
  \Dhels{M(\pi)}{M'(\pi)}=2\indic\crl{\om(\pi)\neq{}o\sups{M'}(\pi)}
\]
for all $M,M'\in\cM$, so we have
\[
\EstH = \sum_{t=1}^{T}\En_{\pi\ind{t}\sim{}p\ind{t}}\En_{\Mhat\ind{t}\sim\mu\ind{t}}\brk*{\Dhels{\Mhat\ind{t}(\pi\ind{t})}{\Mstar(\pi\ind{t})}}
  = 2\sum_{t=1}^{T}\En_{\pi\ind{t}\sim{}p\ind{t}}\En_{\Mhat\ind{t}\sim\mu\ind{t}}\brk*{\indic\crl{\omhatt(\pi\ind{t})\neq\omstar(\pi\ind{t})}}.
\]
For the estimation algorithm, we choose
\[
  \mu\ind{t}(M)\propto\exp\prn*{
    -\sum_{i<t}\indic\crl{\om(\pi\ind{i})\neq{}o\ind{i}}
    }.
\]
\pref{lem:exp_weights} implies that with probability $1$, the sequence
$\pi\ind{1},\ldots,\pi\ind{T}$ satisfies
\begin{align*}
  \sum_{t=1}^{T}\En_{\Mhat\ind{t}\sim\mu\ind{t}}\brk*{\indic\crl{\omhatt(\pi\ind{t})\neq{}o\ind{t}}}
  -
  \sum_{t=1}^{T}\indic\crl{\omstar(\pi\ind{t})\neq{}o\ind{t}}
  \leq{}
  \frac{1}{2}\sum_{t=1}^{T}\En_{\Mhat\ind{t}\sim\mu\ind{t}}\brk*{\indic\crl{\omhatt(\pi\ind{t})\neq{}o\ind{t}}}
  + \log\abs{\cM}.
\end{align*}
Since $\omstar(\pi\ind{t})=o\ind{t}$, rearranging yields
\[
  \sum_{t=1}^{T}\En_{\Mhat\ind{t}\sim\mu\ind{t}}\brk*{\indic\crl{\omhatt(\pi\ind{t})\neq{}o\ind{t}}}
  \leq{}2\log\abs{\cM}.
\]
From here, a standard application of Freedman's inequality
\citep{Freedman1975tail} implies
that with probability at least $1-\delta$,
$\sum_{t=1}^{T}\En_{\pi\ind{t}\sim{}p\ind{t}}\En_{\Mhat\ind{t}\sim\mu\ind{t}}\brk*{\indic\crl{\omhatt(\pi\ind{t})\neq{}o\ind{t}}}\approxleq{}
\log(\abs{\cM}/\delta) \approxleq{} \log(S/\delta)$.

\paragraph{Lower bound for posterior sampling}
Observe that for all $M_i\in\cM$, the unique optimal policy has
$\pimi(\sfrak)=\bfrak$ at $h=1$. Thus, the posterior sampling algorithm,
which chooses
$p\ind{t}(\pi)=\mu\ind{t}(\crl*{M\in\cM\mid{}\pim=\pi})$, will never
play $a_1=\afrak$, and will never encounter states in $\cE$. As a
result, the problem is equivalent (for this algorithm) to a multi-armed
bandit problem with $N$ arms and noiseless binary rewards, which
requires $\En\brk*{\RegDM}\approxgeq{}N\approxgeq{}S$ in the worst-case \citep{lattimore2020bandit}.

\paragraph{Upper bound for \pref{alg:main}}
We first bound the \newcomp for $\cM$. For any $\mu\in\Delta(\cM)$, we
can write
\begin{align*}
  \ocomphel(\cM,\mu)
  &=
  \inf_{p\in\Delta(\Pi)}\sup_{M\in\cM}\En_{\pi\sim{}p}\En_{\Mbar\sim\mu}\brk*{
  \fmbar(\pimbar) - \fm(\pi)  - \gamma\cdot{}\Dgen{\Mbar(\pi)}{M(\pi)}
    } \\
  &=
  \inf_{p\in\Delta(\Pi)}\sup_{M\in\cM}\En_{\pi\sim{}p}\En_{\Mbar\sim\mu}\brk*{
  \fmbar(\pimbar) - \fm(\pi)  - 2\gamma\cdot{}\indic\crl{\ombar(\pi)\neq\om(\pi)}
    } \\
    &=
  \inf_{p\in\Delta(\Pi)}\sup_{M\in\cM}\En_{\pi\sim{}p}\En_{\Mbar\sim\mu}\brk*{
      \fm(\pim) - \fm(\pi)  - 2\gamma\cdot{}\indic\crl{\ombar(\pi)\neq\om(\pi)}
  },
\end{align*}
where the last equality uses that $\fm(\pim)=1$ for all $M\in\cM$. We choose $p =
(1-\veps)\mu(\crl*{M\in\cM\mid{}\pim=\cdot})
+ \veps{}\pi_{\afrak}$, where $\pi_{\afrak}$ is the policy that plays
action $\afrak$ deterministically.

Now, let $M\in\cM$ be fixed. Since each model $M_i$ transitions
to $s_2=N+i$ deterministically when $a_1=\afrak$, we 
\[
  \En_{\pi\sim{}p}\indic\crl{\ombar(\pi)\neq\om(\pi)}
  \geq \veps\cdot \indic\crl{\ombar(\pi_{\afrak})\neq\om(\pi_{\afrak})}
  =\veps\cdot{}\indic\crl{M\neq\Mbar}
\]
and
\[
  \En_{\Mbar\sim\mu}\En_{\pi\sim{}p}\indic\crl{\ombar(\pi)\neq\om(\pi)}
  \geq{} \veps\mu(\cM\setminus{}\crl{M}).
\]
Similarly, we have
\[
\En_{\pi\sim{}p}\brk*{
  \fm(\pim) - \fm(\pi)
  } = \veps + (1-\veps)\mu(\cM\setminus{}\crl{M}).
\]
By choosing $\veps=\gamma^{-1}$, which is admissible whenever
$\gamma\geq{}1$, we have
\begin{align*}
  \En_{\pi\sim{}p}\En_{\Mbar\sim\mu}\brk*{
      \fm(\pim) - \fm(\pi)  - 2\gamma\cdot{}\indic\crl{\ombar(\pi)\neq\om(\pi)}
  }
  \leq{} \veps + \mu(\cM\setminus{}\crl{M})
  - 2\gamma \veps\mu(\cM\setminus{}\crl{M})
  \leq{} \frac{1}{\gamma}.
\end{align*}
This establishes that
\[
  \ocomphel(\cM) \leq \frac{1}{\gamma}
\]
for all $\gamma\geq{}1$. A regret bound of the form
$\En\brk*{\RegDM}\leq\bigoht(\sqrt{T\log(S)})$ now follows by invoking
\pref{thm:main} with the estimation guarantee in the prequel and
choosing $\gamma$ appropriately.

\end{proof}

\end{document}